\documentclass{article}

\PassOptionsToPackage{hidelinks}{hyperref}
\usepackage[nonatbib,final]{neurips_2021}

\usepackage[utf8]{inputenc} 
\usepackage[T1]{fontenc}    
\usepackage[english]{babel}  
\usepackage[
    bibstyle     = apa,
    citestyle    = apa,
    maxcitenames = 2,
    uniquename   = false,
    backend      = biber,
    backref      = true,
    natbib
    ]{biblatex}              
\usepackage{csquotes}        
\usepackage[font=small]{caption}  
\bibliography{references.bib}
\usepackage{hyperref}       
\usepackage{url}            
\usepackage{booktabs}       
\usepackage{amsfonts}       
\usepackage{nicefrac}       
\usepackage{microtype}      
\usepackage{xcolor}         
\usepackage{graphicx}
\usepackage{subcaption}
\usepackage{amsmath}
\usepackage{cleveref}
\usepackage{bbm}
\usepackage[colorinlistoftodos]{todonotes}
\usepackage{amsthm}
\usepackage{bbding}
\usepackage{amssymb}
\usepackage{mathtools}
\usepackage{pifont}

\usepackage{apptools}
\usepackage[normalem]{ulem}

\renewcommand{\cite}{\PackageError{document}{Do not use cite. Use citet or citep.}{}}
\renewcommand{\citealp}[2][]{\citeauthor{#2}, \citeyear{#2}#1}

\usetikzlibrary{
    calc,
    positioning,
    fit,
    tikzmark,
    arrows.meta,
    shapes,
    decorations.pathreplacing,
    intersections,
    through
}

\newcommand{\beforethmspace}{0.4em}

\newcommand{\R}{\mathbb{R}}
\newcommand{\E}{\mathbb{E}}
\renewcommand{\H}{\mathbb{H}}
\newcommand{\kl}{\operatorname{kl}}
\renewcommand{\d}{\partial}
\newcommand{\sd}{\text{d}}
\newcommand{\vardot}{\,\cdot\,}
\newcommand{\gibbsR}{\overline{R}}
\newcommand{\set}{S}
\newcommand{\inputs}{\mathcal{X}}
\newcommand{\outputs}{\mathcal{Y}}
\newcommand{\datapoints}{\mathcal{Z}}
\newcommand{\hypotheses}{\mathcal{H}}
\newcommand{\testset}{S_{\mathrm{test}}}
\newcommand{\trainset}{S_{\mathrm{train}}}
\newcommand{\priorset}{S_{\mathrm{prior}}}
\newcommand{\riskset}{S_{\mathrm{risk}}}
\newcommand{\setsize}{N}
\newcommand{\testsetsize}{N_{\mathrm{test}}}
\newcommand{\trainsetsize}{{N_{\mathrm{train}}}}
\newcommand{\priorsetsize}{{N_{\mathrm{prior}}}}
\newcommand{\risksetsize}{{N_{\mathrm{risk}}}}

\newcommand{\shortsubsection}[1]{\vspace*{.25em}\textbf{#1}}

\definecolor{plotorange}{RGB}{255,127,14}
\definecolor{plotblue}{RGB}{31,119,180}
\definecolor{plotgreen}{RGB}{44,160,44}
\definecolor{plotred}{RGB}{214,39,40}
\definecolor{plotpurple}{RGB}{148,103,189}
\definecolor{plotbrown}{RGB}{140,86,75}

\newcommand\sbullet[1][.5]{\mathbin{\vcenter{\hbox{\scalebox{#1}{$\bullet$}}}}}
\mathchardef\mhyphen="2D

\newtheorem{theorem}{Theorem}
\newtheorem{proposition}{Proposition}

\newtheorem{remark}{Remark}
\newtheorem{openproblem}{Open Problem}
\newtheorem{corollary}{Corollary}
\newtheorem{lemma}{Lemma}
\theoremstyle{definition}

\AtAppendix{\counterwithin{theorem}{section}}
\AtAppendix{\counterwithin{proposition}{section}}
\AtAppendix{\counterwithin{definition}{section}}
\AtAppendix{\counterwithin{remark}{section}}
\AtAppendix{\counterwithin{openproblem}{section}}
\AtAppendix{\counterwithin{corollary}{section}}
\AtAppendix{\counterwithin{lemma}{section}}
\AtAppendix{\counterwithin{example}{section}}

\crefname{corollary}{Corollary}{Corollaries}
\Crefname{corollary}{Corollary}{Corollaries}
\crefname{definition}{Definition}{Definitions}
\Crefname{definition}{Definition}{Definitions}
\crefname{example}{Example}{Examples}
\Crefname{example}{Example}{Examples}
\crefname{lemma}{Lemma}{Lemmas}
\Crefname{lemma}{Lemma}{Lemmas}
\crefname{proposition}{Proposition}{Propositions}
\Crefname{proposition}{Proposition}{Propositions}
\crefname{openproblem}{Open Problem}{Open Problems}
\Crefname{openproblem}{Open Problem}{Open Problems}
\crefname{remark}{Remark}{Remarks}
\Crefname{remark}{Remark}{Remarks}
\crefname{theorem}{Theorem}{Theorems}
\Crefname{theorem}{Theorem}{Theorems}
\crefname{table}{Table}{Tables}
\Crefname{table}{Table}{Tables}
\crefname{figure}{Figure}{Figures}
\Crefname{figure}{Figure}{Figures}
\crefname{appendix}{Appendix}{Appendices}
\Crefname{appendix}{Appendix}{Appendices}
\crefname{section}{Section}{Sections}
\Crefname{section}{Section}{Sections}
\crefname{subsection}{Subsection}{Subsections}
\Crefname{subsection}{Subsection}{Subsections}
\crefname{equation}{Equation}{Equations}
\Crefname{equation}{Equation}{Equations}

\makeatletter
\newtheorem*{rep@theorem}{\rep@title}
\newcommand{\newinformaltheorem}[2]{%
\newenvironment{informal#1}[1]{%
 \def\rep@title{#2 \ref{##1} (informal)}%
 \begin{rep@theorem}}%
 {\end{rep@theorem}}}
\makeatother

\newinformaltheorem{theorem}{Theorem}

\title{
    How Tight Can PAC-Bayes be in the \\Small Data Regime?
}

\author{%
   \parbox{5cm}{\centering\textbf{Andrew Y.~K.~Foong\thanks{Equal contribution.}}} \\
   University of Cambridge \\
   \texttt{ykf21@cam.ac.uk} \\
   \and
   \parbox{5cm}{\centering\textbf{Wessel P.~Bruinsma$^*$}} \\
   University of Cambridge \\
   Invenia Labs \\
   \texttt{wpb23@cam.ac.uk} \\
   \AND
   \parbox{5cm}{\centering\textbf{David R.~Burt}} \\
   University of Cambridge \\
   \texttt{drb62@cam.ac.uk} \\
   \and
   \parbox{5cm}{\centering\textbf{Richard E.~Turner}} \\
   University of Cambridge \\
   \texttt{ret26@cam.ac.uk} \\
}

\begin{document}

\maketitle

\vspace{-1.5em}
\begin{abstract}
In this paper, we investigate the question: \emph{Given a small number of datapoints, for example $N = 30$, how tight can PAC-Bayes and test set bounds be made?} For such small datasets, test set bounds adversely affect generalisation performance by withholding data from the training procedure. In this setting, PAC-Bayes bounds are especially attractive, due to their ability to use all the data to simultaneously learn a posterior and bound its generalisation risk. We focus on the case of i.i.d. data with a bounded loss and consider the generic PAC-Bayes theorem of Germain et al. While their theorem is known to recover many existing PAC-Bayes bounds, it is unclear what the tightest bound derivable from their framework is. For a fixed learning algorithm and dataset, we show that the tightest possible bound coincides with a bound considered by Catoni; and, in the more natural case of distributions over datasets, we establish a lower bound on the best bound achievable in expectation. Interestingly, this lower bound recovers the Chernoff test set bound if the posterior is equal to the prior. Moreover, to illustrate how tight these bounds can be, we study synthetic one-dimensional classification tasks in which it is feasible to meta-learn both the prior and the form of the bound to numerically optimise for the tightest bounds possible. We find that in this simple, controlled scenario, PAC-Bayes bounds are competitive with comparable, commonly used Chernoff test set bounds. However, the sharpest test set bounds still lead to better guarantees on the generalisation error than the PAC-Bayes bounds we consider.
\end{abstract}

\section{Introduction}
\everypar{\looseness=-1}
Generalisation bounds are of both practical and theoretical importance. Practically, tight bounds provide certificates that algorithms will perform well on unseen data.
Theoretically, the bounds and underlying proof techniques can help explain the phenomenon of learning.
Among the tightest known bounds
are \emph{PAC-Bayes} \citep{mcallester1999pacaveraging} and \emph{test set} bounds \citep{langford2002quantitatively}. In this paper, we investigate their numerical tightness when applied to small datasets ($N \approx 30$--$60$ datapoints). The comparison between PAC-Bayes and test set bounds is particularly interesting in this setting as one cannot discard data to compute a test set bound without significantly harming post-training performance due to a reduced training set size. PAC-Bayes on the other hand provides valid bounds while using all of the data for learning, since it provides bounds that hold uniformly. The small data setting can also be quite different from the big data setting, as lower-order terms in PAC-Bayes bounds have a non-negligible contribution, and the detailed structure of the bound becomes important. 

Fortunately, we do not have to study each PAC-Bayes bound separately: remarkably, \citet{germain2009pac} showed that a wide range of bounds can be obtained as special cases of a single \emph{generic PAC-Bayes theorem} that captures the central ideas of many PAC-Bayes proofs (see also \citet{begin2016pac}). This theorem has a free parameter: it holds for any convex function, $\Delta$. By choosing $\Delta$ appropriately, one can recover the well-known bounds of \citet{langford2001bounds}, \citet{catoni2007pac} and other bounds. We focus on two questions related to this set-up. \emph{First}, \emph{what is the tightest bound achievable by any convex function $\Delta$}?
An answer would characterise the limits of the generic PAC-Bayes theorem, and thereby of a wide range of bounds, by telling us how much improvement could be obtained before new ideas or assumptions are needed.
\emph{Second}, since test set bounds are the \textit{de facto} standard for larger datasets, but PAC-Bayes has benefits when $N$ is small, we ask: \emph{in the small data regime, can PAC-Bayes be tighter than test set bounds?}

In \cref{sec:characterising-generic-bounds}, \cref{thm:nonsep-delta}, we show that in the (artificial) case  when $\Delta$ can be chosen depending on the dataset (without taking a union bound),
the tightest version of the generic PAC-Bayes theorem is obtained by one of the \emph{Catoni bounds} \citep{catoni2007pac}. In the more realistic case when $\Delta$ must be chosen \emph{before} sampling the dataset, we do not fully characterise the tightest bound, but in \cref{cor:optimistic-MLS} we \emph{lower bound} the tightest bound achievable (in expectation) with any $\Delta$.
We also provide numerical evidence in \cref{fig:numerical_evidence_deterministic_case} that suggests this lower bound can in some cases be attained, by flexibly parameterising a convex function $\Delta$ with a constrained neural network. Interestingly, this lower bound coincides with removing a lower-order term from the \citet{langford2001bounds} bound (something that \citet{langford2002quantitatively} conjectured was possible), and relaxes to the well-known \emph{Chernoff test set bound} (see \cref{lem:kl-chernoff-bound} below) when the PAC-Bayes posterior is equal to the prior.

In \cref{sec:classification}, we investigate the tightness of PAC-Bayes and test set bounds in synthetic 1D classification. The goal of this experiment is to find out how tight the bounds could be made in principle. We use meta-learning to adapt all aspects of the bounds and learning algorithms, producing meta-learners that are trained to optimise the value of the bounds on this task distribution. We find that, in this setting, PAC-Bayes can be competitive with the Chernoff test set bound, but is outperformed by the \emph{binomial tail test set bound}, of which the Chernoff bound is a relaxation. This suggests that, for standard PAC-Bayes to be quantitatively competitive with the best test set bounds on small datasets, a new proof technique leading to bounds that gracefully relax to the binomial tail bound is required. Code to reproduce all experiments can be found at \url{https://github.com/cambridge-mlg/pac-bayes-tightness-small-data}.

\section{Background and Related Work} \label{sec:background-related-work}
\vspace{-7pt}
    We consider supervised learning. 
    Let $\mathcal{X}$ and $\mathcal{Y}$ denote the \emph{input space} and \emph{output space}, and let $\datapoints=\inputs \times \outputs$. Assume there is an (unknown) probability measure\footnote{We will colloquially refer to measures on sets without specifying a $\sigma$-algebra. We implicitly assume functions are measurable with respect to the $\sigma$-algebras on which the relevant measures are defined.} $D$ over $\datapoints$, with the dataset $S \sim D^N$. 
  Denote the \emph{hypothesis space} by $\hypotheses \subseteq \outputs^\inputs$. A learning algorithm is then a map $\datapoints^\setsize \to \hypotheses$. In PAC-Bayes, we also consider maps $\datapoints^\setsize \to \mathcal{M}_1(\hypotheses)$, where $\mathcal{M}_1$ is the set of probability measures on its argument. The performance of a hypothesis $h \in \mathcal{H}$ is measured by a \emph{loss function} $\ell\colon \datapoints \times \hypotheses \to [0,1]$. 
    The \emph{(generalisation) risk} of $h$ is $\smash{R_D(h)\coloneqq\mathbb{E}_{(x,y)\sim D}[\ell((x,y), h)]}$ and its \emph{empirical risk on $\set$} is $R_S(h)\coloneqq \smash{\frac{1}{N}\sum_{(x,y)\in S}\ell((x,y), h)}$. For $Q\in \mathcal{M}_1(\mathcal{H})$ its \emph{(generalisation Gibbs) risk} is $\gibbsR_D(Q) \coloneqq \mathbb{E}_{h\sim Q}[R_D(h)]$ and its \emph{empirical (Gibbs) risk} is $\gibbsR_S(Q) \coloneqq \mathbb{E}_{h\sim Q}[R_S(h)]$.
    In PAC-Bayes, we usually fix a \emph{prior} $P \in \mathcal{M}_1(\mathcal{H})$, chosen without reference to $S$ and learn a \emph{posterior} $Q\in \mathcal{M}_1(\mathcal{H})$ which can depend on $S$. 
    The \emph{KL-divergence} between $Q$ and $P$ is defined as $\mathrm{KL}(Q\|P) = \int \log \frac{\mathrm{d}Q}{\mathrm{d}P} \, \mathrm{d}Q$ if $Q \ll P$ and $\infty$ otherwise.
    Let $\mathcal{C}$ denote the set of proper, convex, lower semicontinuous (l.s.c.) functions $\mathbb{R}^2 \to \mathbb{R}\cup \{ +\infty \}$; if a convex function's domain is a subset of $\mathbb{R}^2$, extend it to all of $\mathbb{R}^2$ with the value $+\infty$.
    See \cref{app:convex} for more details on convex analysis,
    which we use in \cref{sec:characterising-generic-bounds}.
    
\shortsubsection{Test Set Bounds.}
        \emph{Test set bounds} rely on a subset of data which is not used to select the hypothesis, called a \emph{test set} or \emph{held-out set}.
        Let $\set = \trainset \cup \testset$, with $|\set| = \setsize$, $|\trainset|=\trainsetsize$ and $|\testset|=\testsetsize$.
        In \cref{lem:bin-inv-bound,lem:kl-chernoff-bound}, we assume $h$ is chosen independently of $\testset$. 
        For the zero-one loss, $\ell((x,y), h) \coloneqq \mathbbm{1}[h(x) \neq y]$, we have that $\testsetsize R_{\testset}(h)$ is a binomial random variable with parameters $(\testsetsize, R_D(h))$. This leads to the following simple bound, which, for $\ell \in \{0, 1\}$, is tight among test set bounds:
        \vspace{\beforethmspace}\begin{theorem}[Binomial tail test set bound, {\citealp[, Theorem 3.3]{langford2005tutorial}}] \label{lem:bin-inv-bound}
        ~\\Let $\overline{e}(M,k, \delta)\coloneqq \sup\big\{p: \delta \leq \sum_{i=0}^{ k } \binom{M}{i}p^{i}(1-p)^{M-i} \big\}$. For any $h\in \hypotheses$, $\ell \in \{0,1\}$ and $\delta \in (0,1)$,\vspace{-0.5em}
        \begin{equation}
          \textstyle \mathrm{Pr} \Big(R_D(h) \leq \overline{e}(\testsetsize,\testsetsize R_{\testset}(h), \delta) \Big) \geq 1-\delta.
        \end{equation}\vspace{-1.5em}
        \end{theorem}
Often, looser bounds with a simpler form are applied. These can be obtained via the Chernoff method:\vspace{\beforethmspace}
       \begin{theorem}[Chernoff test set bound, {\citealp[, Corollary 3.7]{langford2005tutorial}}] \label{lem:kl-chernoff-bound}
        ~\\For $q,p\in [0,1]$, let $\mathrm{kl}(q,p)\coloneqq q\log \frac{q}{p} + (1-q)\log \frac{1-q}{1-p}$. For any $h\in \hypotheses$, $\ell \in  [0, 1]$, and $\delta \in (0,1)$,\vspace{-0.5em}
        \begin{equation}
         \textstyle \mathrm{Pr}\left(\mathrm{kl}(R_{\testset}(h), R_D(h)) \leq \frac{1}{\testsetsize} \log \frac{1}{\delta} \right) \geq 1-\delta.
        \end{equation}~
        \end{theorem}\vspace{-2em}

\shortsubsection{PAC-Bayes Bounds.} %
The PAC-Bayes approach bounds the generalisation Gibbs risk of \emph{stochastic classifiers}, and does not require discarding data, as all the data can be used to choose the posterior, while still obtaining a valid generalisation bound.
Since the seminal paper of \citet{mcallester1999pacaveraging}, a large variety of PAC-Bayes bounds have been derived. \Citet{germain2009pac} prove a very general form of the PAC-Bayes theorem which encompasses many of these (see also \citet{begin2016pac,rivasplata2020pac}). Their proof technique consists of a series of inequalities shared by PAC-Bayes proofs (Jensen's, change of measure, Markov's, supremum over risk\footnote{The supremum over risk step was introduced in \citet{begin2016pac}, although for certain $\Delta$ it can be omitted.}), and reveals their common structure. Thus understanding the properties of this generic theorem can give insight into many PAC-Bayes bounds at once:

\vspace{\beforethmspace}\begin{theorem}[Generic PAC-Bayes theorem, \citet{germain2009pac,begin2016pac}]\!\!\!\footnote{We state a simpler version of their result WLOG, absorbing a free parameter into the function $\Delta$.}
\label{thm:begin-bound}
Fix $P \in \mathcal{M}_1(\hypotheses)$, $\ell \in [0,1]$, $\delta \in (0,1)$, and $\Delta$ a proper, convex, l.s.c.\ function $[0,1]^2 \to \mathbb{R}\cup \{ +\infty \}$.
Then
\begin{equation}
    \textstyle \mathrm{Pr}\left( (\forall Q) \,\, \Delta(\gibbsR_S(Q), \gibbsR_D(Q)) \leq \frac{1}{\setsize} \left[ \mathrm{KL}(Q\|P) + \log \frac{\mathcal{I}_{\Delta} (\setsize)}{\delta} \right]\right) \geq 1-\delta,
\end{equation}
where $\smash{\mathcal{I}_\Delta(N) \coloneqq \sup_{r \in [0, 1]} \sum_{k=0}^N \binom{N}{k}r^k(1-r)^{N - k} e^{N \Delta(k/N, r)}.}$
\end{theorem}
\vspace{\beforethmspace}\begin{remark}
We lose no generality in assuming $\Delta(q, \cdot)$ is monotonically increasing for all $q \in [0, 1]$, i.e.~for any convex $\Delta$ we can define a $\Delta'$ that is monotonically increasing in its second argument and produces a bound that is at least as tight as the bound produced by $\Delta$.
See \cref{app:monotonicity} for a proof.
\end{remark}
Note that the PAC-Bayes bound holds simultaneously for all posteriors $Q$, and hence is valid even when $Q$ is chosen by minimising the bound. For completeness, we provide a proof of \cref{thm:begin-bound} in \cref{app:proof-begin-bound}. Following \citet{germain2009pac}, we briefly recap some of the bounds that can be recovered as special cases (or looser versions) of \cref{thm:begin-bound}. Setting $\Delta(q, p) = C_\beta(q,p) \coloneqq -\log(1+p(e^{-\beta}-1)) - \beta q$ for $\beta > 0$, we recover the \emph{Catoni bounds}:
\vspace{\beforethmspace}\begin{corollary}[{\citealp[, Theorem 1.2.6]{catoni2007pac}}]\label{cor:catoni}
For any $\beta >0$,
\begin{align}
\label{eqn:catoni-bound}
    \textstyle \smash{\mathrm{Pr}\left((\forall Q) \,\, \gibbsR_D(Q) \leq \frac{1}{1-e^{-\beta}} \left[ 1 - \exp \left( - \beta \gibbsR_S(Q) - \frac{1}{N} \Big(\mathrm{KL}(Q\|P) + \log \frac{1}{\delta} \Big) \right) \right]\right) \geq 1-\delta.}
\end{align}
\end{corollary}
This specifies a bound for every value of $\beta > 0$. 
If we instead choose $\Delta(q, p) = \mathrm{kl}(q, p)$, we obtain the bound of \citet{langford2001bounds}, also called the \emph{PAC-Bayes-kl bound}, but with the slightly sharper dependence on $\setsize$ established by \citet{maurer2004note}:
\vspace{\beforethmspace}\begin{corollary}[{\citealp[, Theorem 3]{langford2001bounds}}, {\citealp[, Theorem 5]{maurer2004note}}] \label{cor:maurer}
\begin{equation} \label{eqn:maurer-bound} \textstyle
    \textstyle \smash{\mathrm{Pr}\left((\forall Q) \,\, \mathrm{kl}(\gibbsR_{\set}(Q), \gibbsR_D(Q)) \leq \frac{1}{N} \left[ \mathrm{KL}(Q\|P) + \log \frac{2\sqrt{\setsize}}{\delta} \right]\right) \geq 1-\delta.}
\end{equation}
\end{corollary}

\Cref{cor:maurer} is actually very slightly looser than \cref{thm:begin-bound} with $\Delta = \mathrm{kl}$, since \citet{maurer2004note} upper bounds 
$\mathcal{I}_{\mathrm{kl}} (N)$ by $\smash{2\sqrt{N}}$ 
using Stirling's formula.\footnote{\citet{maurer2004note} only proves this bound for $N \geq 8$, but the cases where $1 \leq N \leq 7$ can be easily verified numerically \citep[Lemma 19]{germain2015majorityvote}.}
The Catoni and PAC-Bayes-kl bounds are among the tightest PAC-Bayes bounds known and have been applied in settings where numerical tightness is key, such as obtaining generalisation bounds for stochastic neural networks \citep{dziugaite2017computing,zhou2018non}.
Many other bounds can be obtained by loosening these bounds. Applying Pinsker's inequality $\mathrm{kl}(q, p) \geq 2(q-p)^2$ to \cref{eqn:maurer-bound} yields the ``square-root'' version of the PAC-Bayes theorem \citep{mcallester1999pacaveraging,mcallester2003pac}. The ``PAC-Bayes-$\lambda$'' \citep{thiemann2017strongly} and ``PAC-Bayes-quadratic'' bounds \citep{rivasplata2019pac} can be derived as loosened versions of the PAC-Bayes-kl bound using the inequality $\mathrm{kl}(q, p) \geq (q - p)^2 / (2p)$, valid for $q < p$. The ``linear'' bound in \citet{mcallester2013pac} can be derived by loosening the Catoni bound using: $C_\beta(q, p) \leq A \implies p \leq \frac{1}{1 - \beta / 2} (q + \frac{1}{\beta}A)$, which is valid for $\beta \in (0, 2)$. 


\shortsubsection{How Tight Are PAC-Bayes Bounds?}
A fundamental question we can ask about a generalisation bound is how tight it is, and whether it can be tightened. Comparing the PAC-Bayes-kl and Chernoff test set bounds when $Q=P$ (so the PAC-Bayes bound essentially becomes a test set bound) shows they are identical except for a $\smash{\log (2\sqrt{N})/N}$ on the RHS of the PAC-Bayes-kl bound. Whether this term (or similar discrepancies between PAC-Bayes and \emph{Occam bounds} \citep[Corollary 4.6.2]{langford2002quantitatively}; see \cref{app:occam}) can be removed has been an open question since \citet[Problem 6.1.2]{langford2002quantitatively}. \Citet{maurer2004note} reduced this term to its current form, improving on work by \citet{langford2001bounds}. 
Interestingly, \citet[Proposition 2.1]{germain2009pac} shows that the expression obtained by dropping  $\log \smash{(2\sqrt{N})/N}$ from the PAC-Bayes-kl bound is identical to that obtained by \emph{illegally}\footnote{That is, optimising $\beta$ depending on the dataset $S$ without taking a union bound.} minimising the Catoni bound with respect to $\beta$;
\citet[Theorem 1.2.8]{catoni2007pac} shows that a union bound can be used to, in a legal way, approximately optimise with respect to $\beta$ at the cost of an additional lower order term.
The Chernoff test set bound is itself a looser version of the binomial tail bound, raising the question of whether a PAC-Bayes bound can be found that reduces to the binomial tail bound when $Q=P$. We provide new insights into these problems in \cref{sec:characterising-generic-bounds}.

Researchers have also compared PAC-Bayes bounds numerically on actual learning problems. \Citet{langford2005tutorial} and \citet{germain2009pac} were able to obtain reasonable guarantees on small datasets. However, \citet{langford2005tutorial} found that on datasets with $N\approx 145$, PAC-Bayes was outperformed by test set bounds. 
\Citet{langford2001not,dziugaite2017computing,perez2020tighter} provide non-vacuous bounds for neural networks using PAC-Bayes. Even so, \citet{dziugaite2021role} states that tighter bounds would be obtained using a test set instead. In \cref{sec:classification} we find that if the bounds and learning algorithms are optimised for a task distribution, PAC-Bayes can be tight enough to compete with the Chernoff test set bound, but not the binomial tail test set bound.

\begin{figure}[t]
    \centering\small
    \begin{tikzpicture}[
        line/.style = {
            thick,
            ->,
            > = {
                Triangle[length=1.5mm, width=1.5mm]
            }
        },
    ]
        \node (delta) at (0, 0) {$\Delta$};
        \node [align=center] (bound) at ($(delta) + (0, -1.2)$) {$\E[\overline{p}_\Delta]$\\\scriptsize (\cref{thm:begin-bound})};
        \draw [line] (delta) -- (bound);
        \node [plotorange] (existsdelta) at ($(delta) + (2.5, 0)$) {$\exists\,\Delta\,?$};
        \node [align=center] (optimisticmls) at ($(existsdelta) + (0, -1.2)$) {\footnotesize conjectured PAC-B.-kl \\\scriptsize (\cref{cor:optimistic-MLS})};
        \draw [line, plotorange] (existsdelta) -- (optimisticmls);
        \path (bound) -- node [pos=0.5] {$\ge$} (optimisticmls);
        \node [align=center] (chernoff) at ($(optimisticmls) + (3.5, 0)$) {Chernoff \\\scriptsize (\cref{lem:kl-chernoff-bound})};
        \node [align=center] (binom) at ($(chernoff) + (2.25, 0)$) {Binomial tail \\\scriptsize (\cref{lem:bin-inv-bound})};
        \draw [thick, dashed] (-4, -1.75) rectangle (7, 0.35);
        \path (-4, 0.35) -- node [pos=0.5, above] {Potential limits of the generic PAC-Bayes theorem \scriptsize (\cref{thm:begin-bound})} (7, 0.35);
        \draw
            [thick, dotted]
            (5, -1.85) rectangle (9.35, 0.25);
        \path (5, -1.85)
            -- node [fill=white, pos=0.31] {$\phantom{q}$} (5, 0.25)
            -- node [midway, anchor=south, fill=white] {Test set bounds} (9.35, 0.25);
        \node [align=center] (catoni) at ($(delta) + (-2.5, -0.25)$) {
            Catoni: $\Delta=C_\beta$ \\\scriptsize (\cref{cor:catoni})
        };
        \node [align=center] (mls) at ($(delta) + (-2.5, -1.2)$) {
            PAC-B.-kl: $\Delta=\kl$ \\
            {\scriptsize (\cref{cor:maurer})}
        };
        \draw [line] (catoni.east) -- (bound);
        \draw [line] (mls.east) -- (bound);
        \path (chernoff) -- node [pos=0.525, fill=white] {$\ge$} (binom);
        \draw [line] (optimisticmls) -- node [pos=0.5, anchor=south, fill=white, yshift=2pt] {$Q\!=\!P$} (chernoff);
        \node [plotred] (nodelta) at ($(existsdelta) + (5.75, 0)$) {$\nexists\, \Delta$};
        \draw [line, plotred] (nodelta) -- node [pos=0.5, anchor=west] {$Q\!=\!P$} (binom);
    \end{tikzpicture}
    \caption{\small
        \textbf{Illustration of the relationship between various PAC-Bayes and test set bounds}; see \cref{sec:characterising-generic-bounds}.
        It is unclear if there always exists a $\Delta$ that recovers the conjectured PAC-Bayes-kl bound (and hence the Chernoff bound when $Q\!=\!P$; see \cref{open:lower-bound}), but there certainly does not exist a $\Delta$ that recovers the Binomial tail bound when $Q\!=\!P$.
    }
    \label{fig:relationships}
    \vspace{-1em}
\end{figure}
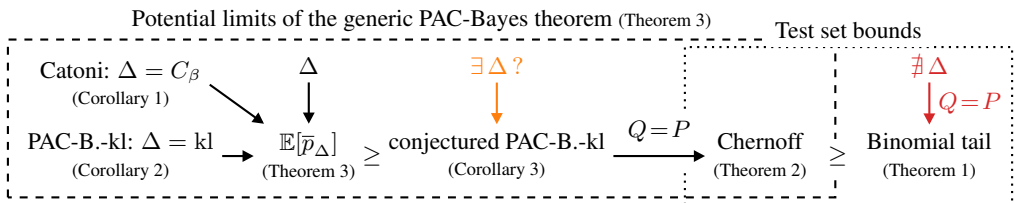

\section{Characterising the Limits of the Generic PAC-Bayes Proof Technique}\label{sec:characterising-generic-bounds}
This section establishes our main theoretical contributions, which characterise the limits of the generic PAC-Bayes theorem (\cref{thm:begin-bound}).
For a convex $\Delta \in \mathcal{C}$, \cref{thm:begin-bound} gives a high-probability upper bound on $\Delta(\overline{R}_S(Q), \overline{R}_D(Q))$.
Define
$
    B[f, y] \coloneqq \sup\,\{p \in [0,1] : f(p) \le y\}
$ for $f \colon [0,1] \to \R$ and $y \in \R$, where we take $\sup \varnothing = 1$.
This upper bound (\cref{thm:begin-bound}) can be ``inverted'' to obtain a high-probability upper bound on $\overline{R}_D(Q)$:
with probability at least $1 - \delta$, for all $Q \in \mathcal{M}_1(\hypotheses)$,
\begin{equation} \label{eq:generic-gen-bound}
    \textstyle \overline{R}_D(Q) \le \overline{p}_\Delta
    \quad\text{where}\quad
    \overline{p}_\Delta \coloneqq B\big[\Delta(\overline{R}_S(Q), \vardot), \tfrac1{N}\big(\!\operatorname{KL}(Q\|P)+\log\tfrac{\mathcal{I}_\Delta(N)}{\delta}\big)\big].
\end{equation}
Since \eqref{eq:generic-gen-bound} holds for all $\Delta \in \mathcal{C}$,
a natural question is:
\emph{Which $\Delta$ minimises $\overline{p}_\Delta$?}
This would characterise how tight, numerically, PAC-Bayes theorems can be made without introducing ideas beyond those needed to prove the bounds stated in \cref{sec:background-related-work}.
Before considering the case when $\Delta$ is selected before observing $S \sim D^N$, we first characterise the optimal $\Delta$ in the simplified scenario where $\Delta$ can depend on the dataset $S$ and the posterior $Q$ (\cref{thm:nonsep-delta}).
This setting is artificial, since choosing $\Delta$ based on $S$ (without taking a union bound) does not yield a valid generalisation bound.
However, using \cref{thm:nonsep-delta} as a building block, we later derive a \emph{lower bound} on the best possible generic PAC-Bayes bound (in expectation) in the more realistic case when we cannot choose $\Delta$ based on $S$ (\cref{cor:optimistic-MLS}).
We then connect this lower bound to various existing PAC-Bayes and test set bounds. An overview is shown in \cref{fig:relationships}.
We now state our first result:


\vspace{\beforethmspace}\begin{theorem} \label{thm:nonsep-delta}
    Given any fixed dataset $S$ and any $Q, P \in \mathcal{M}_1(\mathcal{H})$,
    the tightest Catoni bound is as tight as the tightest bound possible within the generic PAC-Bayes theorem (\cref{thm:begin-bound}).
    Precisely, let $\Delta\in \mathcal{C}$ and $\delta \in (0,1)$.
    Choose some fixed values for $\overline{R}_S(Q) \eqqcolon q \in [0, 1]$ and $\operatorname{KL}(Q\| P) \eqqcolon \operatorname{KL} \in [0, \infty)$.
    If $q > 0$, then there exists a $\beta \in (0,\infty)$ such that $\overline{p}_\Delta \ge \overline{p}_{C_\beta}$, where $\overline{p}$ is defined in \cref{eq:generic-gen-bound}.
    Moreover, if $q = 0$, then $\overline{p}_\Delta \ge \lim_{\beta \to \infty} \overline{p}_{C_{\beta}}$.
\end{theorem}

\begin{remark} \label{rem:nonsep-delta}
    By \cref{thm:nonsep-delta}, 
    for all $\Delta \in \mathcal{C}$, we have
    $\overline{p}_\Delta \ge \inf_{\beta > 0} \smash{\overline{p}_{C_{\beta}}}$, and, by Proposition 2.1 of \citet{germain2009pac}, $\inf_{\beta > 0} \smash{\overline{p}_{C_{\beta}}} = B[\kl(q, \vardot), \tfrac1N(\operatorname{KL} + \log\tfrac1{\delta})]$.
    Hence, for all $\Delta \in \mathcal{C}$, it holds that $\overline{p}_\Delta \ge B[\kl(q, \vardot), \tfrac1N(\operatorname{KL} + \log\tfrac1{\delta})]$, which is also shown directly in the proof of \cref{thm:nonsep-delta} (\cref{eq:illegal-maurer-proof}).
    Note that optimising $\beta$ in this way is illegal in the general case when the dataset $S$ (and hence $q$ and $\operatorname{KL}$) is stochastic, and would typically require a union bound to be valid.
\end{remark}

\begin{figure}[t]
    \centering
    \begin{subfigure}[t]{0.235\linewidth}
        \includegraphics[width=1\linewidth, trim={0 3em 0 0}]{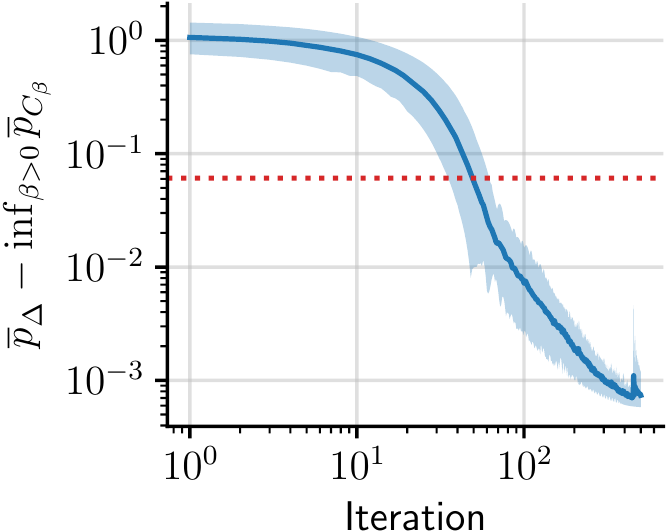}
        \caption{\hspace{8em}
        }
        \label{fig:det1-1}
    \end{subfigure}
    \hfill
    \begin{subfigure}[t]{0.235\linewidth}
        \includegraphics[width=1\linewidth, trim={0 3em 0 0}]{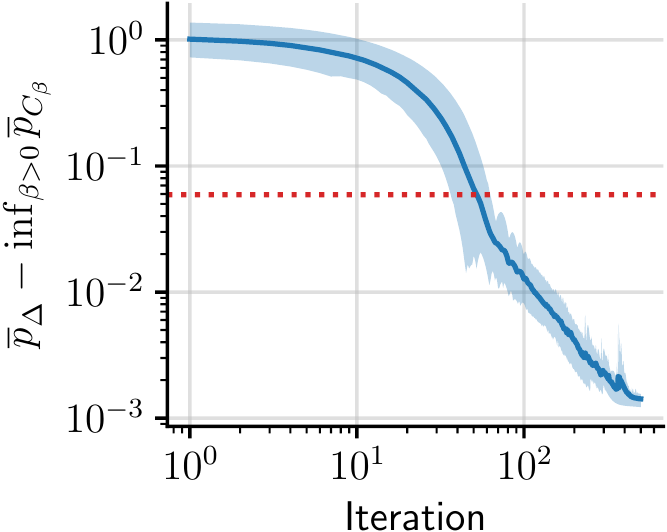}
        \caption{\hspace{8em}
        }
        \label{fig:det1-2}
    \end{subfigure}
    \hfill
    \begin{subfigure}[t]{0.235\linewidth}
        \includegraphics[width=1\linewidth, trim={0 3em 0 0}]{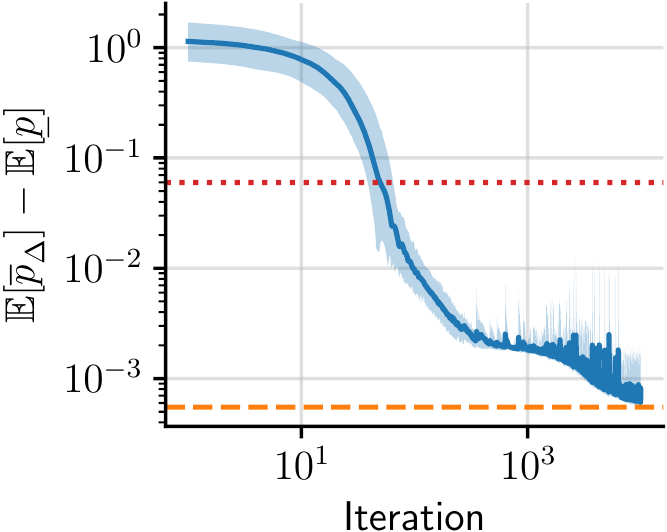}
        \caption{\hspace{8em}
        }
        \label{fig:stoch1}
    \end{subfigure}
    \hfill
    \begin{subfigure}[t]{0.235\linewidth}
        \includegraphics[width=1\linewidth, trim={0 10cm 0 0}]{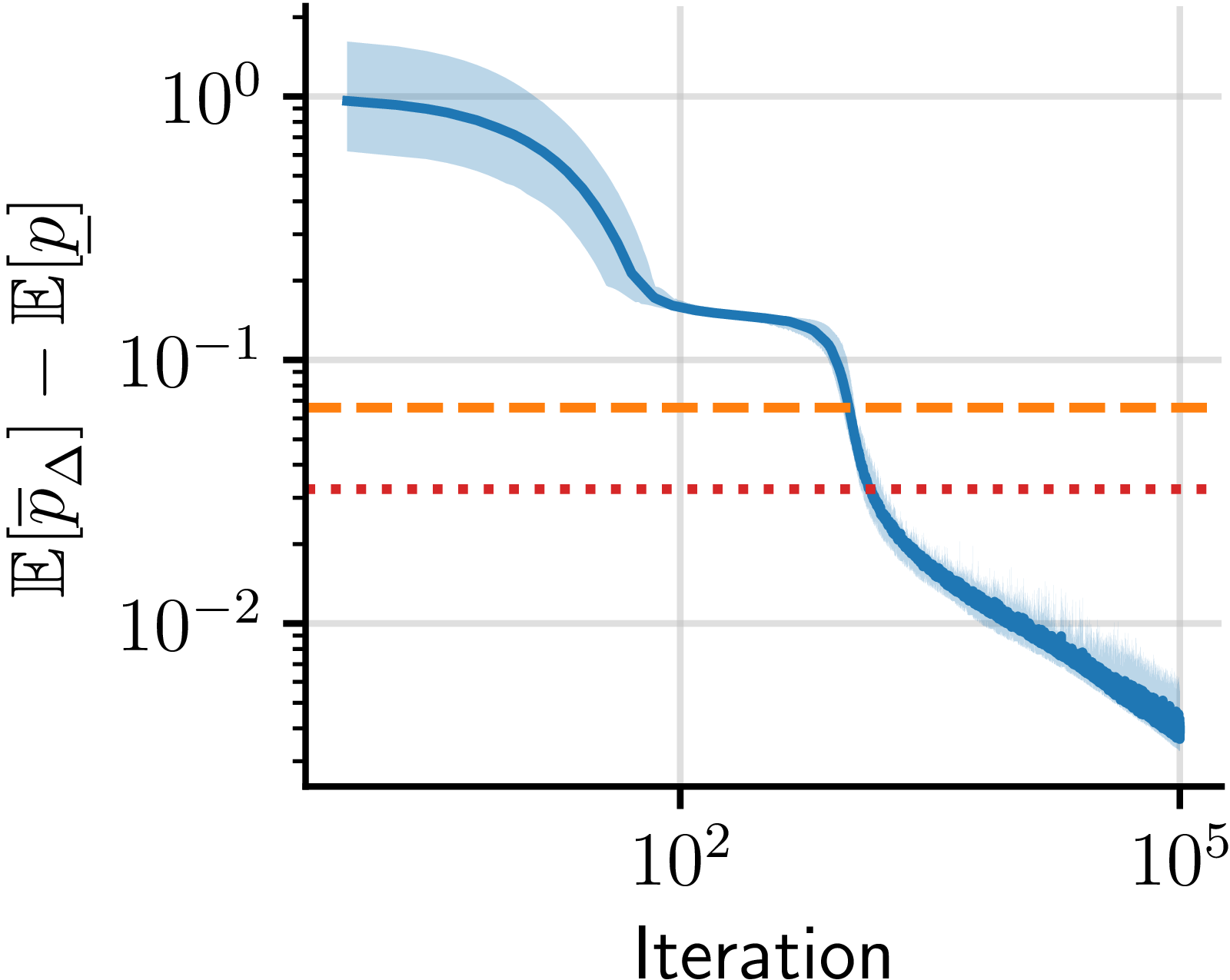}
        \caption{\hspace{8em}
        }
        \label{fig:stoch2}
    \end{subfigure}\\
    \caption{\small
        \textbf{The tightest Catoni bound is the optimal generic PAC-Bayes bound for a fixed dataset, but not the optimal expected bound for a random dataset.} We optimise a convex function with $H$ hidden units to minimise $\overline{p}_\Delta$ with $\delta=0.1, N=30$.
        \textbf{(a) and (b) consider fixed $q$ and $\operatorname{KL}$} (precise values below) and show the difference with the best Catoni bound (\cref{thm:nonsep-delta});
        \textbf{(c) and (d) consider random $q$ and $\operatorname{KL}$} and show the \emph{expected} difference with the conjectured PAC-Bayes-kl bound (\cref{cor:optimistic-MLS}).
        Shaded regions show the minimum and maximum over ten initialisations.
        All plots show the PAC-Bayes-kl bound ({\color{plotred} dotted red}) and (c) and (d) show the optimal Catoni bound with parameter $\beta^*$ ({\color{plotorange} dashed orange}).
        All runs quickly converged to non-vacuous values.
        (a): $(q, \operatorname{KL}) = (2\%, 1)$, $\beta^*\approx2.24$, $H = 256$.
        (b): $(q, \operatorname{KL}) = (5\%, 2)$, $\beta^*\approx1.84$, $H = 256$.
        (c): $(q, \operatorname{KL}) \in \{(2\%, 1),(5\%, 2)\}$ uniformly, $\beta^*\approx1.99$, $H = 512$.
        (d): $(q, \operatorname{KL}) \in \{(30\%, 1),(40\%, 50)\}$ uniformly, $\beta^*\approx2.32$, $H = 1024$.
    }
    \label{fig:numerical_evidence_deterministic_case}
    \vspace{-1em}
\end{figure}

We defer the proof of \cref{thm:nonsep-delta} to the end of this section.
We numerically verify \cref{thm:nonsep-delta} by optimising $\overline{p}_\Delta$ with respect to an arbitrary convex $\Delta$ for various settings of fixed $q$ and $\operatorname{KL}$.
To parametrise a convex $\Delta$, we use a one-hidden-layer neural network with positive weights at the output layer and softplus nonlinearities.
The inversion performed by $B$ is approximated numerically by discretising the second argument of $\Delta$ and detecting an upcrossing.
Gradients are then approximated using the inverse function theorem:
$ \textstyle
    \frac{\sd}{\sd \theta} B[f_\theta,c(\theta)]
    = {(\d_\theta c(\theta) - \d_\theta f_\theta(x))}/{\d_x f_\theta(x)}.
$
See \cref{app:learning-convex-function} for details.\footnote{Numerical inversion of $\Delta$ when $\Delta=\mathrm{kl}$ has been considered by many authors, including \citet{dziugaite2017computing} who use Newton's method and \citet{pmlr-v87-majumdar18a} who propose using convex optimisation methods. However, to our knowledge, the specific inversion algorithm we propose for general convex $\Delta$, along with the method for backpropagating through the inverse, are novel in the PAC-Bayes setting.}
\Cref{fig:det1-1,fig:det1-2} show the difference between the numerically optimised $\Delta$ and the best Catoni bound for two settings of fixed $q \in [0,1]$ and $\operatorname{KL} \in [0, \infty)$.
In both cases, $\overline{p}_{\Delta} - \inf_{\beta > 0} \overline{p}_{C_\beta}$ appears to converge to zero from above, as expected from \cref{thm:nonsep-delta}.
Interestingly, \cref{app:learning-convex-function} shows that the learned $\Delta$ can deviate substantially from $C_\beta$, suggesting that there are choices for $\Delta$ besides Catoni's which achieve $\inf_{\Delta \in \mathcal{C}} \overline{p}_\Delta$.

For any \emph{fixed} dataset $S$, \cref{thm:nonsep-delta} states that the tightest bound is one of the Catoni bounds; precisely: $\inf_{\Delta \in \mathcal{C}} \overline{p}_\Delta = \inf_{\beta > 0} \smash{\overline{p}_{C_\beta}}$. Note that the optimal value of $\beta$ may depend on the dataset $S$.
The more interesting question is whether, when $S \sim D^N$ is sampled randomly, one of the Catoni bounds can still achieve the tightest bound (in expectation) for a \emph{single} value of $\beta$ that is chosen \emph{before} sampling $S$.
The answer is no:
\cref{fig:stoch2} gives a numerical counterexample where $\textcolor{plotblue}{\inf_{\Delta \in \mathcal{C}} \E[\overline{p}_\Delta]} <  \textcolor{plotred}{\E[\smash{\overline{p}_{\kl}}]} < \textcolor{plotorange}{\inf_{\beta > 0} \E[\smash{\overline{p}_{C_\beta}}]}$.
Since the Catoni family of bounds cannot generally achieve the tightest bound in expectation, which $\Delta$ do? And how tight is $\inf_{\Delta \in \mathcal{C}} \E[\overline{p}_\Delta]$?
Whilst we do not have a full answer, we establish a simple lower bound on $\inf_{\Delta \in \mathcal{C}} \E[\overline{p}_\Delta]$.
Define the \emph{conjectured PAC-Bayes-kl bound} $\smash{\underline{p}}$ as the quantity from \cref{rem:nonsep-delta}, which equals the PAC-Bayes-kl bound \emph{without the $\smash{\tfrac1N}\log\mathcal{I}_{\kl}(N)$ term on the RHS:}
\begin{equation}
    \underline{p} \coloneqq B[\kl(\overline{R}_S(Q), \vardot), \tfrac1N(\operatorname{KL}(Q\|P) + \log\tfrac1{\delta})].
\end{equation}
The conjectured PAC-Bayes-kl bound has \emph{not} been proven to be a valid generalisation bound.
When $S \sim D^N$ is random,
\Cref{rem:nonsep-delta} tells us that $\inf_{\Delta \in \mathcal{C}} \overline{p}_\Delta = \underline{p}$ a.s. Taking expectations and interchanging the expectation and infimum yields the following corollary:

\vspace{\beforethmspace}\begin{corollary} \label{cor:optimistic-MLS}
    Consider the setting from \cref{thm:begin-bound}.
    Then the expected conjectured PAC-Bayes-kl bound $\E[\underline{p}]$ gives a lower bound on all expected generalisation bounds obtained through the generic PAC-Bayes theorem (\cref{thm:begin-bound}). That is, for any distribution over datasets, any prior, and any learning algorithm, 
    \begin{equation}\label{eq:inequality_in_expectation} \textstyle
        \inf_{\Delta\in \mathcal{C}} \E[\overline{p}_\Delta]
        \ge
        \mathbb{E}[\underline{p}]
    \end{equation}
    Moreover, there exists a distribution over datasets, a prior, and a posterior such that equality holds.
    For example, let $(x, y)$ be constant almost surely, which reduces to the setting of \cref{thm:nonsep-delta}.
    Note that in \eqref{eq:inequality_in_expectation}, $\Delta$ is chosen \emph{not} depending on $S$, which leads to a valid generalisation bound on the LHS.
\end{corollary}

\Cref{fig:relationships} shows how \cref{cor:optimistic-MLS} fits into the picture so far.
The conjectured PAC-Bayes-kl bound is at least as tight as the bound achieved by any $\Delta$, but \cref{cor:optimistic-MLS} does not establish the existence of a $\Delta$ which achieves it.
\Cref{cor:optimistic-MLS} has practical utility: the conjectured PAC-Bayes-kl bound can be used to prove optimality of a choice of $\Delta$.
Specifically, if a practitioner computes a valid bound based on the generic PAC-Bayes theorem, and finds that it is close to the conjectured PAC-Bayes-kl bound, they can be assured by \cref{cor:optimistic-MLS} that they would not have gotten a much better bound (in expectation) with any other choice of $\Delta$.
Conversely, the conjectured PAC-Bayes-kl bound can quantify potential slack in the bound due to a suboptimal choice of $\Delta$.
\Cref{app:worked-example-optimistic-MLS} considers an example of this application of \cref{cor:optimistic-MLS} in the simplified scenario where $\overline{R}_S(Q) = \tfrac12$ almost surely.
%
%

The conjectured PAC-Bayes-kl bound also recovers the Chernoff test set bound (\cref{lem:kl-chernoff-bound}) when setting $Q = P$.
Since the binomial tail bound (\cref{lem:bin-inv-bound}) is strictly tighter than the Chernoff bound, this shows there does not exist a $\Delta$ such that the generic PAC-Bayes bound (\cref{thm:begin-bound}) recovers the Binomial tail bound when $Q=P$;
this is illustrated in \cref{fig:relationships}.
What is unclear, however, is whether there always exists a $\Delta$ such that \cref{thm:begin-bound} recovers the Chernoff test set bound;
or, alternatively, such that the conjectured PAC-Bayes-kl bound is attained.
A positive answer to the latter would establish that the conjectured PAC-Bayes-kl bound is a valid generalisation bound.\footnote{
    By \cref{rem:nonsep-delta}, $\overline{p}_\Delta \ge \underline{p}$, so $\E[\overline{p}_\Delta] = \E[\underline{p}]$ implies that $\overline{p}_\Delta = \underline{p}$ a.s., meaning that $\underline{p}$ is a valid gen.\ bound.
}
As a first piece of evidence, the traces from \cref{fig:stoch1,fig:stoch2} suggest that a convex function could actually achieve $\E[\underline{p}]$;
see \cref{app:additional-numerical-results} for more traces.
We leave a full resolution of this question as an open problem; see \cref{sec:conclusion}. Interestingly, \cref{fig:stoch1} shows that a Catoni bound is sometimes \emph{nearly} optimal even in the stochastic case; we will see another example of this in \cref{fig:gen_bounds_1d_GNP_post_opt}.

We end this section with the proof of \cref{thm:nonsep-delta}.
Recall that the Catoni family of bounds follows from \cref{thm:begin-bound} by considering $
    \Delta(q,p) = C_\beta(q, p)
    \coloneqq \mathcal{F}_\beta(p) - \beta q
$ with $
    \mathcal{F}_\beta(p) \coloneqq - \log( p( e^{-\beta} - 1) + 1)
$ and $\beta > 0$.
To simplify the notation, we denote $\alpha = \frac{1}{N}(\operatorname{KL} + \log \frac{1}{\delta}) \in (0, \infty)$.

\vspace{-0.5em}\begin{proof}[Proof of \cref{thm:nonsep-delta}]
    The proof proceeds in three steps.
    In the first two steps, we lower bound $\frac1N\log\mathcal{I}_\Delta(N)$ and upper bound $\Delta$.
    In the third step, we use these bounds to lower bound $\smash{B[\Delta(q,\vardot), \alpha + \tfrac1N \log \mathcal{I}_\Delta(N)]}$ and identify the result with a particular Catoni bound.
    
    \textbf{Lower bound on $\frac1N\log\mathcal{I}_\Delta(N)$:}
    Since $\Delta\in \mathcal{C}$, 
    it is equal to its own double convex conjugate:
    $
        \Delta(q, p) = \Delta^{**}(q, p) = \sup_{c_q, c_p \in \mathbb{R}}\,(c_q q + c_p p - \Delta^*(c_q, c_p)),
    $ where ${}^*$ denotes convex conjugation.
    Let $X \sim \operatorname{Bin}(r,N)$.
    Then
    \begin{align} 
        \mathcal{I}_\Delta(N) &= \textstyle \sup_{r\in [0,1]}\E[e^{N \Delta(X/N, r)}]  
=  \sup_{r\in [0,1]}\E[e^{\sup_{c_q,c_p\in\R}(c_q X + N c_p r - N \Delta^*(c_q, c_p))}]\\
        &\geq \textstyle  \sup_{r\in [0,1]} \sup_{c_q,c_p\in\R}e^{N c_p r - N \Delta^*(c_q, c_p)}\E[e^{c_q X}] 
    \end{align}
        where $\E[e^{c_q X}] = (r( e^{c_q} - 1) + 1)^N$ is the moment-generating function of $X$.
    Consequently, taking $\log$, dividing by $N$, and noting that $\tfrac1N \log \E[e^{c_q X}] = - \mathcal{F}_{-c_q}(r)$,
    \begin{equation} \label{eq:tightened_bound} \textstyle
        \tfrac1N
        \log \mathcal{I}_\Delta(N)
        \ge A
        \quad\text{where}\quad
        A \coloneqq \sup_{c_q, c_p \in \mathbb{R}}
        [
            - \Delta^*(c_q, c_p)
            + \sup_{r \in[0,1]}(
            c_p r
            - \mathcal{F}_{-c_q}(r))
        ].
    \end{equation}
    
    \textbf{Upper bound on $\Delta$:}
    We upper bound $\Delta$ by making $\Delta^*$ as small as possible without exceeding the supremum from \eqref{eq:tightened_bound}.
    Note that $A$ is finite, because $\Delta^*$ is proper.
    Define $\smash{\tilde \Delta^*}$ as follows:
    $
        \smash{\tilde\Delta^*}(c_q, c_p)
        = -A + \sup_{r \in[0,1]}(
            c_p r
            - \mathcal{F}_{-c_q}(r)
        )
    $.
    Note that $\smash{\tilde\Delta^*}$ is proper, convex as a pointwise supremum of convex functions, and l.s.c.\ as a supremum of l.s.c.\ functions.
    In fact, $\smash{\tilde\Delta^*}$ is finite for all inputs.
    As the notation suggests, define $\smash{\tilde \Delta} \coloneqq \smash{(\tilde \Delta^*)^*}$.
    Then $\smash{\tilde\Delta^*}$ is indeed the convex conjugate of $\smash{\tilde \Delta}$, because $\smash{\tilde \Delta^*}\in \mathcal{C}$, so it is equal to its own double convex conjugate.
    Moreover,
    \begin{align}
        \smash{\tilde \Delta}(q, p)
        &=\textstyle
            A + \sup_{c_q, c_p \in \mathbb{R}}[
                c_q q + c_p p
                - \sup_{r \in[0,1]}(
                    c_p r
                    - \mathcal{F}_{-c_q}(r)
                )
            ] \\
        &=\textstyle 
            A + \sup_{c_q \in \mathbb{R}}\,[
                c_q q
                + \sup_{c_p \in \mathbb{R}}\,[
                c_p p
                - \mathcal{F}_{-c_q}^* (c_p)
                ]
            ] \\
        &=\textstyle 
            A + \sup_{c_q \in \mathbb{R}}\,[
                c_q q
                + \mathcal{F}_{-c_q}(p)
            ],
    \end{align}
    by observing that $p \mapsto \mathcal{F}_{-c_q}(p)\in \mathcal{C}$, so it is equal to its own double convex conjugate.
    Therefore,
    \begin{equation} \label{eq:tilde-Delta} \textstyle
        \smash{\tilde \Delta}(q, p)
        = A + \sup_{c_q \in \mathbb{R}}\,
            C_{-c_q}(q, p)
        \overset{\text{(i)}}{=} A + \kl(q, p)
    \end{equation}
    where (i) follows from a direct computation; see \cref{lem:sup-catoni} (\cref{app:additional-lemmas}).
    \textit{Claim:} For all $q, p \in [0, 1]$, $\tilde \Delta(q, p) \ge \Delta(q, p)$.
    This follows from the definitions and finiteness of $\smash{\tilde \Delta^*}$ and $A$:
    for all $c_q, c_p \in \mathbb{R}$,
    \begin{align} \textstyle
        \hspace*{-5pt}-\smash{\tilde\Delta^*}(c_q, c_p)
        \!+\! \sup_{r \in[0,1]}(
            c_p r
            - \mathcal{F}_{-c_q}(r)
        )
        = A
        \ge
        - \Delta^*(c_q, c_p)
        \!+\! \sup_{r \in[0,1]}(
            c_p r
            - \mathcal{F}_{-c_q}(r)
        ),
    \end{align}
    which means that $\tilde \Delta^* \le \Delta^*$, so $\tilde \Delta \ge \Delta$ by the order-reversing property of the convex conjugate.
    
    \textbf{Conclusion:}
    Assume that $\overline{p}_\Delta \!\!<\! 1$; otherwise, any $\beta \!>\! 0$ works.
    To begin with, use the previous steps:
    \begin{equation} \textstyle \label{eq:illegal-maurer-proof}
        \overline{p}_\Delta
        = B[\Delta(q,\vardot), \alpha + \tfrac1N \log \mathcal{I}_\Delta(N)]
        \smash{\overset{\text{\eqref{eq:tightened_bound}, claim}}{\ge}} B[\tilde\Delta(q,\vardot), \alpha + A]
        \overset{\text{\eqref{eq:tilde-Delta}}}{=} B[\kl(q, \vardot), \alpha]
        = \underline{p}.
    \end{equation}
    Since $\alpha > 0$, clearly $\underline{p} > q$, so $0 \le q < \underline{p} < 1$.
    Hence, if $q > 0$, then there exists a $\beta > 0$ such that $\kl(q, \underline{p}) = C_{\beta}(q, \underline{p})$ (\cref{lem:sup-catoni-pos}; \cref{app:additional-lemmas}).
    Using that $p \mapsto C_{\beta}(q,p)$ is continuous and strictly increasing for all $\beta > 0$, we have that $\underline{p} = B[C_{\beta}(q, \vardot), \alpha]$, so
    \begin{equation} \label{eq:conclusion}
        \overline{p}_\Delta
        \ge B[C_{\beta}(q, \vardot), \alpha] 
        \overset{\text{(i)}}{=} B[C_{\beta}(q, \vardot), \alpha + \tfrac1N \log\mathcal{I}_{C_{\beta}}(N)] 
        = \overline{p}_{C_\beta},
    \end{equation}
    where (i) uses that $\tfrac1N \log\mathcal{I}_{C_{\beta}}(N) \!=\! 0$ (\cref{lem:catoni-zero-I}; \cref{app:additional-lemmas}).
    If $q \!=\! 0$, then $\kl(0, \underline{p}) \!=\! \lim_{\beta \to \infty} C_{\beta}(0, \underline{p})$ (\cref{lem:sup-catoni-pos}; \cref{app:additional-lemmas}), so $\overline{p}_\Delta \!\ge\! B[\smash{\displaystyle\lim_{\beta \to \infty}}C_{\beta}(0, \cdot), \alpha]$, and conclude like in \eqref{eq:conclusion} using \cref{lem:B_limit} (\cref{app:additional-lemmas}).
\end{proof}

\vspace{-1em}
\section{Meta-Learning the Tightest Bounds for Synthetic Classification} \label{sec:classification}
\vspace{-0.5em}

We now consider, for a particular distribution over tasks, how tight each bound can be made in expectation. 
Two questions naturally arise: \emph{Which PAC-Bayes bounds are tightest?} and \emph{Can PAC-Bayes bounds be tighter than test set bounds?}
While test set bounds have traditionally been considered tighter than PAC-Bayes bounds, here we work in the small data regime where a substantial proportion of the data must be removed to form a test set, which could impact generalisation performance and hence lead to worse bounds.
Our goal is \emph{not} to compare these bounds when using standard practice, but to see how tight they can be \emph{in principle} if we use every tool in our toolbox to minimise the expected bounds.\footnote{Our goal here is to minimise \emph{high probability} PAC-Bayes and test set bounds \emph{in expectation}. See \citet[Appendix J]{dziugaite2021role} for a relevant discussion.} While these optimisations will be impractical for large models and datasets, they can provide some statistical insight.

\textbf{Learning Algorithm.\,} Certain learning algorithms may work better with test set bounds, and others with PAC-Bayes bounds. 
Instead of choosing a fixed algorithm, we \emph{meta-learn} \citep{schmidhuber:1987:srl,thrun2012learning} separate algorithms to optimise each bound in expectation: we parametrise a hypothesis space $\hypotheses_{\theta}$ and a \emph{posterior map} $Q_{\theta}\colon \datapoints^N \to \mathcal{M}_1(\hypotheses_{\theta})$ by a finite dimensional vector $\theta$, which is trained to optimise the expected bound (we will amalgamate all meta-learnable parameters into the single vector $\theta$).
This is explained in more detail below.
This way, we obtain algorithms that are optimised for each bound. 
After meta-learning, we can further refine each PAC-Bayes posterior by minimising the PAC-Bayes bound, see \cref{app:post_optimisation}. 

\textbf{Task Distribution.\,} In meta-learning, we refer to a data-generating distribution $D$ and dataset $S \sim D^N$ as a \emph{task}. We consider a \emph{distribution} over tasks, $D \sim \mathcal{T}$, where $\mathcal{T}$ is a distribution over data-generating distributions, and aim to find the best expected bounds for this distribution achievable by an optimised algorithm.\footnote{We could also  consider drawing all datasets from a \emph{single} task $D$, which would more directly match \cref{sec:characterising-generic-bounds}. We regard this case as less interesting, since we would often want a bound to perform well on a variety of tasks.} We choose especially simple learning tasks --- synthetic 1-dimensional binary classification problems, generated by thresholding Gaussian process (GP) samples --- which allows us to fully control the task distribution and easily inspect predictive distributions visually to diagnose learning. \Cref{app:data} contains full details.



\textbf{Priors.\,}
The choice of prior is crucial in PAC-Bayes, and the role of \emph{data-dependent priors} (DDPs) \citep{ambroladze2007tighter,parrado2012pac,perez2020tighter} has been gaining increased attention. This involves splitting the dataset into $N = N_{\mathrm{prior}} + N_{\mathrm{risk}}$ datapoints. The DDP is allowed to depend on the \emph{prior set} of size $N_{\mathrm{prior}}$ (standard priors use $N_{\mathrm{prior}}=0$), and the risk bound is computed on the \emph{risk set} of size $N_{\mathrm{risk}}$. Crucially, \emph{the bound is valid when the posterior depends on all $N$ datapoints}. Recently, \citet{dziugaite2021role} showed that DDPs can lead to tighter expected bounds than the optimal non-data-dependent prior, and are sometimes even \emph{required} to obtain non-vacuous bounds.
\citet{perez2020tighter} also report much tighter bounds when using DDPs.
In our experiments we meta-learn a DDP as a map from the prior set to the prior, $P_{\theta}\colon \datapoints^{\priorsetsize} \to \mathcal{M}_1(\hypotheses)$. To compare PAC-Bayes DDPs against test set bounds, we sweep the prior/train set proportion from $0$ to $0.8$ and see what the tightest value obtained is. Strictly this would require a union bound over the proportions, but here we are primarily interested in comparing the various bounds against each other on an even footing and vary the proportion for illustrative purposes.

\shortsubsection{The Meta-Learning Objective.} 
%
%
We now discuss meta-learning in more detail. During meta-training, $\theta$ is trained to optimise the expected PAC-Bayes generalisation bound over the task distribution:
\begin{equation} \label{eqn:meta-objective-full-expectation} \textstyle 
\mathbb{E}_{D \sim \mathcal{T}} \mathbb{E}_{S \sim D^N} \, B\big[\Delta_{\theta}(\gibbsR_{\riskset}(Q_{\theta}(\set)), \vardot), \tfrac1{\risksetsize}\big(\!\operatorname{KL}(Q_{\theta}(\set)\|P_{\theta} (\priorset))+\log\tfrac{\mathcal{I}_{\Delta_{\theta}}(\risksetsize)}{\delta}\big)\big],
\end{equation}
where the $\theta$ in $\Delta_{\theta}$ denotes that some bounds (Catoni and learned convex function) have meta-learnable parameters.
Alternatively, for a meta-learner that minimises a test set bound, the objective is simply $\mathbb{E}_{D \sim \mathcal{T}} \mathbb{E}_{S \sim D^N} \, \smash{\gibbsR_{S_{\mathrm{test}}}(Q_{\theta}(S_{\mathrm{train}}))}$, since all test set bounds are monotonic in the test set risk. 
We use the $0/1$ loss. As the classifiers are stochastic, the empirical risk is still differentiable with respect to $\theta$. In contrast to PAC-Bayes, the predictor that minimises the test set bound can be made deterministic after $\theta$ is learned, since it tends to eventually learn essentially deterministic classifiers; see \cref{app:deterministic_val_model}.
We sample $T = 80\,000$ tasks $D_t \sim \mathcal{T}$, with associated datasets $S_t \sim D_t^N$.
These form the \emph{meta-trainset}.
Additionally, we sample $1024$ tasks that form a \emph{meta-testset} used to estimate the average bounds over $\mathcal{T}$ after meta-training.
For the PAC-Bayes bounds, we then Monte Carlo estimate \eqref{eqn:meta-objective-full-expectation}.
Hence, the final objective for a PAC-Bayes meta-learner is (a minibatched version of):
\begin{equation} \textstyle \label{eqn:pac-bayes-batch-obj}
\frac{1}{T} \sum_{t=1}^T \, B\big[\Delta_{\theta}(\gibbsR_{S_{t, \mathrm{risk}}}(Q_{\theta}(\set_t)), \vardot), \tfrac1{\risksetsize}\big(\!\operatorname{KL}(Q_{\theta}(\set_t)\|P_{\theta} (S_{t, \mathrm{prior}}))+\log\tfrac{\mathcal{I}_{\Delta_{\theta}}(\risksetsize)}{\delta}\big)\big].
\end{equation}
Similarly, the objective for the test set bound meta-learner is $\frac{1}{T} \sum_{t=1}^T \smash{\gibbsR_{S_{t, \mathrm{test}}}(Q_{\theta}(S_{t, \mathrm{train}}))}$.
The bounds we compute on datasets in the meta-testset, after meta-training is complete and $\theta$ is frozen, are valid even though $\theta$ was optimised on the meta-trainset. This highlights a contrast between our procedure and the PAC-Bayes meta-learning in \citet{amit2018meta,rothfuss2020pacoh,liu2021pac} and \citet{farid2021pac}. While those works use PAC-Bayes to analyse generalisation of a meta-learner on new tasks, we use PAC-Bayes to analyse generalisation \emph{within} individual tasks.



\shortsubsection{Parametrising the Meta-Learner and Hypothesis Space.} 
We now describe how to parametrise the hypothesis space $\hypotheses_{\theta}$ and the maps $Q_{\theta}, P_{\theta}$. 
We meta-learn a feature map $\phi_{\theta}\colon \R \to \R^K$ and choose\footnote{The dimensionality $K$ is fixed a priori.}
$
\hypotheses_\theta = \{h_w: h_w(x)=\mathrm{sign} \langle w, \phi_{\theta}(x)\rangle,\, w\in \R^K \}
$. 
For $Q_\theta$ and $P_\theta$ Gaussian, this hypothesis space allows us to compute the empirical Gibbs risk without Monte Carlo integration; see \cref{app:computing_emp_risk} for details.
For the form of $Q_{\theta}$, we take inspiration from Neural Processes (NPs) \citep{garnelo2018conditional,garnelo2018neural,kim2019attentive}. NPs use neural networks to flexibly parametrise a map from datasets to predictive distributions that respects the permutation invariance of datasets \citep{zaheer2017deep}. They are regularly benchmarked on 1D meta-learning tasks, making them ideally suited. We make a straightforward modification to NPs to make them output Gaussian measures over weight vectors $w \in \mathbb{R}^K$. Hence, they act as parametrisable maps from $\mathcal{Z}^{\setsize}$ to the set of Gaussian measures on $\mathbb{R}^K$.

We considered two kinds of NP, one based on multilayer perceptrons (MLP-NP) and another based on convolutional neural networks (CNN-NP) (detailed in \cref{app:MLP-NP,app:GNP}) Although the MLP-NP is very flexible, the state-of-the-art in NPs on 1D tasks is given by CNN-based NPs \citep{gordon2019convolutional,foong2020meta,bruinsma2021}. 
We use an architecture closely based on the \emph{Gaussian Neural Process} \citep{bruinsma2021}, which outputs full-covariance Gaussians. 
As expected, we found the CNN-NP to produce tighter (or comparable) average bounds to the MLP-NP, while using far fewer parameters, and training much more reliably and quickly.
This is because the CNN-NP is translation equivariant, and hence exploits a key symmetry of the problem.
Hence, we focus on the CNN-NP, but report some results for the MLP-NP in \cref{app:additional_plots_mlp}. Hyperparameter details are given in \cref{app:hyperparameters}.

\shortsubsection{Results.}
\begin{figure}
\centering
\begin{subfigure}{.49\textwidth}
  \centering
  \includegraphics[width=.49\linewidth]{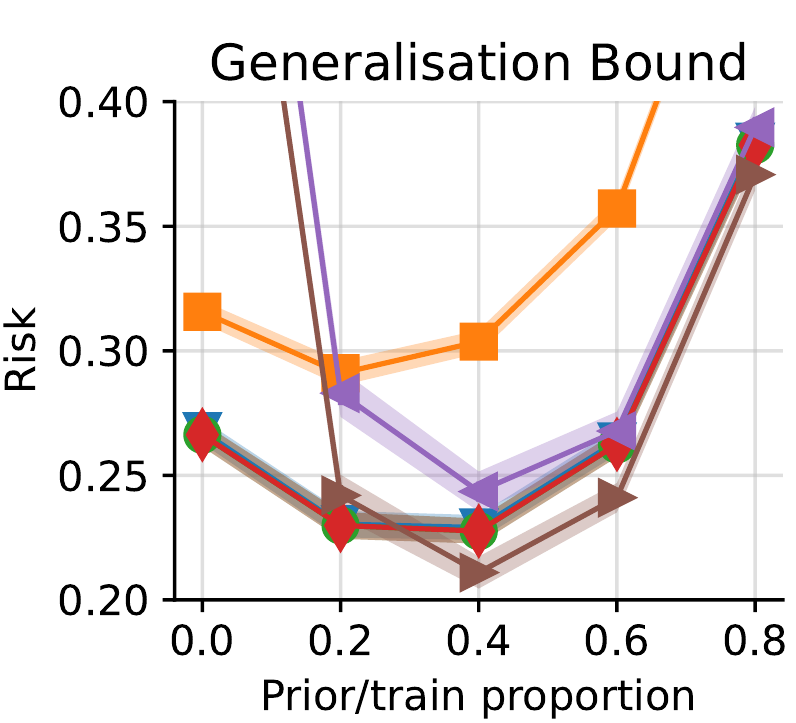}
  \includegraphics[width=.49\linewidth]{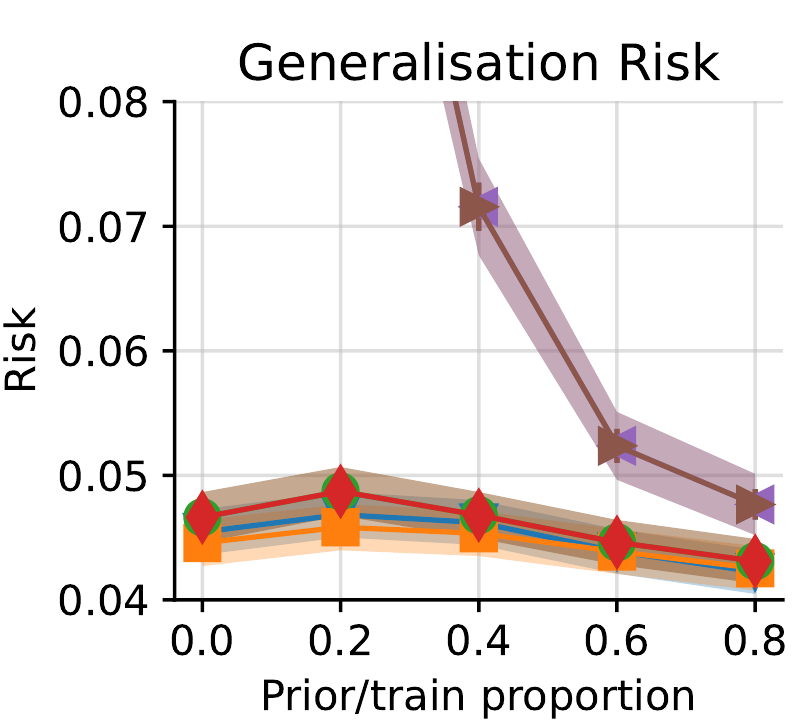}
  \caption{$N=30$ datapoints.}
\end{subfigure}%
\begin{subfigure}{.49\textwidth}
  \centering
  \includegraphics[width=.49\linewidth]{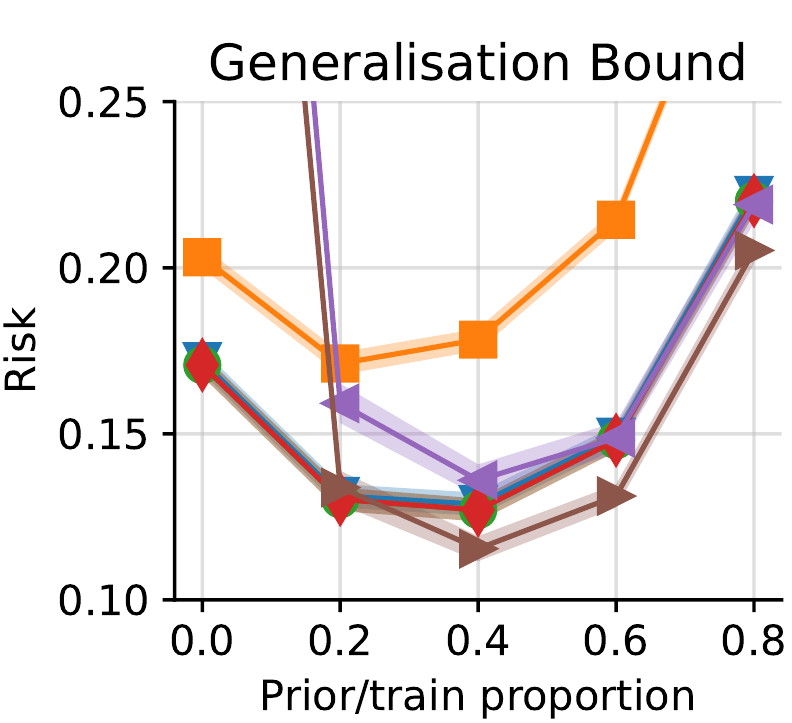}
  \includegraphics[width=.49\linewidth]{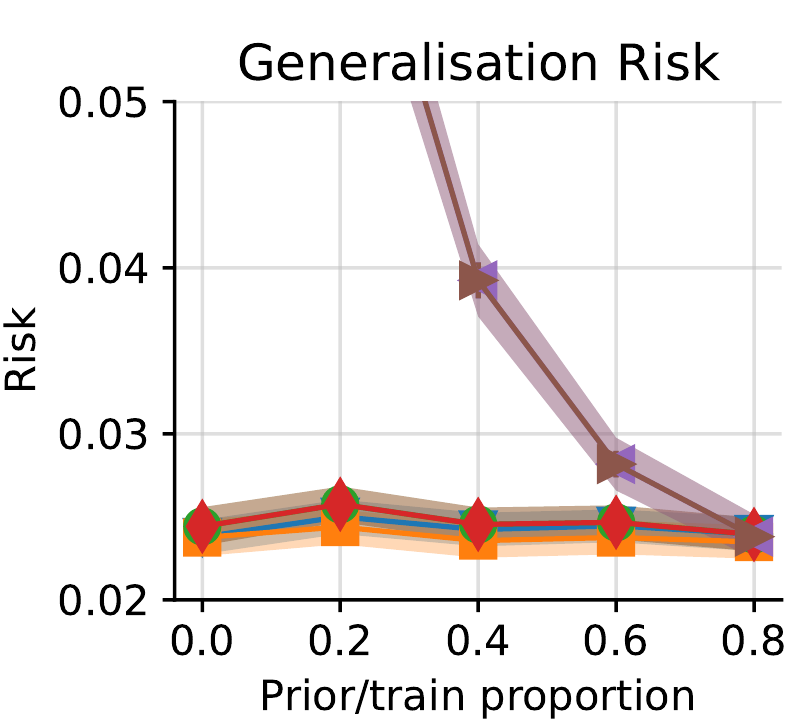}
  \caption{$N=60$ datapoints.}
\end{subfigure} 
\caption{\small
\textbf{Average generalisation bound and actual generalisation risk}
($\pm$ two standard errors) for CNN-NP meta-learners trained to optimise \textcolor{plotblue}{Catoni} ($\textcolor{plotblue}{\blacktriangledown}$), \textcolor{plotorange}{PAC-Bayes-kl} ($\textcolor{plotorange}{\blacksquare}$), \textcolor{plotgreen}{conjectured PAC-Bayes-kl} ($\textcolor{plotgreen}{\sbullet[1.5]}$), \textcolor{plotred}{learned convex} ($\textcolor{plotred}{\blacklozenge}$), \textcolor{plotpurple}{Chernoff test set} ($\textcolor{plotpurple}{\blacktriangleleft}$), 
and \textcolor{plotbrown}{binomial tail test set} ($\textcolor{plotbrown}{\blacktriangleright}$) bounds.
Catoni, conjectured PAC-Bayes-kl, and learned convex overlap. The generalisation risks for Chernoff and binomial tail test set bounds are identical as they share the same meta-learner; only the bound computation differs. The bounds are valid with failure probability $\delta = 0.1$ except for conjectured PAC-Bayes-kl, which should be a lower bound on the best bound achievable with \cref{thm:begin-bound}. Corresponding plots for the MLP-NP are in \cref{app:additional_plots_mlp}.
} \label{fig:gen_bounds_1d_GNP_post_opt}
\vspace{-0.5em}
\end{figure}
\begin{figure}\small
\centering
\begin{subfigure}{.49\textwidth}
  \centering
  \includegraphics[width=.99\linewidth]{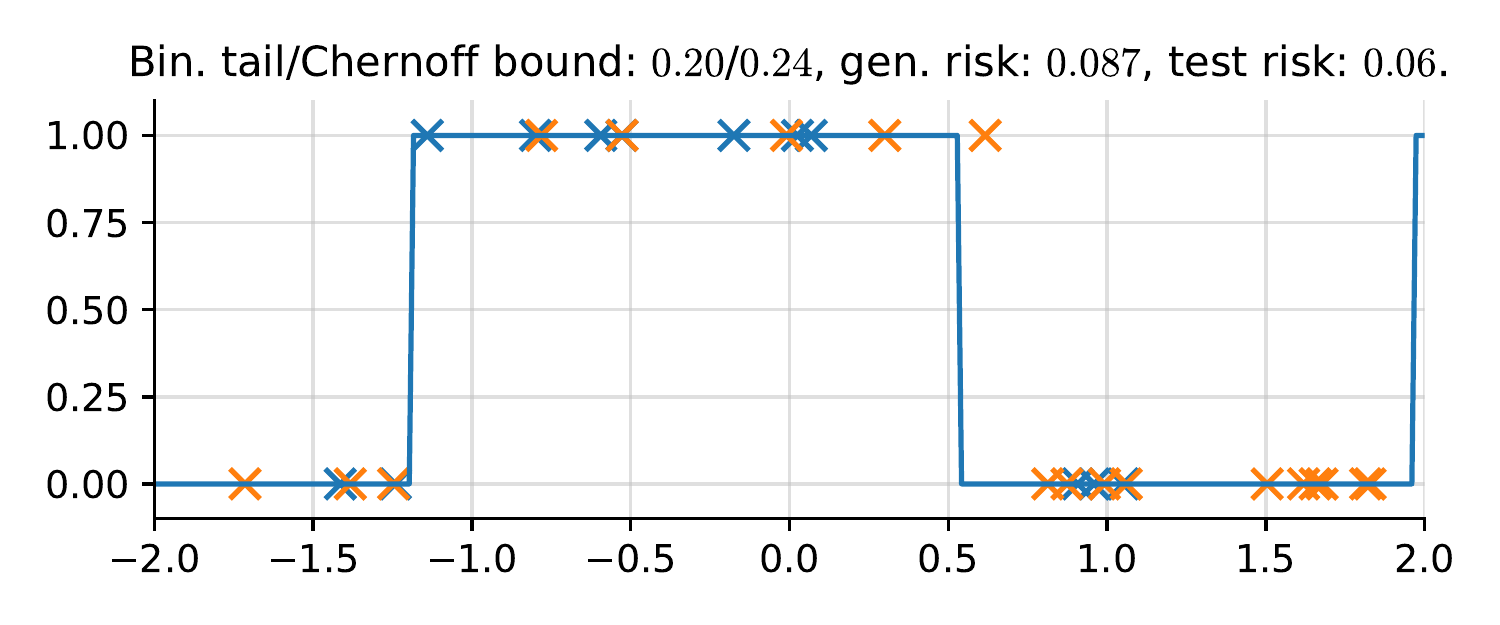}
  \caption{\textbf{Binomial tail/Chernoff test set} bounds, showing the learned hypothesis (\textcolor{plotblue}{---}), the train set (\textcolor{plotblue}{{\tiny \XSolid}}) of size 12 and the test set (\textcolor{plotorange}{{\tiny \XSolid}}) of size 18.}
\end{subfigure}\hfill%
\begin{subfigure}{.49\textwidth}
  \centering
  \includegraphics[width=.99\linewidth]{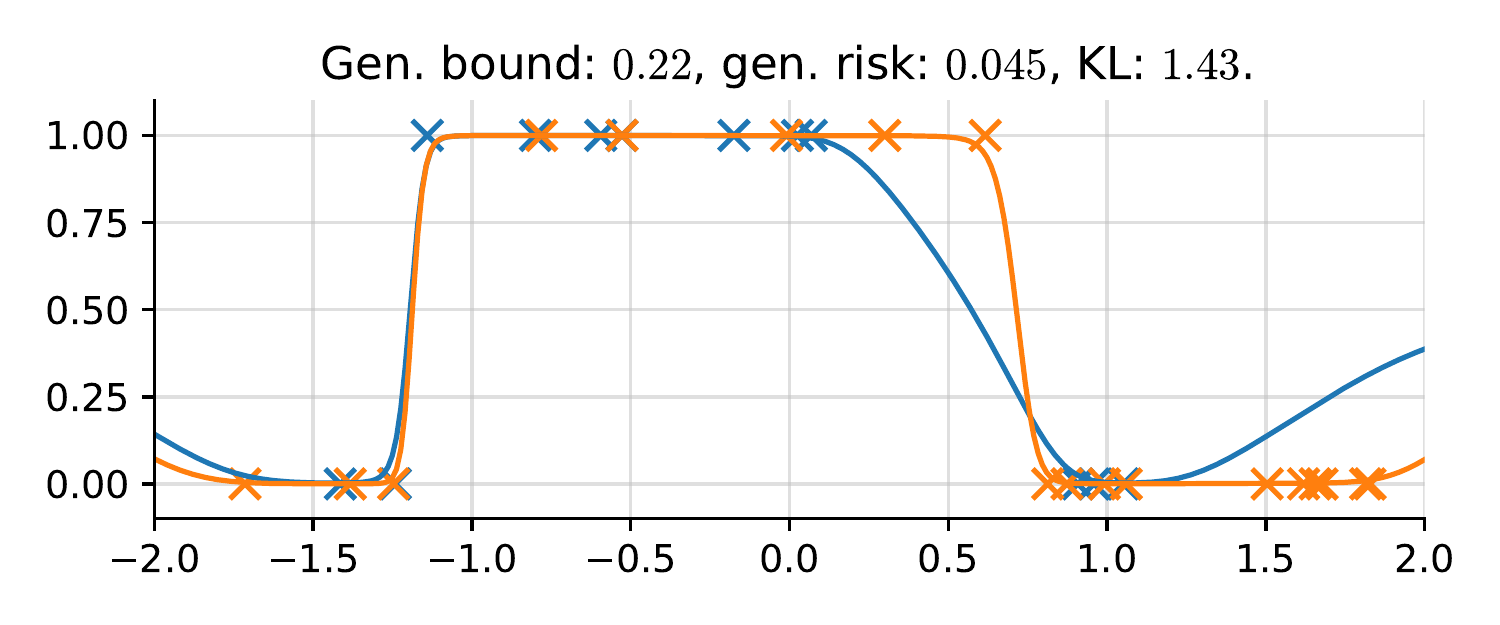}
  \caption{\textbf{Learned convex} bound with data-dependent prior, showing the prior (\textcolor{plotblue}{---}) and posterior (\textcolor{plotorange}{---}) predictive, prior set (\textcolor{plotblue}{{\tiny \XSolid}}) of size 12 and risk set (\textcolor{plotorange}{{\tiny \XSolid}}) of size 18.}
\end{subfigure}
\caption{\small\textbf{Predictions on one of the 1D datasets in the meta-test set} with $N=30$ and prior/train proportion $0.4$. For each method, we report the generalisation bound and actual generalisation risk. For the test set model, we also show the risk on the test set, and for the PAC-Bayes model we show the KL-divergence. The learned convex bound meta-learner has learned a DDP that provides a ``first guess'' given the prior set, which is then refined by the posterior. Figures for other PAC-Bayes bounds and datasets are provided in \cref{app:additional_plots_predictive}.}
\label{fig:example_classification_kl-val_catoni}
\vspace{-1em}
\end{figure}
We show example classification tasks and average bounds on the meta-test set in \cref{fig:example_classification_kl-val_catoni,fig:gen_bounds_1d_GNP_post_opt}. Note that the test set classifier became deterministic and makes hard predictions whereas the PAC-Bayes classifier shows uncertainty; see \cref{app:deterministic_val_model} for a discussion. The PAC-Bayes-kl bound is loosest, which is unsurprising as it has no optimisable parameters to adapt to $\mathcal{T}$.\footnote{This is in contrast with usual applications of PAC-Bayes, where one does not have a meta-dataset with which to optimise parameters of the bound. In that setting, it can be an advantage to not have tunable parameters.}  
Surprisingly, the results for Catoni, conjectured PAC-Bayes-kl, and learned convex are nearly identical. 
As long as optimisation has succeeded reasonably, this suggests empirically that, in light of \cref{cor:optimistic-MLS}, one of the Catoni bounds may be very nearly optimal among all convex functions for this task distribution --- there is not much ``slack'' from choosing suboptimal $\Delta$ here.
We also see that the Catoni and learned convex bounds with prior proportion $0.4$ are tighter than any Chernoff test set bound. Hence, \emph{PAC-Bayes can provide slightly tighter (or comparable) generalisation bounds to a Chernoff test set bound}. However, we see that \emph{the binomial tail test set bound with $40\%$ of the data used for the selecting the predictor and the remaining $60\%$ used for evaluating the bound leads to the tightest generalisation bounds overall}. \Cref{cor:optimistic-MLS} sheds light on this behaviour: the optimal generic PAC-Bayes bound reduces, at best, to the Chernoff test set bound when the posterior equals the prior. However, the Chernoff bound is itself looser than the binomial tail bound.
Of course, the posterior does not equal the prior here, but \cref{cor:optimistic-MLS} indicates there is an extra source of looseness that PAC-Bayes has to overcome relative to the binomial tail bound.
Finally, although the test set meta-learner leads to the tightest generalisation bounds, its generalisation risk is roughly double that of the PAC-Bayes meta-learner when the prior/train set proportion is $0.4$. 

\vspace*{-0.5em}
\section{Conclusions, Open Problems, and Limitations}
\label{sec:conclusion}
\vspace*{-0.5em}

PAC-Bayes presents a potentially attractive framework for obtaining tight generalisation bounds in the small-data regime. We have investigated the tightness of PAC-Bayes and test set bounds in this regime both theoretically and experimentally. Theoretically, we showed that the generic PAC-Bayes theorem of \citet{germain2009pac} and \citet{begin2016pac} which encompasses a wide range of PAC-Bayes bounds, cannot produce tighter bounds in expectation than the expression obtained by discarding the $\log (2\smash{\sqrt{N}})/N$ term in the \citet{langford2001bounds} bound (\textit{i.e.}, the \emph{conjectured PAC-Bayes-kl bound}; \cref{cor:optimistic-MLS}).
Although we did not prove that the conjectured PAC-Bayes-kl bound is a valid generalisation bound, numerical evidence suggests (\cref{fig:stoch1,fig:stoch2}) that there may exist a convex function $\Delta$ which achieves it, at least for the distributions over empirical risk and KL-divergence we considered. This suggests the following open problem:
\vspace{\beforethmspace}\begin{openproblem} \label{open:lower-bound}
    For an arbitrary distribution over datasets, does there exist a choice of $\Delta$ such that the expected conjectured PAC-Bayes-kl bound is attained (\cref{cor:optimistic-MLS})?
    If not, how close can one get to the expected conjectured PAC-Bayes-kl bound?
\end{openproblem}
If such a $\Delta$ exists, then that would imply the conjectured PAC-Bayes-kl bound is a valid generalisation bound (see \cref{sec:characterising-generic-bounds}) and resolve Problem 6.1.2 of \citet{langford2002quantitatively} in the affirmative.

We then considered, in a controlled experimental setting where meta-learning all parameters of the bounds and learning algorithms was feasible, whether PAC-Bayes bounds could be tighter than test set bounds. Although we found PAC-Bayes competitive with Chernoff bounds, both were outperformed by the binomial tail test set bound. This motivates a second open problem:
\vspace{\beforethmspace}\begin{openproblem} \label{open:relax-to-binom}
    Can a PAC-Bayes bound be found that relaxes gracefully to the binomial tail test set bound (\cref{lem:bin-inv-bound}) when the posterior is equal to the prior?
\end{openproblem}
Resolving these problems could have a significant impact on the tightness of PAC-Bayes applied to small-data, and clarify our understanding of the relationship between PAC-Bayes and test set bounds.

\shortsubsection{Limitations.}
In this paper, we  concern ourselves with understanding the tightness of bounds in what might be called the \emph{standard PAC-Bayes setting} of supervised learning: bounded losses, i.i.d.~data, and Gibbs risk. We also focus on bounds that are first order in the sense that they rely only on the empirical Gibbs risk, though extending the analysis to consider other PAC-Bayes theorems (\textit{e.g.}~\citet{tolstikhin2013pac,rivasplata2020pac}) would be of interest, especially with regards to \cref{open:lower-bound,open:relax-to-binom}. For many practical applications in which performance guarantees are needed (\textit{e.g.}~health care), the i.i.d.~assumption should be considered carefully, as it is likely an unrealistic simplification. Furthermore, Gibbs classifiers are less commonly used than deterministic classifiers in practice. To address these and other concerns, PAC-Bayes has been generalised in many directions beyond the scope of the standard setting we consider. Examples include bounds for non-i.i.d.~data \citep{6257492,alquier2018simpler,rivasplata2020pac}, unbounded losses \citep{germain2016pac}, derandomised classifiers \citep{blanchard2007occam,viallard2021general}, and Bayes risk \citep{germain2015majorityvote,masegosa2020secondorder}. Bounds based on other divergences besides the KL have also been proposed \citep{begin2016pac,alquier2018simpler}. As our proof relies primarily on tools from convex analysis, and Jensen's inequality is ubiquitous in PAC-Bayes bounds, it would be interesting to see if our arguments can be extended beyond the limited setting we focus on. 

Finally, our meta-learning experiments only considered 1D classification, and the results might not necessarily be representative of more realistic datasets. We also only consider Gaussian prior and posterior distributions in our experiments for the sake of tractability. Scaling up the experiments and considering more flexible distributions is an important, but potentially challenging, avenue for future work.
\clearpage

\section*{Acknowledgements and Funding Transparency Statement}
We would like to thank Pierre Alquier and Yann Dubois for insightful discussions, John Langford for clarifying a remark on test set bounds, David Janz, Will Tebbutt, and Austin Tripp for providing helpful comments on a draft version of this paper, and Omar Rivasplata for helpful comments on an earlier version of this manuscript. Andrew Y.~K.~Foong gratefully acknowledges funding from a Trinity Hall Research Studentship and the George and Lilian Schiff Foundation.
Wessel P.\ Bruinsma was supported by the Engineering and Physical Research Council (studentship number 10436152).
David R.~Burt acknowledges funding from the Qualcomm Innovation Fellowship and the Williams College Herchel Smith Fellowship. Richard E.~Turner is supported by Google, Amazon, ARM, Improbable, Microsoft, EPSRC grant EP/T005637/1, and the UKRI Centre for Doctoral Training in the Application of Artificial Intelligence to the study of Environmental Risks (AI4ER).
\printbibliography


\clearpage

\appendix

\section{Relationship Between PAC-Bayes and Occam Bound} \label{app:occam}

The well-known \emph{Occam} bounds can be derived by a simple application of the union bound to a countable hypothesis class $\mathcal{H}$. In particular, we can consider a ``prior'' distribution $P$, which functions similarly to the PAC-Bayes prior. Then by applying the union bound and weighting each hypothesis $h$ with a failure probability of $P( \{ h \}) \delta$, we can convert any test set bound into a corresponding train set bound. Applying this to \cref{lem:bin-inv-bound}, we obtain:
\begin{theorem}[Binomial tail Occam bound, \citet{langford2002quantitatively}, Theorem 4.6.1] \label{thm:binomial-occam}
Let $\mathcal{H}$ be countable, and fix $P \in \mathcal{M}_1(\mathcal{H})$, $\ell \in \{ 0 , 1\}$ and $\delta \in (0,1)$. Then
\begin{align}
    \mathrm{Pr} \Big( (\forall h) \,\, R_D(h) \leq \overline{e} \big(\setsize, \setsize R_{\set} (h), P( \{ h \}) \delta \big) \Big) \geq 1 - \delta,
\end{align}
where $\overline{e}$ is defined in \cref{lem:bin-inv-bound}.
\end{theorem}

Alternatively, applying this procedure to \cref{lem:kl-chernoff-bound} yields a looser bound:
\begin{theorem}[Chernoff Occam bound, \citet{langford2002quantitatively}, Corollary 4.6.2] \label{thm:chernoff-occam}
Let $\mathcal{H}$ be countable, and fix $P \in \mathcal{M}_1(\mathcal{H})$, $\ell \in [0 , 1]$ and $\delta \in (0,1)$. Then
\begin{align}
    \mathrm{Pr} \left( (\forall h) \,\, \mathrm{kl}(R_S(h), R_D(h)) \leq \frac{1}{N} \left[ \log \frac{1}{P(\{ h \})} + \log \frac{1}{\delta} \right] \right) \geq 1 - \delta.
\end{align}
\end{theorem}

Following \citet[Sec 5.1]{langford2005tutorial}, it is instructive to compare the Chernoff Occam bound with the PAC-Bayes-kl bound (\cref{cor:maurer}) when $\mathcal{H}$ is countable and $Q$ is constrained to be a point mass, \textit{i.e.}~$Q = Q_h \coloneqq \delta_h$, where $\delta_h$ denotes the Dirac measure at $h$. In that case, $\mathrm{KL}(Q_h \| P)$ reduces to $\log(1/P(\{ h \}))$, and the Gibbs risks $\smash{\gibbsR}_S(Q_h), \smash{\gibbsR}_D(Q_h)$ simply reduce to the risks $R_S(h), R_D(h)$. Then the PAC-Bayes-kl bound states that:
\begin{align} \label{eqn:MLS-vs-occam}
    \mathrm{Pr} \left( (\forall h) \,\, \mathrm{kl}(R_S(h), R_D(h)) \leq \frac{1}{N} \left[ \log \frac{1}{P(\{ h \})} + \log \frac{2\sqrt{N}}{\delta} \right] \right) \geq 1 - \delta.
\end{align}
Comparing \cref{eqn:MLS-vs-occam} with \cref{thm:chernoff-occam}, we see that the PAC-Bayes-kl bound leads to a bound on $\mathrm{kl}(R_S(h), R_D(h))$ which is looser by an additive constant of $\smash{\log (2 \sqrt{N}) / N}$ compared to the Chernoff Occam bound. Hence the PAC-Bayes bound does not relax gracefully to the Occam bound in this case, which motivates Open Problem 6.1.2 in \citet{langford2002quantitatively}. In fact, by \cref{rem:nonsep-delta,cor:optimistic-MLS}, we know that \emph{if} we could find a convex $\Delta$ that allowed us to remove this $\smash{\log (2 \sqrt{N}) / N}$ term (\textit{i.e.}, the conjectured PAC-Bayes-kl bound), this would be the tightest possible bound obtainable from the generic PAC-Bayes theorem (\cref{thm:begin-bound}). This motivates \cref{open:lower-bound}. Finally, noting that \cref{thm:chernoff-occam} is itself a looser version of \cref{thm:binomial-occam}, we see that a PAC-Bayes bound that relaxes gracefully to \cref{thm:binomial-occam} is not obtainable from \cref{thm:begin-bound}, motivating \cref{open:relax-to-binom}.

\section{Proof of Generic PAC-Bayes Theorem (Theorem \ref{thm:begin-bound})} \label{app:proof-begin-bound}

We provide a proof of \cref{thm:begin-bound} here for convenience, which closely follows the proof given in \citet{begin2016pac}. We first require a well-known lemma:
\begin{lemma}[Kullback-Leibler change of measure inequality, {\citealp[, Corollary 4.15]{boucheron2013concentration}}] \label{lem:change-of-measure}
    For any set $\mathcal{H}$, probability measures $P, Q \in \mathcal{M}_1(\mathcal{H})$, and measurable function $\phi: \mathcal{H} \to \mathbb{R}$,
    \begin{align}
        \mathbb{E}_{h \sim Q} \phi(h) \leq \mathrm{KL} (Q \| P) + \log \left( \mathbb{E}_{h \sim P} e^{\phi(h)} \right).
    \end{align}
\end{lemma}

In order to deal with general bounded losses $\ell \in [0,1]$, we also use a lemma proven in \citet{maurer2004note}:
\begin{lemma}[\citet{maurer2004note}, Lemma 3] \label{lem:maurer-convex}
    For any $[0, 1]$-valued random variable $z$, let $z'$ denote the unique Bernoulli random variable with $\mathbb{E}[z'] = \mathbb{E}[z]$. Let $S=(z_1, \hdots, z_N)$ and $S' = (z_1', \hdots, z_N')$ denote tuples of $N$ such random variables. Then for any convex function $f: [0, 1]^N \to \mathbb{R}$,
\begin{align}
    \mathbb{E}[f(S)] \leq \mathbb{E}[f(S')].
\end{align}
\end{lemma}

We can now prove \cref{thm:begin-bound}:
\begin{proof}[Proof of \cref{thm:begin-bound}.]
    Applying Jensen's inequality followed by \cref{lem:change-of-measure}, we have, for all $Q \in \mathcal{M}_1(\mathcal{H})$:
    \begin{align}
        N \Delta(\gibbsR_S(Q), \gibbsR_D(Q)) &= N \Delta(\mathbb{E}_{h \sim Q} R_S(h), \mathbb{E}_{h \sim Q} R_D(h)) \\
        &\leq \mathbb{E}_{h \sim Q} N \Delta (R_S(h), R_D(h)) \\
        &\leq \mathrm{KL}(Q \| P) + \log \left( \mathbb{E}_{h \sim P} e^{N\Delta (R_S(h), R_D(h))} \right). \label{eqn:KL-log-MGF-bound}
    \end{align}
    Applying Markov's inequality to the random variable $\mathbb{E}_{h \sim P} e^{N\Delta (R_S(h), R_D(h))}$ (which is random through $S \sim D^N$), we obtain, for any $\delta \in (0,1)$:
    \begin{align}
        \mathrm{Pr} \left( \mathbb{E}_{h \sim P} e^{N\Delta (R_S(h), R_D(h))} \leq \frac{\mathbb{E}_{S \sim D^N} \mathbb{E}_{h \sim P} e^{N\Delta (R_S(h), R_D(h))}}{\delta}  \right) \geq 1 - \delta.
    \end{align}
    Combining this with \cref{eqn:KL-log-MGF-bound} yields, with probability at least $1 - \delta$, for all $Q$ simultaneously:
    \begin{align} \label{eqn:markov}
       \Delta(\gibbsR_S(Q), \gibbsR_D(Q)) \leq \frac{1}{N} \left[ \mathrm{KL}(Q \| P) + \log \frac{\mathbb{E}_{S \sim D^N} \mathbb{E}_{h \sim P} e^{N\Delta (R_S(h), R_D(h))}}{\delta}  \right].
    \end{align}
    Finally, we upper bound $\mathbb{E}_{S \sim D^N} \mathbb{E}_{h \sim P} e^{N\Delta (R_S(h), R_D(h))} $ by a quantity than can be computed without knowing the true distribution $D$. By Tonelli's theorem:
    \begin{align} \label{eqn:fubini}
        \mathbb{E}_{S \sim D^N}  \mathbb{E}_{h \sim P} e^{N\Delta (R_S(h), R_D(h))}  &= \mathbb{E}_{h \sim P} \mathbb{E}_{S \sim D^N}  e^{N\Delta (R_S(h), R_D(h))}.
    \end{align}
    We will now upper bound the inner expectation by a quantity that is independent of $h$. Denote the datapoints in $S$ as $S= ((x_1, y_1), \hdots, (x_N, y_N))$. Recall that $R_S(h) \coloneqq \frac{1}{N} \sum_{n=1}^N \ell((x_n,y_n), h)$, and note that $\ell((x_n,y_n), h)$ is a $[0, 1]$-valued random variable. Let $L$ denote the $N$-tuple of these random variables for each datapoint in $S$, \textit{i.e.}~$L=(\ell((x_1,y_1), h), \hdots, \ell((x_N,y_N), h))$, and let $M(L) \coloneqq \frac{1}{N} \sum_{n=1}^N L_n$ be the arithmetic mean of $L$. Then the function $f: L \mapsto e^{N\Delta (M(L), R_D(h))}$ is convex since $M(L)$ is linear in $L$, $\Delta$ is convex and the exponential function is convex and nondecreasing. Hence defining $L'$ as the $N$-tuple of Bernoulli random variables such that $\mathbb{E}[L'_n] = \mathbb{E}[L_n] = R_D(h)$ for $1 \leq n \leq N$ and applying \cref{lem:maurer-convex},
    \begin{align} \textstyle
         \mathbb{E}_{S \sim D^N}  e^{N\Delta (R_S(h), R_D(h))} &= \mathbb{E}\,  e^{N\Delta (M(L), R_D(h))} \\
        &= \textstyle \mathbb{E}\, f(L)\\
        &\leq \mathbb{E}\, f(L') \\
        &= \textstyle \sum_{k=0}^N \mathrm{Pr} \big( M(L') = k/N \big) e^{N \Delta(k / N, R_D(h))} \\
        &= \textstyle \sum_{k=0}^N \binom{N}{k} R_D(h)^k (1 - R_D(h))^{N-k} e^{N \Delta(k / N, R_D(h))} \\
        &\leq \textstyle \sup_{r \in [0, 1]} \sum_{k=0}^N \binom{N}{k} r^k (1 - r)^{N-k} e^{N \Delta(k / N, r)}\\
        &= \textstyle \mathcal{I}_{\Delta}(N).
    \end{align}
    Substituting this into \cref{eqn:fubini} and then \cref{eqn:markov} completes the proof.
\end{proof}

\section{Basic Facts from Convex Analysis}
\label{app:convex}
A function $f \colon \mathbb{R}^n \to \mathbb{R} \cup \{+\infty\}$ is called \emph{proper} if it not everywhere $\infty$ and nowhere $-\infty$.
The convex conjugate $f^* \colon \mathbb{R}^n \to \mathbb{R} \cup \{+\infty\}$ of $f$ is defined as
$
    f^*(c) = \sup_{x \in \mathbb{R}^n}\,(\langle c, x \rangle - f(x))
$.
If $f$ is proper, convex, and l.s.c., then
$f^*$ is also proper, convex, and l.s.c.
Moreover, if $f$ is proper, convex, and l.s.c., then $f$ is equal to its double convex conjugate:
$
    f(x) = \sup_{c \in \mathbb{R}^n}\,(\langle c, x \rangle - f^*(c))
$.
A pointwise supremum of convex functions is again convex;
and a pointwise supremum of l.s.c. functions remains l.s.c.
Convex functions $f \colon A \to \mathbb{R}$ defined on only a convex subset $A \subseteq \R^n$ are extended to the whole of $\R^n$ by setting $f|_{\R^n\setminus A}= +\infty$. See, for example, \citet{rockafellarvariational} for proofs of these results and an introduction to the topic.

\section{Monotonicity of \texorpdfstring{$\Delta$}{Delta}} \label{app:monotonicity}
\begin{proposition}
     For $\Delta \colon [0, 1]^2 \to \mathbb{R} \cup \{ +\infty \}$ a proper, convex, and lower semi-continuous function, $q \in [0, 1]$, $\delta \in (0,1]$, and $\operatorname{KL} \geq 0$,
    define
    \begin{equation}
        \overline{p}_{\Delta}
        = \sup\,\left\{
            p \in [0, 1]
            :
            \Delta(q, p) \le \frac{1}{N} \left(
                \operatorname{KL} + \log \frac{\mathcal{I}_\Delta(N)}{\delta}
            \right)
        \right\}
    \end{equation}
    Then there exists a proper, convex, lower semi-continuous $\Delta' \colon [0, 1]^2 \to \mathbb{R} \cup \{ +\infty \}$ such that $\overline{p}_{\Delta'} \leq \overline{p}_{\Delta}$, and for every $q\in [0,1]$, $\Delta'(q,\cdot)$ is monotonically increasing. 
\end{proposition}
\begin{proof}
Define $\Delta'(q,p) = \inf_{p' \geq p} \Delta(q,p')$. We will prove that $\Delta'$ has the desired properties. First, $\Delta'$ is not infinity everywhere, as $\Delta'(q,p) \leq \Delta(q,p)$ and $\Delta$ is proper. Second, since $\Delta$ is l.s.c.~and proper, it obtains a minimum on the compact set $[0,1]^2$, hence $\Delta'$ does not take the value $-\infty$. Therefore, $\Delta'$ is proper.

Since $\Delta$ is l.s.c., the strict sublevel sets of $\Delta$ are open; that is, for all $y\in \mathbb{R}$, $\{(q,p):\Delta(q,p) < y\}$ is open. Then,
\begin{align}
    \{(q,p):\Delta'(q,p) < y\} &= \{(q,p): \inf_{p'\geq p} \Delta(q,p') < y\}  \\
    &= \bigcup_{p' \in [p, 1]} \{(q,p'): \Delta(q,p') < y\}. 
\end{align}
The equality follows from noting that the infimum on the closed set $[p,1]$ must be achieved as $\Delta$ is l.s.c.\footnote{The equality holds if $\Delta$ is not l.s.c~by the definition of the infimum as well} As we have written $\{(q,p):\Delta'(q,p) < y\}$ as a union of open sets, it is open. Hence, the sublevel sets of $\Delta'$ are open, implying $\Delta'$ is l.s.c.

We next show that $\Delta'$ is convex. Define the function $D: \mathbb{R}^3 \to \mathbb{R} \cup \{+\infty\}$ by,
\[
D(q,p', p) = \begin{cases}
    \Delta(q,p) & p' \geq p \\
    +\infty & \text{otherwise}.
\end{cases}
\]
$D(q,p,p')$ is convex since $\Delta$ is convex and $p' \geq p$ is a convex set. Also, $\inf_{p \in \R} D(q,p', p) = \inf_{p' \geq p} \Delta(q,p') = \Delta'(q,p)$. As the infimum projection of a convex function is convex \citep[Proposition 2.22]{rockafellarvariational}, $\Delta'$ is convex.
Also, $\Delta'(q, p)$ is monotonically increasing in $p$ as the infimum is taken over a smaller set for larger $p$. 

It remains to show that $\overline{p}_{\Delta'} \leq \overline{p}_{\Delta}$. For all pairs $(q,p)$, $\Delta'(q,p) = \inf_{p' \geq p} \Delta(q,p') \leq \Delta(q,p)$. From this it follows that 
\begin{equation}\label{eqn:monotone-bound}
\frac{1}{N} \left(
                \operatorname{KL} + \log \frac{\mathcal{I}_{\Delta'}(N)}{\delta}\right) \leq \frac{1}{N} \left(
                \operatorname{KL} + \log \frac{\mathcal{I}_\Delta(N)}{\delta}\right).
\end{equation}
Finally, for any $\beta \in \mathbb{R} \cup \{ + \infty \}$, 
\begin{align}\label{eqn:monotone-sup}
p'_\beta \coloneqq \sup\,\left\{
            p \in [0, 1]
            :
            \Delta'(q, p) \leq \beta \right\} = \sup\,\left\{
            p \in [0, 1]
            :
             \Delta(q, p) \leq \beta \right\}=: p_\beta. 
\end{align}
One inequality follows from $\Delta' \leq \Delta$. For the other, for any $p' \geq p \ge p_\beta$, we have $\Delta(q,p') > \beta$. Taking an infimum over such $p'$, noting that $\Delta$ is lower semi-continuous and therefore obtains a minimum on the closed interval $[p,1]$, $\Delta'(q,p) > \beta$. Hence $p_\beta' \leq p_\beta$. The result follows from combining \cref{eqn:monotone-bound} and \cref{eqn:monotone-sup}. 
\end{proof}

\section{Lemmas for Theorem \ref{thm:nonsep-delta}}
\label{app:additional-lemmas}

Let $C_\beta(p, q) \coloneqq -\log(p(e^{-\beta} - 1) + 1) - \beta q$ for $\beta > 0$ and
\begin{equation}
    \mathcal{I}_\Delta(N)
    = \sup_{r\in [0,1]} \mathbb{E}_{X \sim \operatorname{Bin}(r, N)}[e^{N\Delta(X/N,r)}]
\end{equation}

\begin{lemma} \label{lem:sup-catoni}
    Consider $q, p \in [0, 1]$.
    Then
    \begin{equation}
        \sup_{\beta \in \R} C_{\beta}(q, p) = \kl(q, p).
    \end{equation}
\end{lemma}
\begin{proof}
    If $q = p = 0$, then $C_{\beta}(q, p) = 0 = \kl(q, p)$; and, if $q = p = 0$, then also $C_{\beta}(q, p) = 0 = \kl(q, p)$.
    If, on the other hand, $q = 0$ but $p = 1$, then clearly $\sup_{\beta \in \R} C_{\beta}(q, p) = \infty = \kl(q, p)$;
    and if $q = 1$ but $p = 0$, then also clearly $\sup_{\beta \in \R} C_{\beta}(q, p) = \infty = \kl(q, p)$.
    It remains to deal with the case that $q, p \in (0, 1)$.
    In that case, to compute the supremum, set the derivative to zero:
    \begin{equation}
        q = \frac{p e^{-\beta}}{p( e^{-\beta} - 1) + 1}
    \end{equation}
    and verify that we indeed have a maximum.
    This gives
    \begin{equation}
        -\beta = \log \frac{1 - p}{1 - q} + \log \frac{q}{p},
    \end{equation}
    so
    \begin{equation}
        -\log(p( e^{-\beta} - 1) + 1)
        = \beta + \log \frac{q}{p}
        = \log \frac{1 - q}{1 - p}.
    \end{equation}
    Therefore,
    \begin{equation}
        \sup_{\beta \in \R} C_\beta(q, p)
        = 
            - \log \frac{1 - p}{1 - q}
            + q\log \frac{1 - p}{1 - q} + q\log \frac{q}{p}
        = \operatorname{kl}(q, p).
    \end{equation}
\end{proof}

The following lemma is essentially Prop 2.1 from \citet{germain2009pac}, but stated in a slightly more careful form:

\begin{lemma}[\citet{germain2009pac}] \label{lem:sup-catoni-pos}
    Consider $0 \le q < p < 1$.
    If $q > 0$, then there exists a unique $\beta > 0$ such that
    \begin{equation}
        C_{\beta}(q, p) = \kl(q, p).
    \end{equation}
    On the other hand, if $q = 0$, then
    \begin{equation}
        \lim_{\beta \to \infty} C_{\beta}(0, p) = \kl(0, p).
    \end{equation}
\end{lemma}
\begin{proof}
    If $0 < q < p < 1$, then the unique $\beta$ follows from the proof of \cref{lem:sup-catoni}:
    \begin{equation}
        \beta = \log \frac{1-q}{1-p} + \log \frac{p}{q} \in (0, \infty). 
    \end{equation}
    On the other hand, if $q = 0 < p < 1$, then
    \begin{equation}
        \kl(0, p) = -\log(1 - p) = \lim_{\beta \to \infty} C_{\beta}(0, p). 
    \end{equation}
\end{proof}

\begin{lemma}[\citet{catoni2007pac,germain2009pac}] \label{lem:catoni-zero-I}
    For every $\beta > 0$, it holds that $\mathcal{I}_{C_\beta}(N) = 1$.
\end{lemma}
\begin{proof}
    Let $r \in [0, 1]$ and $X \sim \operatorname{Bin}(r, N)$.
    Note that
    \begin{equation}
        \E[e^{-\beta X}] = (r( e^{-\beta} - 1) + 1)^N.
    \end{equation}
    Therefore,
    \begin{equation}
        \mathbb{E}[e^{N C_\beta(X/N,r)}]
        = \mathbb{E}[e^{- N \log(r(e^{-\beta} - 1) + 1) - \beta X}] \\
        = \frac{\E[e^{-\beta X}]}{(r( e^{-\beta} - 1) + 1)^N} \\
        = 1.
    \end{equation}
\end{proof}

\begin{lemma} \label{lem:B_limit}
    Let $y \ge 0$.
    Then
    $
        B[\lim_{\beta \to \infty} C_{\beta}(0, \vardot), y] 
        = \lim_{\beta \to \infty} B[C_{\beta}(0, \vardot), y].
    $
\end{lemma}
\begin{proof}
    Note that
    \begin{equation}
        \lim_{\beta \to \infty} C_{\beta}(0, p)
        = -\log(1 - p),
        \qquad
        C_{\beta}(0, p)
        = -\log(p(e^{-\beta} - 1) + 1).
    \end{equation}
    Therefore,
    \begin{equation} 
        B[{\textstyle\lim_{\beta \to \infty}} C_{\beta}(0, \vardot), y] = 1 - e^{-y},
        \qquad
        B[C_{\beta}(0, \vardot), y] = \min\left(\frac{1 - e^{-y}}{1 - e^{-\beta}}, 1 \right)
    \end{equation}
    The result then follows from the observation that
    \begin{equation}
        \lim_{\beta \to \infty} \frac{1 - e^{-y}}{1 - e^{-\beta}}
        = 1 - e^{-y}.
    \end{equation}
\end{proof}

\section{Learning a Convex Function}
\label{app:learning-convex-function}

In \cref{sec:characterising-generic-bounds}, we optimised an objective with respect to a function $\Delta\colon [0, 1]^2 \to \mathbb{R} \cup \{+\infty\}$ that was proper, l.s.c., and convex.
In this appendix, we describe how a function $\Delta\colon [0, 1]^2 \to \mathbb{R}$ that is differentiable and convex can be generally parametrised. 
We also discuss two challenges encountered during the optimisation:
(1) computing and differentiating through a supremum and (2) computing and optimising a partial inverse.

\subsection{Parametrising a Convex Function}

To generally parametrise a differentiable and convex $\Delta\colon [0, 1]^2 \to \mathbb{R}$, we use the sum of an affine function and a one-hidden-layer neural network with softplus activation functions and positive weights at the output layer.
The combination of positive weights and softplus activation functions ensures that the neural network is a convex function.
The number of hidden units used is varied between $128$ and $1024$; the precise numbers are specified in the descriptions of the experiments.

\subsection{Computing and Differentiating Through a Supremum} \label{sec:diff_through_supremum}

The generic PAC-Bayes theorem (\cref{thm:begin-bound}) involves $\mathcal{I}_{\Delta}(N)$, which in turn involves a supremum of a function over $r \in [0, 1]$.
When optimising with respect to $\Delta$, we therefore need to compute and  differentiate through a supremum.
To compute the supremum, we finely discretise $[0, 1]$ and compute the maximum over this grid.
Technically, by approximating the supremum in this way, the bound is approximate, which means that it might not be a valid generalisation bound.
However, by making the discretisation very fine, using an inter-point spacing of $10^{-5}$, the error on the generalisation bound is negligible. 
To differentiate the supremum, we simply run automatic differentiation on the approximation.
In the remainder of this subsection, we give a plausible explanation for why this procedure also approximates the gradients correctly.
The following discussion is based on \url{https://math.stackexchange.com/questions/3753495/derivative-of-argmin-in-a-constrained-problem}.

Consider $f \colon [0, 1] \times \mathbb{R} \to \mathbb{R}$ continuously differentiable in its interior.
We aim to compute
\begin{equation}
    \frac{\sd}{\sd \theta} \sup_{x \in [0, 1]} f(x, \theta)
    = \frac{\sd}{\sd \theta} \max_{x \in [0, 1]} f(x, \theta)
\end{equation}
where the supremum turns into a maximum by compactness of $[0, 1]$ and continuity of $f$.
We assume that the maximum is uniquely obtained and write
\begin{equation}
    z(\theta) = \operatorname{arg\,max}_{x \in [0, 1]} f(x, \theta).
\end{equation}
Then
\begin{equation}
    \sup_{x \in [0, 1]} f(x, \theta) = f(z(\theta), \theta),
\end{equation}
so we can compute the derivative with respect to $\theta$ with the chain rule if we can compute $z'(\theta)$.

\textsc{Case 1:}
The constraint $x \in [0, 1]$ is not binding.
In that case, the stationarity condition is satisfied in a neighbourhood of $\theta$:
\begin{equation}
    \d_x f(z(\theta), \theta) = 0.
\end{equation}
Therefore,
\begin{equation}
    \frac{\sd}{\sd \theta} \sup_{x \in [0, 1]} f(x, \theta)
    = \d f(z(\theta), \theta) z'(\theta) + \d_\theta f(z(\theta), \theta)
    = \d_\theta f(z(\theta), \theta).
\end{equation}

\textsc{Case 2:}
The constraint $x \in [0, 1]$ is binding.
In that case, $\d_x f(z(\theta), \theta) \neq 0$, so $\d_x f(z(\theta), \theta) < 0$ and you can argue that the optimum will remain to be attained at the constraint in a neighbourhood of $\theta$.
Therefore, $z'(\theta) = 0$, which means that again
\begin{equation}
    \frac{\sd}{\sd \theta} \sup_{x \in [0, 1]} f(x, \theta)
    = \d f(z(\theta), \theta) z'(\theta) + \d_\theta f(z(\theta), \theta)
    = \d_\theta f(z(\theta), \theta).
\end{equation}

In either case,
\begin{equation}
    \frac{\sd}{\sd \theta} \sup_{x \in [0, 1]} f(x, \theta) = \d_\theta f(z(\theta), \theta),
\end{equation}
which shows that the derivative with respect to the maximiser can be ignored. Assuming that $z(\theta)$ can be well approximated by computing the maximiser over the fine discretisation, and that in turn leads to a good approximation of $\d_\theta f(z(\theta), \theta)$, this provides justification for our approach of simply running automatic differentiation on the approximation to the supremum.

Finally, we note that, although computing the derivative accurately is useful for the optimisation to succeed, the bounds we compute are valid regardless of how accurate the derivative is (subject to the computation of the supremum itself being sufficiently accurate). In practice, we observe that the learned convex bound decreases steadily during optimisation (\cref{fig:numerical_evidence_deterministic_case}), and that it approaches, but never goes below, the conjectured PAC-Bayes-kl bound, as per \cref{cor:optimistic-MLS}, which provides evidence that the implementation is sufficiently accurate for our purposes.

\subsection{Computing and Optimising a Partial Inverse}
\label{sec:diff_through_partial_inv}

The objective that we optimise with respect to $\Delta$ is $\overline{p}_\Delta$. Recall from \cref{eq:generic-gen-bound} that 
\begin{equation} 
    \overline{p}_\Delta \coloneqq B\big[\Delta(\overline{R}_S(Q), \vardot), \tfrac1{N}\big(\!\operatorname{KL}(Q\|P)+\log\tfrac{\mathcal{I}_\Delta(N)}{\delta}\big)\big].
\end{equation}
We now abbreviate $f \coloneqq \Delta(\overline{R}_S(Q), \vardot)$ and $c \coloneqq \tfrac1{N}\big(\!\operatorname{KL}(Q\|P)+\log\tfrac{\mathcal{I}_\Delta(N)}{\delta}\big)$, so that the objective is 
$
    B[f, c] = \sup\,\{p \in [0,1] : f(p) \le c\}
$ for $f \colon [0,1] \to \R$ convex and $c \in \R$.
Assuming that $f = f_\theta$ and $c = c(\theta)$ depend on some parameters $\theta$ (\textit{i.e.}, the parameters of the neural network defining $\Delta$), our goal is to compute $B[f_\theta, c(\theta)]$ and optimise it with respect to $\theta$.

A possible issue that can be run into during optimisation is that, if $f(p) \leq c$ for all $p \in [0, 1]$, then $B[f, c] = 1$ and the gradient with respect to $\theta$ may be zero, which means that the optimisation may fail to make progress.
A similar issue is discussed by \citet{dziugaite2017computing} when trying to optimise the PAC-Bayes-kl bound: the derivative of the inverse Bernoulli KL can be zero if $c$ is large enough. In \citet[Sec 2.2]{dziugaite2017computing} this is handled by upper bounding the inverse Bernoulli KL using Pinsker's inequality. This can lead to upper bounds which are greater than $1$ (whereas the exact computation of $B[f,c]$ never allows this to happen), but has the advantage of always providing a useful gradient signal. 

Similarly, we can define $\overline{B}[f, c] \coloneqq \sup\,\{p \in \R_{\ge0} : f(p) \le c\}$ for $f \colon \R_{\ge0} \to \R$ and $c \in \R$, which ignores the constraint that $p \le 1$.
Note that in our case, since $f_\theta$ is defined by a neural network, it is trivial to extend its domain from $[0,1]$ to $\mathbb{R}_{\ge0}$.
This will allow us to obtain a useful derivative even when the bound is vacuous.
Moreover, in our case $f_\theta$ will be convex, which means that $\overline{B}[f_\theta, c(\theta)] = B[f_\theta, c(\theta)]$ if $B[f_\theta, c(\theta)] < 1$:
$B[f_\theta, c(\theta)]$ is characterised by an upcrossing\footnote{
    We say that a function $f \colon \R \to \R$ upcrosses $y \in \R$ at $x \in \R$ if there exists some $\varepsilon > 0$ such that $f(x') < y$ for all $x' \in (x - \varepsilon, x)$ and $f(x') > y$ for all $x' \in (x, x + \varepsilon)$.
} of $c(\theta)$ by $f_\theta$, and, by convexity, $f_\theta$ can have at most one such upcrossing.

We now describe how $\overline{B}[f_\theta, c(\theta)]$ can be (approximately) computed and differentiated with respect to $\theta$.
To compute $\overline{B}[f_\theta, c(\theta)]$, we evaluate $f_\theta$ on a discretisation of the interval $[0, u]$ using an inter-point spacing of $10^{-4}$ and attempt to detect an upcrossing of $c(\theta)$.
Assume that this procedure finds an upcrossing;
otherwise, either increase $u$ (\textit{e.g.}, by doubling) and try again or return $u$ and set the derivative to zero (failure).
Denote $x = \overline{B}[f_\theta, c(\theta)]$ and note that $f_\theta$ is continuously differentiable in its interior, because it is a neural network with softplus activations.
Assume that $x > 0$, which in practice turns out to nearly always be the case.
Using continuity of $f_\theta$, it holds that $f_\theta(x) = c(\theta)$.
It also holds that $\d_x f_\theta(x) > 0$ ($f_\theta$ upcrosses $c(\theta)$ at $x$).
Restricting $f_\theta$ to an appropriate neighbourhood,
the derivative of $\overline{B}[f_\theta, c(\theta)]$ with respect to $\theta$ comes down to computing the derivative of $f^{-1}_\theta(c(\theta))$ with respect to $\theta$.
The latter derivative can be computed as follows:
\begin{equation}
    \d_\theta c(\theta)
    =
        \frac{\sd}{\sd \theta} f_\theta(f^{-1}_\theta(c(\theta)))
    =
        \d_x f_\theta(x)
        \frac{\sd}{\sd \theta} f^{-1}_\theta(c(\theta))
        +
        \d_\theta f_\theta(x),
\end{equation}
which implies that
\begin{equation}
    \frac{\sd}{\sd \theta} f^{-1}_\theta(c(\theta))
    = \frac{\d_\theta c(\theta) - \d_\theta f_\theta(x)}{\d_x f_\theta(x)},
\end{equation}
recalling that $\d_x f_\theta(x) > 0$.

Similarly to \cref{sec:diff_through_supremum}, although computing the gradient through the partial inverse accurately is useful for optimising the convex function, the bound itself will be valid as long as the value of the partial inverse itself is computed sufficiently accurately.

\section{Additional Results for Numerical Verification of Theory}
\label{app:additional-numerical-results}

\Cref{fig:deltas} complements \cref{fig:numerical_evidence_deterministic_case} by comparing the optimal Catoni $C_{\beta}$ to examples of the learned convex functions, which demonstrates that there are other choices than $C_{\beta}$ which achieve bounds that are close to $\inf_{\Delta \in \mathcal{C}} \overline{p}_\Delta$ within a small tolerance.
\Cref{fig:extra_numerical_evidence} complements \cref{fig:numerical_evidence_deterministic_case} by considering three slightly more complicated cases of a random dataset.
Note that, in \cref{fig:1M-1,fig:1M-2}, during iterations $10^4$--$10^6$, the optimiser struggles: the trace jumps around.
The are various reasons for why this might have happened: the neural network parametrising $\Delta$ has too few hidden units, the learning rate of the optimiser is too large, the various approximations that involve $\Delta$ (\cref{sec:diff_through_supremum,sec:diff_through_partial_inv}) introduce too much error, or \cref{open:lower-bound} might be false.

\begin{figure}[t] \small
    \centering
    \begin{subfigure}[t]{0.46\linewidth}
        \includegraphics[width=1\linewidth]{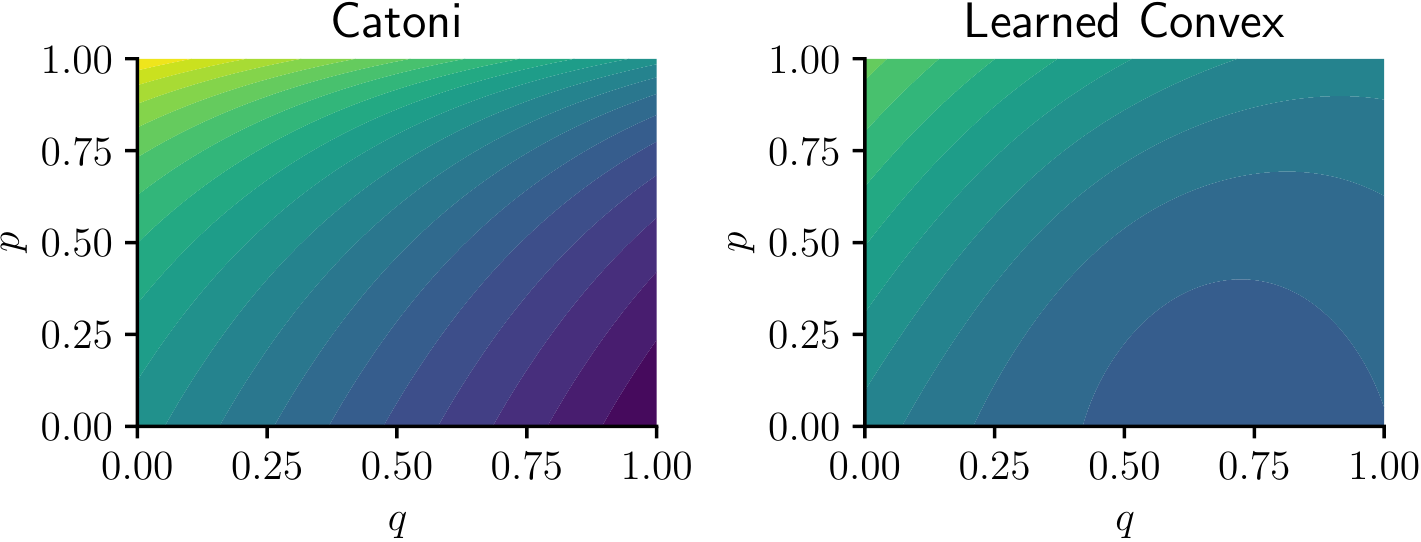}
        \caption{
            $(q, \operatorname{KL}) = (2\%, 1)$, $\beta^*\approx2.24$, $H = 256$.
        }
    \end{subfigure}
    \hfill
    \begin{subfigure}[t]{0.46\linewidth}
        \includegraphics[width=1\linewidth]{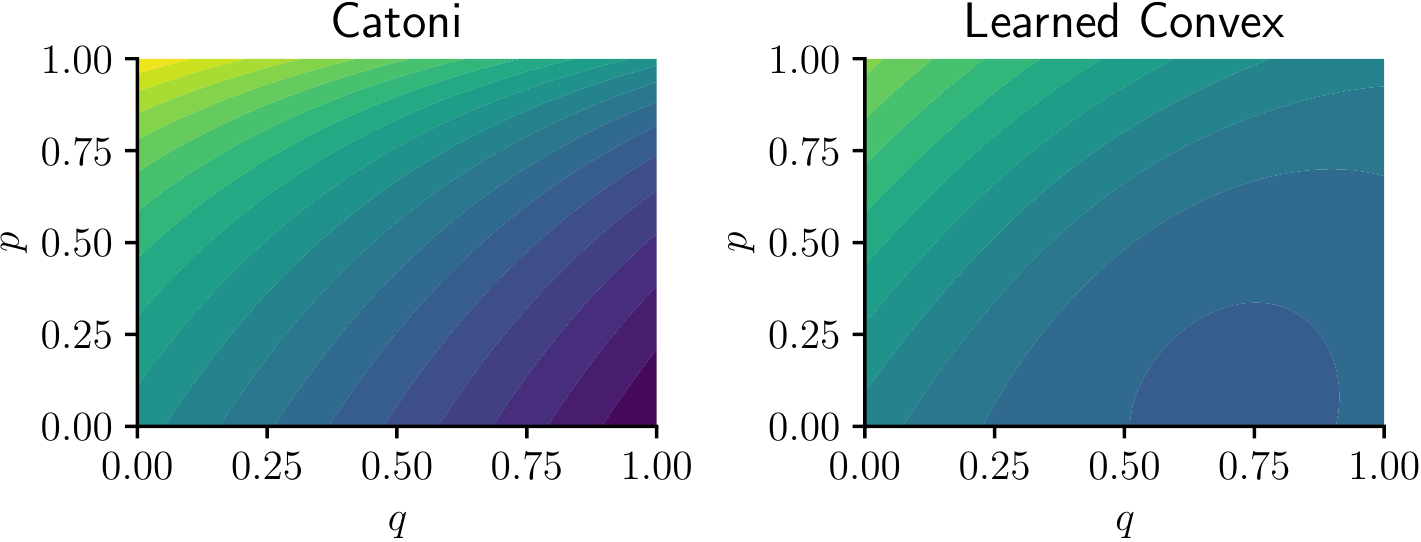}
        \caption{
            $(q, \operatorname{KL}) = (5\%, 2)$, $\beta^*\approx1.84$, $H = 256$.
        }
    \end{subfigure}\\[5pt]
    \begin{subfigure}[t]{0.46\linewidth}
        \includegraphics[width=1\linewidth]{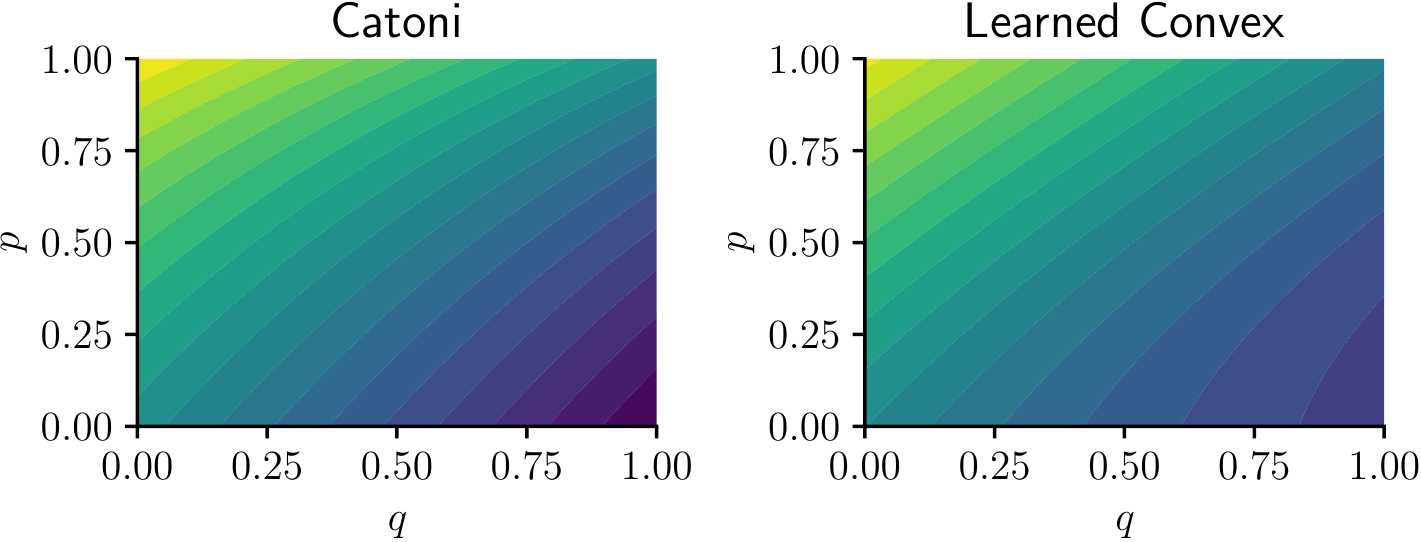}
        \caption{
            $(q, \operatorname{KL}) = (30\%, 1)$, $\beta^*\approx0.976$, $H = 256$.
        }
    \end{subfigure}
    \hfill
    \begin{subfigure}[t]{0.46\linewidth}
        \includegraphics[width=1\linewidth]{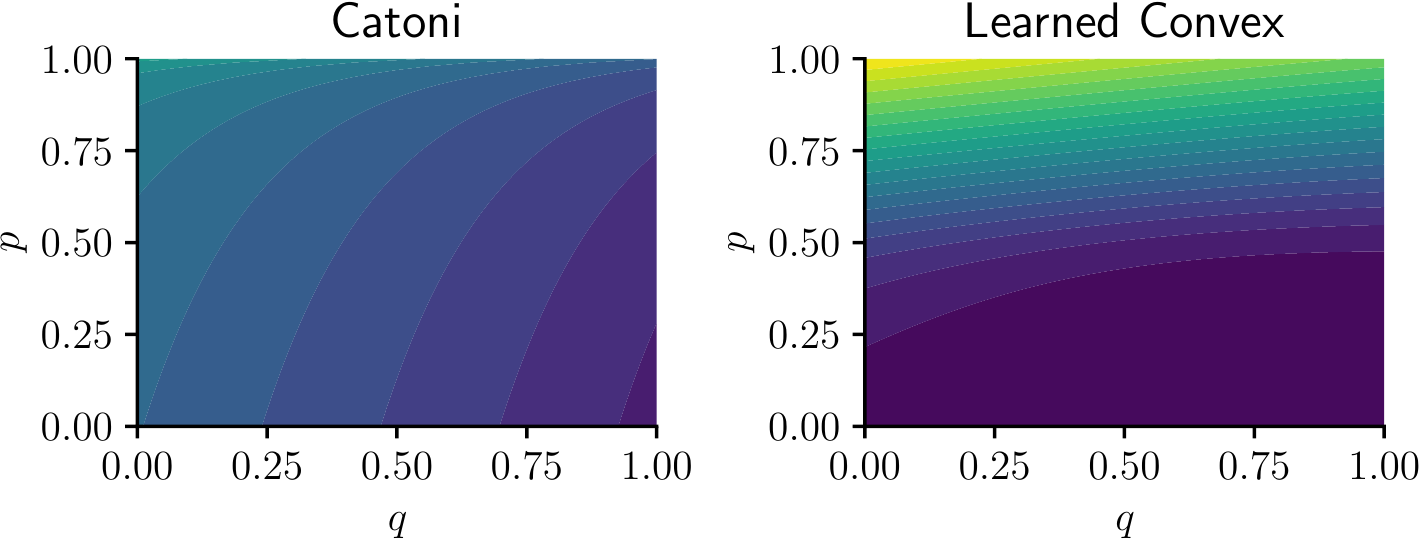}
        \caption{
            $(q, \operatorname{KL}) = (40\%, 50)$, $\beta^*\approx4.40$, $H = 256$.
        }
    \end{subfigure}
    \caption{\small
        \textbf{
            Although the tightest Catoni bound is the optimal generic PAC-Bayes bound for a fixed dataset, evidence suggests that there are other choices for $\Delta$ which also achieve the tightest bound.
        }
        For various settings of fixed $q$ and $\operatorname{KL}$, we optimise a convex function with $H$ hidden units to minimise $\overline{p}_{\Delta}$ with $\delta=0.1$ and $N=30$.
        We compare the optimal Catoni $C_{\beta^*}$ (left sides) with the learned convex function (right sides).
        For all runs, $\overline{p}_\Delta$ converged to $\smash{\inf_{\beta > 0} \overline{p}_{C_\beta}}$ within a small tolerance.
    }
    \label{fig:deltas}
\end{figure}

\begin{figure}[t] \small
    \centering
    ~\hfill
    \begin{subfigure}[t]{0.235\linewidth}
        \includegraphics[width=1\linewidth]{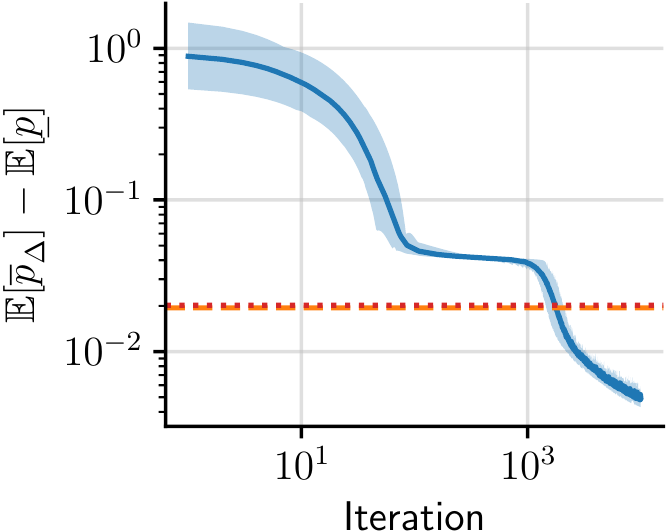}
        \caption{
        }
        \label{fig:stoch3}
    \end{subfigure}
    \hfill
    \begin{subfigure}[t]{0.235\linewidth}
        \includegraphics[width=1\linewidth]{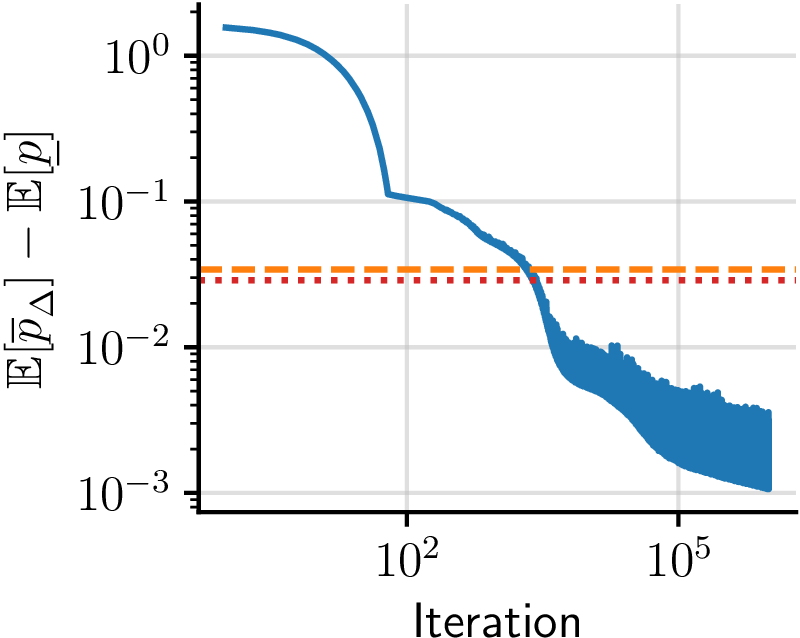}
        \caption{
        }
        \label{fig:1M-1}
    \end{subfigure}
    \hfill
    \begin{subfigure}[t]{0.235\linewidth}
        \includegraphics[width=1\linewidth]{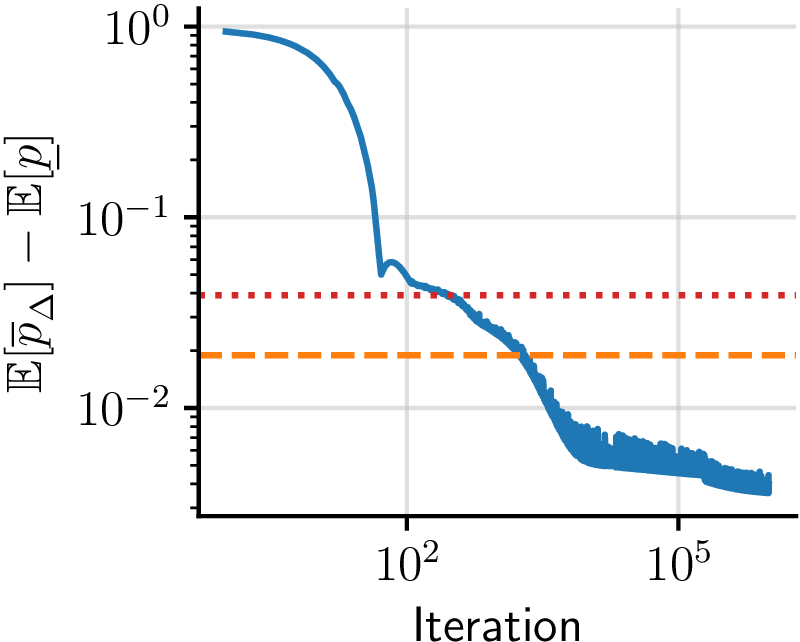}
        \caption{
        }
        \label{fig:1M-2}
    \end{subfigure}
    \hfill~
    \\
    \caption{\small
        \textbf{Extra results indicating that the tightest expected Catoni bound is not the optimal generic PAC-Bayes bound for a random dataset.} We optimise a convex function with $H$ hidden units to minimise $\overline{p}_\Delta$ with $\delta=0.1$ and $N=30$.
        All plots consider random $q$ and $\operatorname{KL}$ and show the \emph{expected} difference with the conjectured PAC-Bayes-kl bound (\cref{cor:optimistic-MLS}).
        In (a), the shaded regions show the minimum and maximum over ten initialisations.
        Due to computational considerations, (b) and (c) only show
        one repetition, since they are run for much longer than (a) (roughly 25 hours for each run).
        All plots show the PAC-Bayes-kl bound ({\color{plotred} dotted red}) and (c) and (d) show the optimal Catoni bound with parameter $\beta^*$ ({\color{plotorange} dashed orange}).
        All runs quickly converged to non-vacuous values.
        (a): $(q, \operatorname{KL}) \in \{(35\%, 5), (45\%, 30), (40\%, 7), (43\%, 25)\}$ uniformly, $\beta^*\approx2.23$, $H = 1024$.
        (b): $(q, \operatorname{KL}) \in \{(3.0\%, 46.0), (9.6\%, 0.52), (14.2\%, 48.9)\}$ (rounded values) uniformly, $\beta^*\approx3.21$, $H = 1024$.
        (c): $(q, \operatorname{KL}) \in \{(18.1\%, 2.56), (8.0\%, 5.83), (16.8\%, 30.0)\}$ (approximate values) uniformly, $\beta^*\approx2.20$, $H = 1024$.
    }
    \label{fig:extra_numerical_evidence}
\end{figure}

\section{Worked Example for Corollary \ref{cor:optimistic-MLS}}
\label{app:worked-example-optimistic-MLS}

In this section, we illustrate an application of \cref{cor:optimistic-MLS} to determine when, in the simplified scenario where $R_S(Q) = \smash{\tfrac12}$ almost surely, the expected PAC-Bayes-kl bound is tighter than the tightest expected Catoni bound.
This verifies with an analytic example, that, as we claim in \cref{sec:characterising-generic-bounds}, although the Catoni bound is optimal for a fixed dataset and learning algorithm, it is not optimal in expectation in the general case of a random dataset.
To make the example more concrete, we also compute the bounds for $\operatorname{KL} = 1$ with probability $\tfrac12$ and $\operatorname{KL} = 100$ otherwise.
This choice for $\operatorname{KL}$ is motivated with the conclusion at the end the section.

Denote $\alpha = \smash{\tfrac1N(\operatorname{KL}(Q\| P) + \log \tfrac{1}{\delta})}$.
We can solve for the optimal expected Catoni bound:
\begin{equation}
    \inf_{\beta > 0}\frac{1}{1 - e^{-\beta}}\left(1 - e^{-\tfrac12 \beta} \mathbb{E}[e^{-\alpha}] \right).
\end{equation}
Denote $u = \mathbb{E}[e^{-\alpha}]$ and set the derivative with respect to $\beta$ to zero:
\begin{equation}
    (1 - e^{-\beta})\tfrac12 u e^{-\tfrac12 \beta}
    - e^{-\beta}(1 - u e^{-\tfrac12 \beta}) = 0.
\end{equation}
Letting $x = e^{-\tfrac12 \beta}$,
this equation becomes a quadratic equation:
\begin{equation}
    x^2 - \frac{2}{u} x + 1 = 0
    \implies
    x = \frac1{u} \pm \sqrt{\frac{1}{u^2} - 1}.
\end{equation}
Therefore,
\begin{equation}
    \beta
    = 2 \log \frac{u}{1 \pm \sqrt{1 - u^2}}
    = 2 \log \frac{\mathbb{E}[e^{-\alpha}]}{1 \pm \sqrt{1 - \mathbb{E}[e^{-\alpha}]^2}},
\end{equation}
so the positive solution for $\beta$ is given by
\begin{equation}
    \beta = 2 \log \frac{\mathbb{E}[e^{-\alpha}]}{1 - \sqrt{1 - \mathbb{E}[e^{-\alpha}]^2}}
    \approx 2.803.
\end{equation}
Plugging this back into the Catoni bound gives that
\begin{equation}
    \inf_{\beta > 0} \mathbb{E}[\overline{p}_{C_\beta}]
    =
    \frac{\sqrt{1 - u^2}}{
        1 - \Big(
            \frac{1 - \sqrt{1 - u^2}}{u}
        \Big)^2
    }
    = \tfrac12 \left(
    1 + \sqrt{1 - \E[e^{-\alpha}]^2}
    \right)
    \approx 0.943.
\end{equation}
We compare this with the choice $\Delta = \operatorname{kl}$, which corresponds to the PAC-Bayes-kl bound.
In that case, the expected bound is given by
\begin{equation}
    \mathbb{E}[\overline{p}_{\kl}] = \mathbb{E}\,B[\operatorname{kl}(\tfrac12, \vardot), \alpha + \gamma]
\end{equation}
where
\begin{equation}
    \gamma
    = \frac1N \log \mathcal{I}_{\kl}(N)
    = \frac1N \log \sum_{n=0}^N \binom{N}{n}\left(\frac Nn\right)^n \left(1 - \frac Nn\right)^{N - n}.
\end{equation}
To compute $B[\operatorname{kl}(\tfrac12, \vardot), \alpha + \gamma]$, note that
\begin{equation}
    \tfrac12 \log \frac{\tfrac12}{p} + \tfrac12 \log \frac{\tfrac12}{1 - p} = y
    \implies
    y_\pm = \tfrac12 \left(
        1 \pm \sqrt{1 - e^{-2y}}
    \right).
\end{equation}
Therefore,
\begin{equation}
     \mathbb{E}[\overline{p}_{\kl}]
    = \tfrac12 \left(
        1 + \E\sqrt{1 - e^{-2(\alpha + \gamma)}}
    \right)
    \approx 0.836.
\end{equation}
This is better than the Catoni bound by more than $10\%$.
Finally, by omitting $\gamma$, we find the conjectured PAC-Bayes-kl bound:
\begin{equation}
    \mathbb{E}[\underline{p}] = \tfrac12 \left(
        1 + \E\sqrt{1 - e^{-2\alpha}}
    \right)
\end{equation}
Note how similar the computed bounds are:
\begin{align}
    \inf_{\beta > 0} \mathbb{E}[\overline{p}_{C_\beta}]
    &= \tfrac12 \left(
    1 + \sqrt{1 - \E[e^{-\alpha}]^2}
    \right), &&\text{(optimal Catoni)} \\
    \mathbb{E}[\overline{p}_{\kl}]
    &= \tfrac12 \left(
        1 + \E\sqrt{1 - e^{-2(\alpha + \gamma)}}
    \right), &&\text{(PAC-Bayes-kl)} \\
    \mathbb{E}[\underline{p}] &= \tfrac12 \left(
        1 + \E\sqrt{1 - e^{-2\alpha}}
    \right). &&\text{(conjectured PAC-Bayes-kl)}
\end{align}
Define
\begin{equation}
    \phi(x) = 1-\tfrac12\sqrt{1 - x^2},
\end{equation}
which is convex.
Define the $\phi$-entropy of a random variable $X$ by
\begin{equation}
    \H_\phi(X) = \E[\phi(X)] - \phi(\E[X]).
\end{equation}
Observe that $\H_\phi(X)$ quantifies the slack in Jensen's inequality, which, in particular, means that $\H_\phi(X) \ge 0$.
We then find that
\begin{align}
    \inf_{\beta > 0} \mathbb{E}[\overline{p}_{C_\beta}]
    - \mathbb{E}[\underline{p}]
    &= \H_\phi(e^{-\alpha}), \\
    \mathbb{E}[\overline{p}_{\kl}]
    - \mathbb{E}[\underline{p}]
    &= \E[\phi(e^{-\alpha}) - \phi(e^{-\gamma} e^{-\alpha})].
\end{align}
Therefore, the PAC-Bayes-kl bound is tighter if and only if
\begin{equation}
    \E[\phi(e^{-\alpha})]
    - \E[\phi(e^{-\gamma} e^{-\alpha})]
    \le
   \H_\phi(e^{-\alpha}).
\end{equation}
In words, the expected PAC-Bayes-kl bound is tighter than the tighest expected Catoni bound if the slack in Jensen's inequality is more than the slack introduced by scaling by $e^{-\gamma}$, which, for example, will be the case if $\operatorname{KL}(Q\|P)$ attains both small and large values.



\section{Additional Details for Synthetic Classification}

\subsection{Data Generation Details} \label{app:data}

We now provide details of the task generation for the 1D classification experiments. For each task, we sample a 1D function $f$ from a Gaussian process (GP) with an exponentiated quadratic kernel with lengthscale $0.7$ and variance $1$. This is then turned into a classification problem by thresholding: $S = \big( (x_n, \mathrm{sign}(f(x_n))) \big)_{n=1}^N $, where $x_n \sim \mathcal{U}[-2, 2]$. Finally, we only select tasks that are approximately balanced, so that the risk of a trivial predictor is $\approx 0.5$. This is done in a way that preserves the i.i.d.~assumptions. In more detail, when sampling from the GP, in addition to sampling the $N$ points that make up the dataset, we also sample an additional $300$ points that make up an extra held-out set which is unseen by any of the meta-learners, and whose sole purpose is for us to be able to estimate the actual generalisation risk of each posterior, which is what we report under ``Generalisation Risk'' in, \textit{e.g.}, \cref{fig:gen_bounds_1d_GNP_post_opt}. Furthermore, jointly with the $N + 300$ datapoints already sampled, we sample an \emph{additional} $300$ datapoints which form a ``balance set''. The sole purpose of the balance set is for us to check if the dataset is roughly balanced between positive and negative examples. If the prevalence of each class in the balance set is not $\approx 0.5$, then we discard the entire GP sample. Since the balance set is disjoint from the original dataset $N$ (and also the $300$ datapoints forming the extra held-out set), doing this does not jeopardise the i.i.d.~property within each dataset.\footnote{The tasks themselves are also still i.i.d.~from the same task distribution, although this does not affect the validitiy of our bounds, which only requires the i.i.d.~assumption to hold \emph{within} each dataset.} Approximately balancing the datasets in this way is convenient because it allows us to interpret results more easily, since the risk of the trivial classifier that always predicts the majority class in the observed dataset is $\approx 0.5$. 
We generate two disjoint meta-train sets (along with their corresponding meta-test sets) this way: one with $N=30$ and another with $N=60$. The meta-learners are either meta-trained and meta-tested exclusively with $N=30$ or exclusively with $N=60$. 

\subsection{Deterministic Classification for Test Set Model} \label{app:deterministic_val_model}

PAC-Bayes bounds naturally lead to stochastic classifiers (also known as Gibbs classifiers), whereby a fresh sample $h \sim Q$ is drawn whenever the classifier is presented with an input. However, this does not need to be the case for \emph{test set} bounds. In fact, it may be easier to bound the risk of a deterministic classifier with a test set bound than a stochastic one, since for a deterministic classifier, each term in the sum defining the empirical risk is a Bernoulli random variable, and hence it is trivial to apply \cref{lem:bin-inv-bound}, which leads to significantly tighter bounds than \cref{lem:kl-chernoff-bound} in the small data regime. 
Additionally, we observed that when optimising $\frac{1}{T} \sum_{t=1}^T \smash{\gibbsR_{S_{t, \mathrm{test}}}(Q_{\theta}(S_{t, \mathrm{train}}))}$ as in \cref{sec:classification}, the learned posterior map $Q_{\theta}$ eventually became essentially deterministic once meta-training was complete.

Another way to use the binomial tail test set bound in the case when $\ell \in [0,1]$ (as it effectively is for Gibbs classifiers, once the zero-one loss is integrated over $Q$ to form $\mathbb{E}_{h \sim Q} [\ell_{0/ 1} ((x, y), h)] \in [0,1]$), is to randomise the computation of the empirical loss. In particular, for each $z\in \testset$, one could sample a Bernoulli random variable with parameter $\ell(z, h)$ and set the empirical risk in \cref{lem:bin-inv-bound} to be the average of these Bernoulli random variables. For test set bounds, these are i.i.d.~hence the sum is binomially distributed and \cref{lem:bin-inv-bound} can be applied directly. We do not pursue this here, as it does not make a significant difference when the classifier is nearly deterministic, as we found.

For these reasons, at meta-\emph{test} time we convert the test set bound meta-learners into \emph{deterministic} classifiers by using a Bayes classifier instead of a Gibbs classifier. That is, the final predictor for a test set bound meta-learner with posterior $Q$ is given by 
\begin{equation}
    \hat{y}(x) \coloneqq \mathrm{sign} \left( \mathbb{E}_{w \sim Q} [ w^{\mathsf{T}} \phi_{\theta}(x) ] \right)
\end{equation}
The risk of this predictor on a dataset $S$, which is the quantity we report and upper bound for the test set bound meta-learners in \cref{sec:classification}, is then simply the usual (non-Gibbs) risk: $R_S(\hat{y}) = \frac{1}{N} \sum_{(x,y) \in S} \ell_{0/1} ((x, y), \hat{y})$. We emphasise that this change primarily serves to simplify the test set bound computation (and allow the use of the tighter \cref{lem:bin-inv-bound} instead of just \cref{lem:kl-chernoff-bound}), and essentially does not affect the performance of the test set classifiers --- the Gibbs and Bayes risks are nearly identical because the Gibbs classifier learned by the test set meta-learners was already nearly deterministic.

\subsection{Computing the Empirical Risk} \label{app:computing_emp_risk}

In this section we provide additional details on how to compute the empirical risk term for the meta-learners. This applies for the PAC-Bayes meta-learners at both meta-train time and meta-test time, but only applies to the test set bound meta-learners during meta-train time --- at meta-test time we use a Bayes classifier for the test set bound meta-learners instead of a Gibbs one; see \cref{app:deterministic_val_model} for a discussion. Recall that we consider hypotheses of the form $h_w(x) \coloneqq \mathrm{sign}(w^{\mathsf{T}} \phi_\theta(x))$. Then the loss function is:
\begin{align}
    \ell_{0/1}((x, y), h_w) &= \mathbbm{1} [ y \neq \mathrm{sign}(w^{\mathsf{T}} \phi_\theta(x)) ]
\end{align}

We can then compute the empirical Gibbs risk as
\begin{align} 
    \gibbsR_S(Q) &= \frac{1}{|S|} \sum_{(x, y) \in S} \mathbb{E}_{w \sim Q} [\mathbbm{1} [ y \neq \mathrm{sign}(w^{\mathsf{T}} \phi_\theta(x)) ]] \\
    &= \frac{1}{|S|} \sum_{(x, y) \in S} \mathrm{Pr}  \left( y w^{\mathsf{T}} \phi_\theta(x) < 0 \right) \label{eqn:empirical_risk}
\end{align}
We now specialise to the case of Gaussian $Q \coloneqq \mathcal{N}(\mu, \Sigma)$. In this case, we can compute the empirical Gibbs risk in \cref{eqn:empirical_risk} in closed form (up to the error function, which has a standard implementation in PyTorch \citep{paszke2017automatic}): 
\begin{align}
    y w^{\mathsf{T}} \phi_\theta(x) &\sim \mathcal{N}(y \mu^{\mathsf{T}}\phi_\theta (x), \, \phi_{\theta}(x)^{\mathsf{T}} \Sigma \phi_{\theta}(x) ), \\
    \mathrm{Pr}  \left( y w^{\mathsf{T}} \phi_\theta(x) < 0 \right) &= \Phi \left(\frac{-y \mu^{\mathsf{T}}\phi_\theta (x)}{\sqrt{\phi_{\theta}(x)^{\mathsf{T}} \Sigma \phi_{\theta}(x)}} \right)
\end{align}
where $\Phi$ is the standard normal cumulative distribution function and where we have used the fact that $y \in \{ -1, +1 \}$ so $y^2 = 1$. Now recall that $\Phi$ is related to the error function $\mathrm{erf}(x)$ (as defined in PyTorch) by $\Phi(x) = \frac{1}{2} [1 + \mathrm{erf}(\frac{x}{\sqrt{2}})]$, which gives:
\begin{align}
    \gibbsR_S(Q) &= \frac{1}{|S|} \sum_{(x, y) \in S} \frac{1}{2} \left[ 1 + \mathrm{erf}\left( \frac{-y \mu^{\mathsf{T}}\phi_\theta (x)}{\sqrt{2\phi_{\theta}(x)^{\mathsf{T}} \Sigma \phi_{\theta}(x)}} \right) \right].
\end{align}
Hence we can backpropagate through the empirical Gibbs risk without the need for Monte Carlo integration over $Q$.

\subsection{Post-Hoc Optimisation of Posteriors} \label{app:post_optimisation}

It is well-known that when performing amortised variational inference (VI) \citep{kingma2013auto}, there is an \emph{amortisation gap} \citep{cremer2018inference}, which is the gap between the performance of the amortised inference network, and the performance obtained when optimising each variational problem separately. The meta-learners we consider in \cref{sec:classification} have similarities with amortised VI, except that PAC-Bayes bound minimisation is being amortised, rather than VI. Similarly, there is an amortisation gap for our meta-learners, which is the gap between the bound obtained by the meta-learner when the posterior that was outputted by the posterior map is directly used, versus the bound we obtain when optimising, for each dataset, the posterior using gradient-based methods (in our case, ADAM \citep{kingma2014adam}).
Optimising the posterior for each dataset individually is costly, but since we are concerned with obtaining the tightest bounds possible, we perform this \emph{post-hoc optimisation} for all of our meta-learners (including the results reported in \cref{sec:classification}). Fortunately, each optimisation does not take too long, since we can initialise the posterior at the distribution output by the meta-learner. 

So far, we have discussed post-hoc optimisation of the PAC-Bayes bound. However, we can also run post-hoc optimisation for the test set bound, \emph{as long as the optimised posterior does not depend on the test set}. We consider post-hoc optimising the \emph{train risk} for each dataset. In principle, this could possibly lead to overfitting the train set. In practice, we observe that this \emph{improves} performance slightly for the MLP-NP (indicating that the MLP-NP test set meta-learner was underfitting the train set somewhat), and leaves performance essentially completely unaffected for the CNN-NP, because the train set risk is \emph{already} essentially zero for the CNN-NP test set meta-learner before post-optimisation. Note that post-hoc optimisation is completely legal as a means of obtaining bounds --- it does not affect the validity of the bounds we consider, but merely closes the amortisation gap.

As an ablation study, we can compare the performance of the meta-learners with and without post-hoc optimisation. \Cref{fig:gen_bounds_1d_GNP_no_post_opt,fig:gen_bounds_MLP_1d_no_post_opt} show the performance of the CNN-NP and MLP-NP meta-learners \emph{without} post-hoc optimisation, which should be compared to \cref{fig:gen_bounds_1d_GNP_post_opt,fig:gen_bounds_1d_MLP_post_opt}, which show their performance \emph{with post-optimisation}. Comparing \cref{fig:gen_bounds_1d_GNP_no_post_opt} with \cref{fig:gen_bounds_1d_GNP_post_opt} we see that post-hoc optimisation improves the performance of the PAC-Bayes meta-learners slightly but leaves the test set meta-learners essentially unaffected for the CNN-NP. Comparing \cref{fig:gen_bounds_MLP_1d_no_post_opt} with \cref{fig:gen_bounds_1d_MLP_post_opt}, we see that post-hoc optimisation tightens the generalisation bounds for all meta-learners slightly. In conclusion, post-hoc optimisation sometimes leads to a small benefit, so we perform it for all meta-learners.

\begin{figure}[h]
\centering
\begin{subfigure}{.49\textwidth}
  \centering
  \includegraphics[width=.49\linewidth]{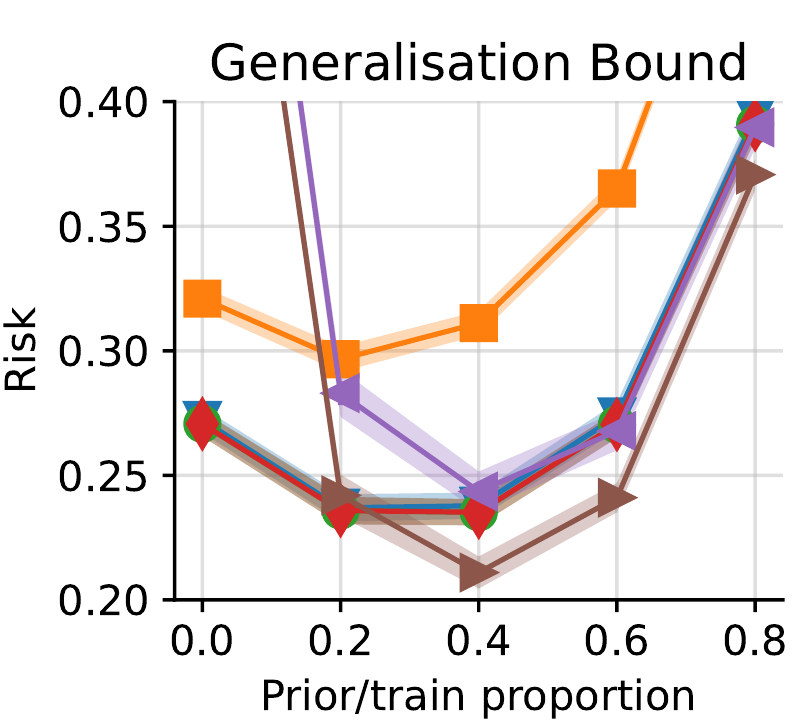}
  \includegraphics[width=.49\linewidth]{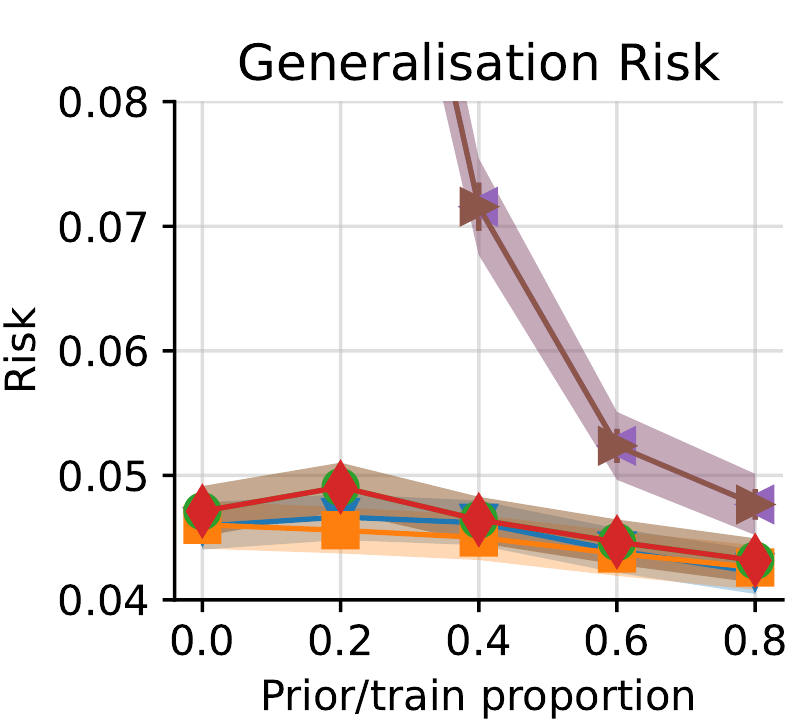}
  \caption{$N=30$ datapoints.}
\end{subfigure}%
\begin{subfigure}{.49\textwidth}
  \centering
  \includegraphics[width=.49\linewidth]{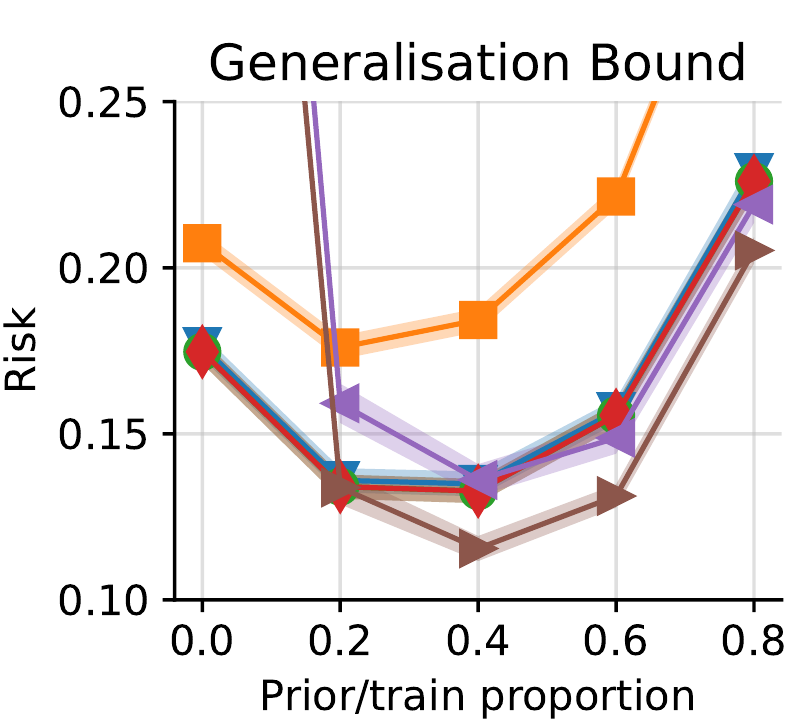}
  \includegraphics[width=.49\linewidth]{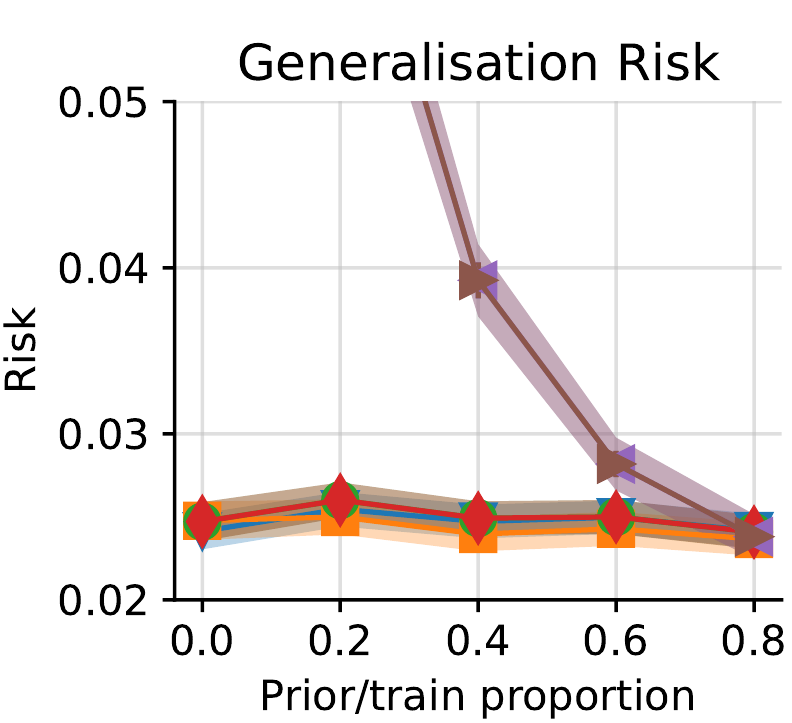}
  \caption{$N=60$ datapoints.}
\end{subfigure} 
\caption{Average generalisation bound and actual generalisation risk \textbf{for CNN-NP without post-hoc optimisation} ($\pm$ two standard errors) for Catoni ($\textcolor{plotblue}{\blacktriangledown}$), PAC-Bayes-kl ($\textcolor{plotorange}{\blacksquare}$), optimisitc PAC-Bayes-kl
($\textcolor{plotgreen}{\sbullet[1.5]}$), learned convex
($\textcolor{plotred}{\blacklozenge}$), Chernoff test set bound
($\textcolor{plotpurple}{\blacktriangleleft}$), 
and binomial tail test set bound 
($\textcolor{plotbrown}{\blacktriangleright}$). All bounds hold with failure probability $\delta = 0.1$ except for conjectured PAC-Bayes-kl which is not proven to be a valid generalisation bound.} \label{fig:gen_bounds_1d_GNP_no_post_opt}
\end{figure}

\subsection{The Multilayer Perceptron Neural Process} \label{app:MLP-NP}

We now describe the multilayer perceptron (MLP)-NP model, which is closely related to (but not identical with) the \emph{conditional neural process} model first described in \citet{garnelo2018conditional}.\footnote{The original conditional neural process outputs a Gaussian distribution directly in function space. This leads to complications when considering the KL term in the PAC-Bayes bounds, hence we modify it to output a Gaussian distribution over the parameters of a linear model.} When using the MLP-NP, we use an MLP to implement the feature map $\phi_{\theta}$. Additionally, each of the maps $Q_{\theta}$/$P_{\theta}$ consists of two MLPs: the \emph{encoder} and \emph{decoder}. The encoder $e_{\theta}$ maps $\mathbb{R} \times \{-1 , +1\} \to \mathbb{R}^M$, where $\mathbb{R}^M$, $M \in \mathbb{N}$ is the \emph{representation space}. The decoder $d_{\theta}$ maps $\mathbb{R}^M \to \mathcal{G}(\mathbb{R}^K)$, where $\mathcal{G}(\mathbb{R}^K)$ is the set of all Gaussian distributions over $\mathbb{R}^K$ (in practice, the decoder outputs a vector in $\mathbb{R}^{K + K(K+1)/2}$, which is converted into the mean of the Gaussian, and also the lower-triangular part of the Cholesky decomposition of the covariance matrix). When given a dataset $S$, the encoder computes a permutation-invariant representation of the dataset as $r_{\theta}(S) \coloneqq \frac{1}{N} \sum_{(x, y) \in S} e_{\theta}((x, y))$. The decoder $d_{\theta}$ then computes a Gaussian posterior distribution over the hypothesis space as $d_{\theta}(r_{\theta}(S))$.

\subsection{The CNN-Based Gaussian Neural Process} \label{app:GNP}

In contrast to the MLP-NP, which uses MLPs to implement the feature map $\phi_\theta$, the CNN-Based Gaussian Neural Process \citep{bruinsma2021} (CNN-NP) lets the $k$\textsuperscript{th} component of the feature map be $\phi_{\theta,k}(x) = \exp(-\frac1{2\ell^2}(x-x_k)^2)$, a Gaussian basis function centred at some fixed input $x_k$, with a learnable lengthscale $\ell$.
The centres of the Gaussian basis functions $\smash{(x_k)_{k=1}^K}$ are evenly spread out through the interval $[-2, 2]$.
The CNN-NP lets $Q_\theta$ and $P_\theta$ be parametrisations of maps from datasets to full-covariance Gaussian posteriors over the weights of these basis functions where the maps incorporate a symmetry called \emph{translation equivariance}:
if all inputs of the observed data are shifted by some amount, then the posterior over the weights for the basis functions should be shifted accordingly.
Translation equivariance enables the CNN-NP to use CNNs for $Q_\theta$ and $P_\theta$ instead of MLPs.

We now give a brief high-level description how the CNN architecture for the posterior mean of the Gaussian works. This follows the way that the mean of the \emph{Convolutional Conditional Neural Process} (ConvCNP) is computed,\footnote{The predictive \emph{mean} of the ConvCNP \citep{gordon2019convolutional} and that of the later, full-covariance Gaussian Neural Process \citep{bruinsma2021} are computed in the same way.} and we refer the reader to Sec 4 and especially Fig 1 of \citet{gordon2019convolutional} for a full description. First, the dataset is embedded into a 1D function with two channels, known as the \emph{data channel} and the \emph{density channel}. This 1D function is then evaluated on a discretised grid, and then fed into a CNN. The CNN output then defines mean of the Gaussian predictive distribution over functions. However, unlike in \citet{gordon2019convolutional} and \citet{bruinsma2021}, we modify this setup slightly, so that, instead of interpreting the CNN output as the mean of the Gaussian predictive over functions, it is interpreted as the mean of the Gaussian posterior over \emph{weights} of the basis functions in $\phi_{\theta}$. Defining the posteriors in weight space instead of function space makes it much easier to compute the KL-divergence.

We also give a brief description of how the CNN architecture for computing the posterior \emph{covariance} works. As this computation is more involved than the computation for the mean, we refer
the reader to Sec 3 and App E.2 of \citet{bruinsma2021} for a detailed description, \url{https://github.com/wesselb/NeuralProcesses.jl} for a full implementation, and \citet{bruinsmatalk} for a useful visual description of the Gaussian neural process architecture, on which we base our CNN-NP architecture used in \cref{sec:classification}.
To compute the covariance matrix of the weights of the basis functions, the dataset $S$ is first embedded into three images on $[-2, 2] \times [-2, 2]$.
The embedding is performed by placing a Gaussian basis function\footnote{These basis functions are distinct from the basis functions used to define the feature map $\phi_{\theta}$.} corresponding to each datapoint along the \emph{diagonal} of the $[-2, 2] \times [-2, 2]$ square.
These three images are known as the \emph{data channel}, \emph{density channel} and \emph{source channel} respectively.
As explained in \citet{bruinsma2021}, the data channel incorporates information about the $y$-values of the observations in $S$, the density channel records information about how many points in $S$ are observed at any particular $x$-location, and the source channel is simply in the shape of an identity matrix which, intuitively speaking, allows CNN-NP to begin with a ``white noise'' covariance matrix that afterwards is modulated to include correlations.
These continuous images, after being appropriately discretised on a regular 2D grid\footnote{This discretisation need not be the same as the spacing used for the Gaussian basis functions which make up the feature map $\phi_{\theta}$.} are passed through a 2D CNN, which outputs an image which is interpreted as a covariance matrix over the interval $[-2, 2]$. In order to ensure that the covariance matrix output is positive semi-definite, we multiply the output by itself: $\Sigma \coloneqq M \smash{M^{\mathsf{T}}}$, where $\Sigma$ is the covariance matrix and $M$ is the $K \times K$ matrix output by the CNN. 
This covariance matrix is finally interpolated onto the grid defined by the locations of the basis functions in $\phi_{\theta}$, which then defines the covariance of the weights of the basis functions.

\subsection{Hyperparameters} \label{app:hyperparameters}

\shortsubsection{General meta-learner training details.}
We fix the failure probability at $\delta=0.1$ for all of the meta-learning experiments. During meta-training, for the PAC-Bayes models we found it was more numerically stable to optimise the logarithm of the objective described in \cref{eqn:pac-bayes-batch-obj}. In particular, for the Catoni bound model, we used the numerically stable implementation of $\log(1 - e^{-x})$ referenced in \url{https://github.com/pytorch/pytorch/issues/39242}. For all meta-learners, we use a mini-batch estimate of the objective in \cref{eqn:pac-bayes-batch-obj}, with a batch size of $16$ datasets. We use a weight decay of $1\times 10^{-5}$ for all meta-learners. 

\shortsubsection{MLP-NP hyperparameters.}
We use a relatively large architecture for the MLP-NP, as we found during preliminary experiments that larger architectures performed better. The feature dimension of the linear model (see \cref{app:MLP-NP}) was set at $K=256$. The MLPs implementing the feature map $\phi_{\theta}$, encoder $e_{\theta}$ and decoder $d_{\theta}$ all had two hidden layers, each with a width of $512$. 
The MLP-NP was trained for $100$ epochs on the meta-train set, with a learning rate of $2\times 10^{-5}$ (we found that higher learning rates could lead to instabilities during training) with the ADAM optimiser \citep{kingma2014adam}.
We did some manual hyperparameter tuning to choose these hyperparameters, but they were not selected exhaustively. To avoid overfitting to the meta-train set when doing manual hyperparameter tuning, we also sampled a meta-validation set of datasets, which we used when tuning hyperparameters.

\shortsubsection{CNN-NP hyperparameters.}
For the CNN in the architecture, we use a U-Net \citep{ronneberger2015u}. The U-Net, we use has $12$ layers, with the number of channels in each layer being $8, 16, 16, 32, 32, 64, 64, 64, 64, 32, 32, 16$ respectively. This architecture matches that used by \citet{gordon2019convolutional}. The number of Gaussian basis functions per unit of input space (which determines the number of features in $\phi_{\theta}$) was set at $16$. The discretisation of the Gaussian Neural Process (\textit{i.e.}, the spacing at which the continuous representation is evaluated before passing it through the CNN) is set at $32$ points per unit. Then CNN-NP was trained for $10$ epochs on the meta-train set, with a learning rate of $1\times 10^{-3}$ with the ADAM optimiser. We did very little manual hyperparameter tuning for the CNN-NP, because we found that it was fairly robust to the choice of learning rate and basis function spacing. In all cases, the CNN-NP optimised much more quickly than the MLP-NP.

\shortsubsection{Post-hoc optimisation.}
We perform post-hoc optimisation at meta-test time, as discussed in \cref{app:post_optimisation}. Given a dataset, we initialise the posterior at the Gaussian distribution which is output by the NP. We then use the ADAM optimiser \citep{kingma2014adam} with a learning rate of $3 \times 10^{-4}$ for a maximum of $3\,000$ optimisation steps to target either the PAC-Bayes bound (for PAC-Bayes meta-learners), or the train risk (for test set meta-learners). If, after $100$ optimisation steps, the generalisation bound has not decreased by at least $0.0001$, then the optimisation is ended early.

\shortsubsection{Compute.}
We used roughly 500--1000 GPU-hours divided NVIDIA Tesla V100 and GeForce RTX 2080 graphics cards using both an internal cluster and Amazon Web Services. Most of the computational budget was spent on the meta-learning experiments. Of these, the MLP-NP was more costly to run than then CNN-NP, since it took longer to train.

\section{Additional Plots for Synthetic Classification}

\subsection{Predictive Distributions} \label{app:additional_plots_predictive}
In this appendix we include extra plots of 1D classification tasks from the meta-test set, similar to \cref{fig:example_classification_kl-val_catoni}.

\newpage

\begin{figure}[h]
\centering
\begin{subfigure}{.49\textwidth}
  \centering
  \includegraphics[width=.99\linewidth]{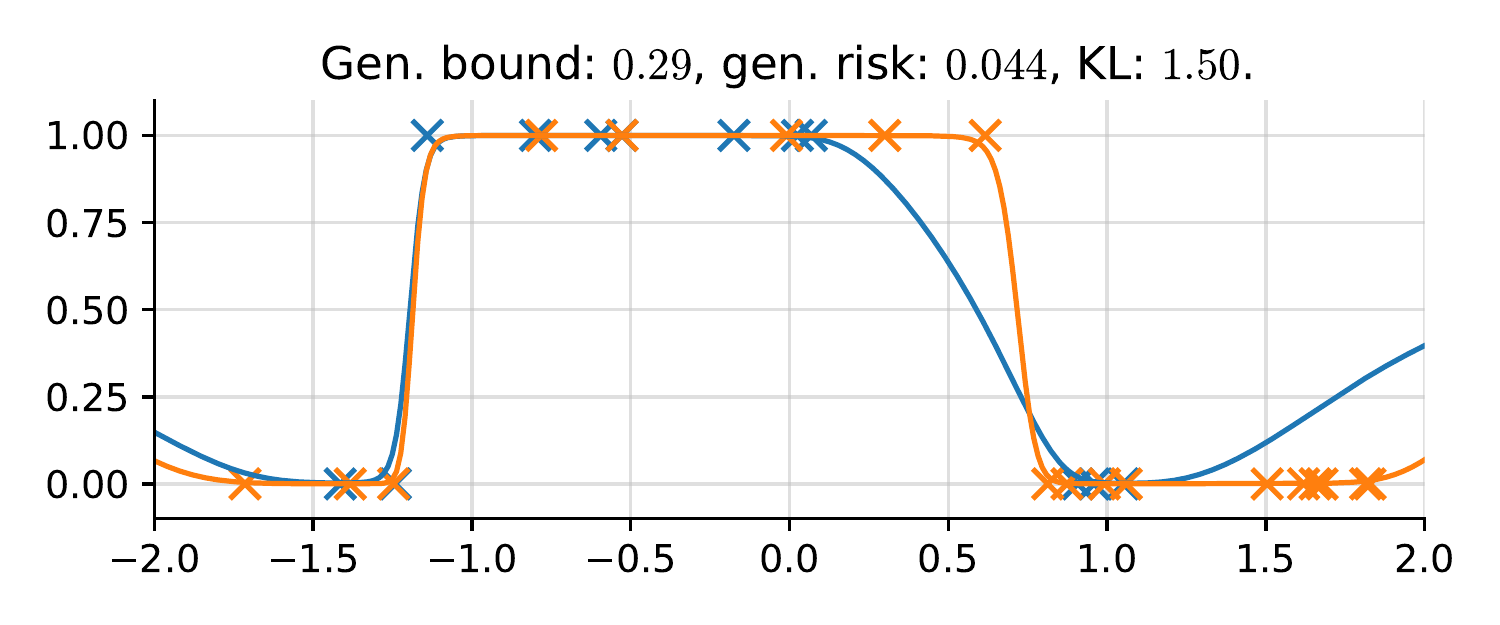}
  \caption{\textbf{PAC-Bayes-kl} bound with data-dependent prior, showing the prior (\textcolor{plotblue}{---}) and posterior (\textcolor{plotorange}{---}) predictive, prior set (\textcolor{plotblue}{{\tiny \XSolid}}) of size 12 and risk set (\textcolor{plotorange}{{\tiny \XSolid}}) of size 18.}
\end{subfigure}\hfill%
\begin{subfigure}{.49\textwidth}
  \centering
  \includegraphics[width=.99\linewidth]{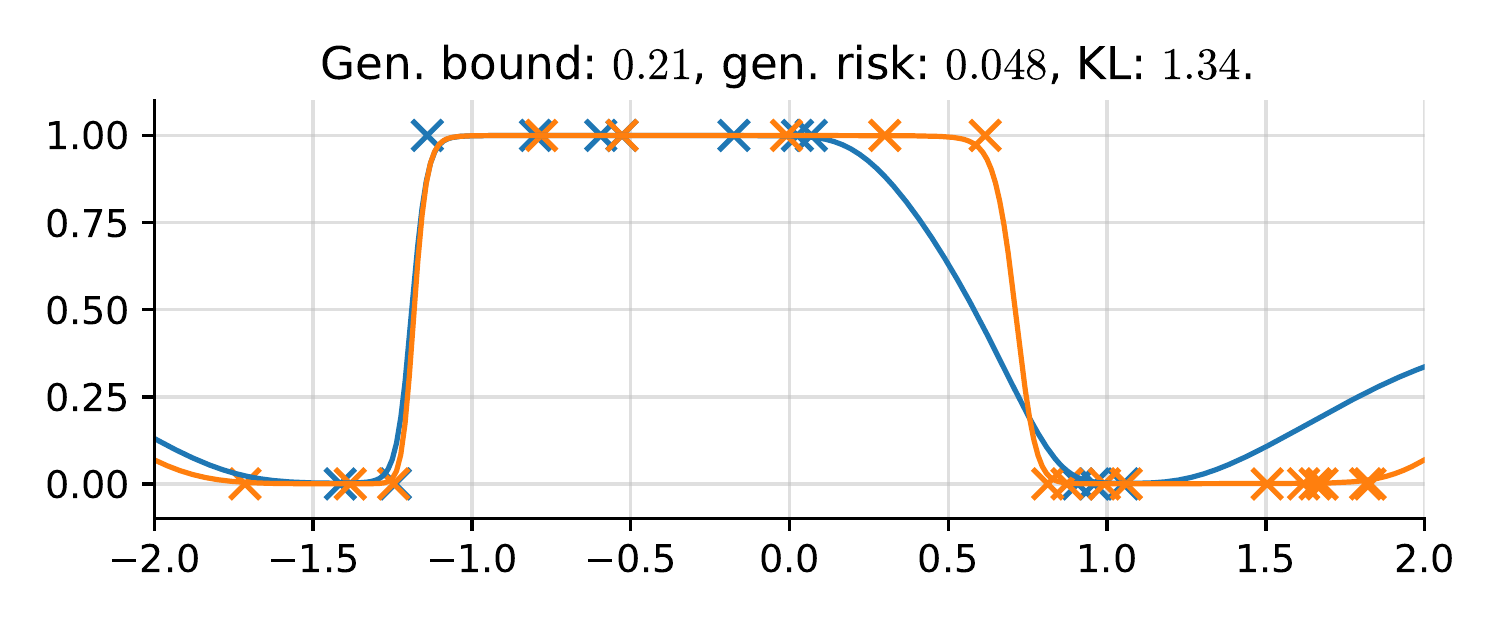}
  \caption{\textbf{Catoni} bound with data-dependent prior, showing the prior (\textcolor{plotblue}{---}) and posterior (\textcolor{plotorange}{---}) predictive, prior set (\textcolor{plotblue}{{\tiny \XSolid}}) of size 12 and risk set (\textcolor{plotorange}{{\tiny \XSolid}}) of size 18.}
\end{subfigure}
\caption{Predictions and bounds on one of the 1D datasets with \textbf{30 datapoints}. The bounds hold with failure probability $\delta = 0.1$.}
\end{figure}

\begin{figure}[h]
\centering
\begin{subfigure}{.49\textwidth}
  \centering
  \includegraphics[width=.99\linewidth]{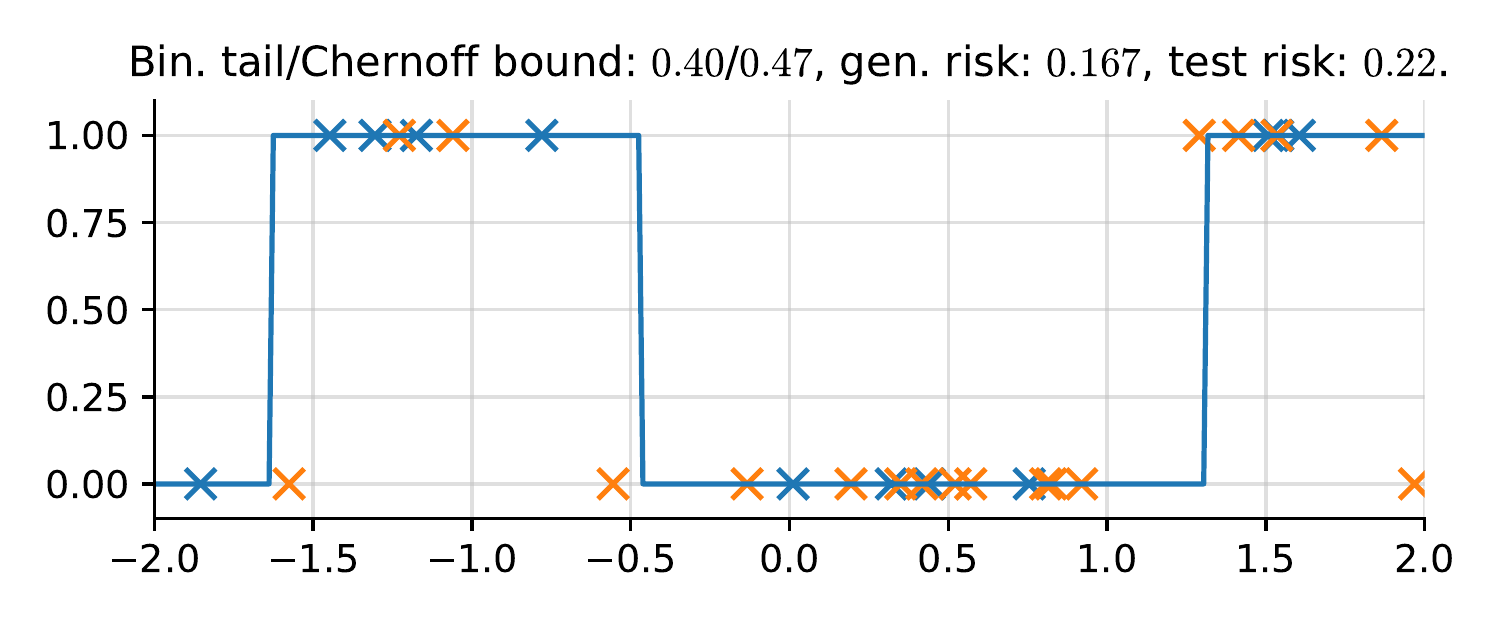}
  \caption{\textbf{Test set} bounds, showing the learned hypothesis, (\textcolor{plotblue}{---}), the train set (\textcolor{plotblue}{{\tiny \XSolid}}) of size 12 and the test set (\textcolor{plotorange}{{\tiny \XSolid}}) of size 18.}
\end{subfigure}\hfill%
\begin{subfigure}{.49\textwidth}
  \centering
  \includegraphics[width=.99\linewidth]{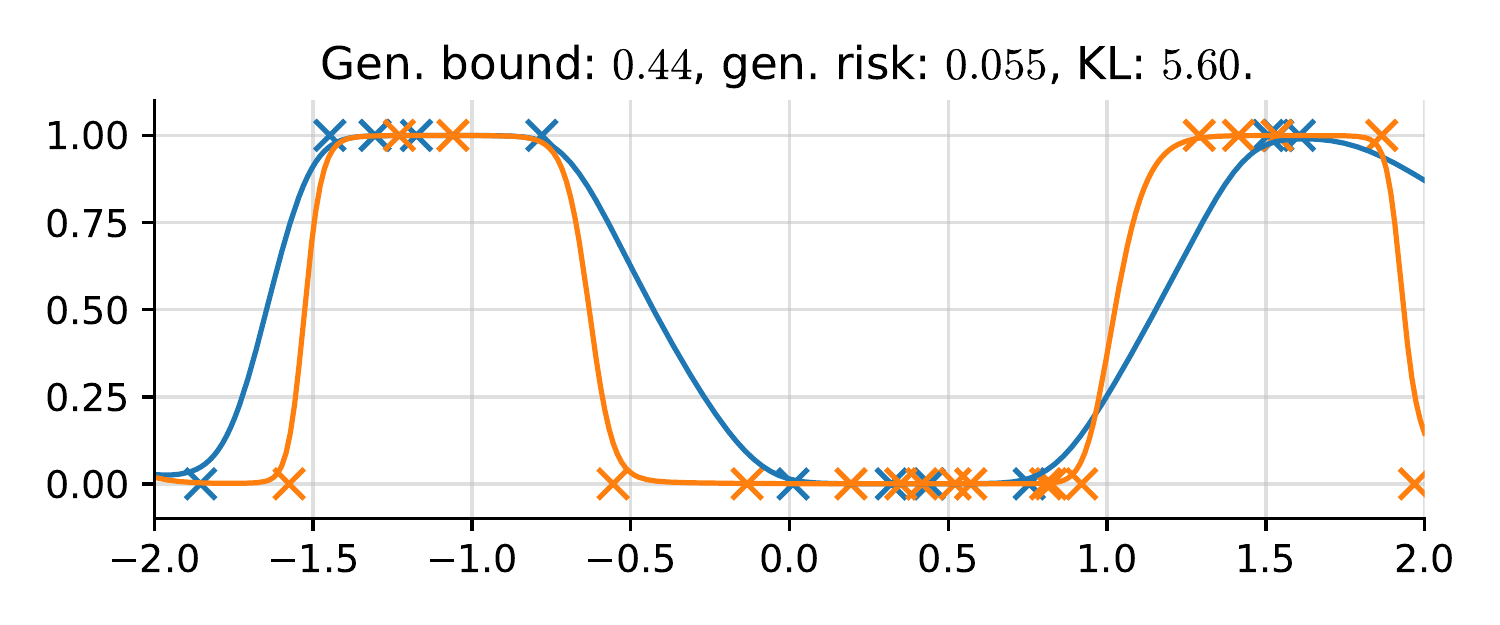}
  \caption{\textbf{Catoni} bound with data-dependent prior, showing the prior (\textcolor{plotblue}{---}) and posterior (\textcolor{plotorange}{---}) predictive, prior set (\textcolor{plotblue}{{\tiny \XSolid}}) of size 12 and risk set (\textcolor{plotorange}{{\tiny \XSolid}}) of size 18.}
\end{subfigure}
\caption{Predictions and bounds on one of the 1D datasets with \textbf{30 datapoints}. The bounds hold with failure probability $\delta = 0.1$.}
\end{figure}

\begin{figure}[h]
\centering
\begin{subfigure}{.49\textwidth}
  \centering
  \includegraphics[width=.99\linewidth]{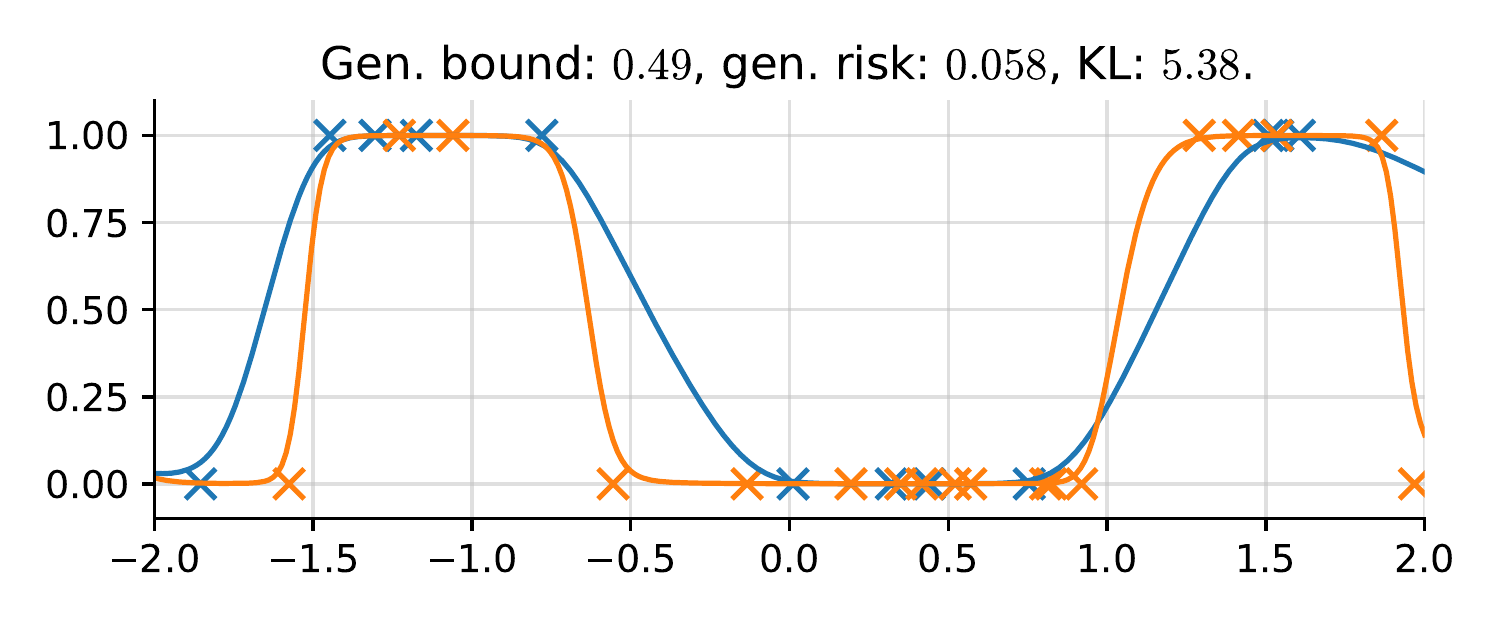}
  \caption{\textbf{PAC-Bayes-kl} bound with data-dependent prior, showing the prior (\textcolor{plotblue}{---}) and posterior (\textcolor{plotorange}{---}) predictive, prior set (\textcolor{plotblue}{{\tiny \XSolid}}) of size 12 and risk set (\textcolor{plotorange}{{\tiny \XSolid}}) of size 18.}
\end{subfigure}\hfill%
\begin{subfigure}{.49\textwidth}
  \centering
  \includegraphics[width=.99\linewidth]{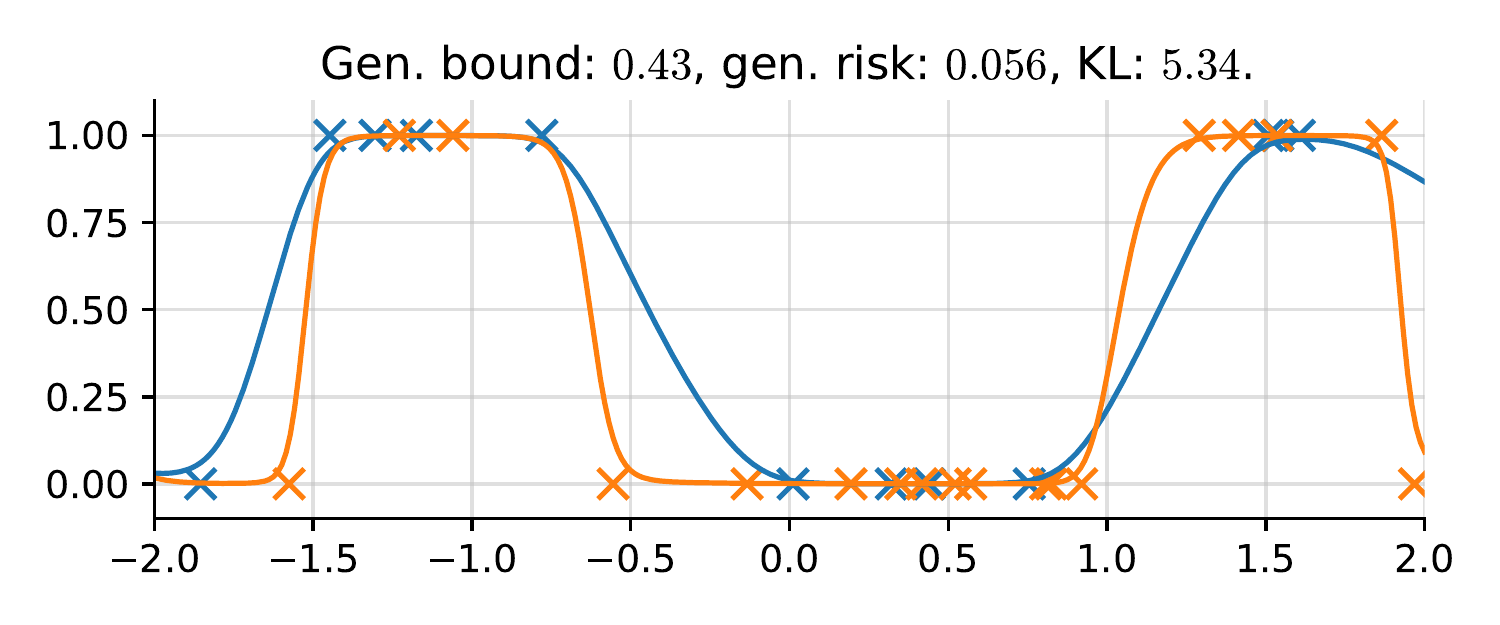}
  \caption{\textbf{Learned convex} bound with data-dependent prior, showing the prior (\textcolor{plotblue}{---}) and posterior (\textcolor{plotorange}{---}) predictive, prior set (\textcolor{plotblue}{{\tiny \XSolid}}) of size 12 and risk set (\textcolor{plotorange}{{\tiny \XSolid}}) of size 18.}
\end{subfigure}
\caption{Predictions and bounds on one of the 1D datasets with \textbf{30 datapoints}. The bounds hold with failure probability $\delta = 0.1$.}
\end{figure}

\begin{figure}[h]
\centering
\begin{subfigure}{.49\textwidth}
  \centering
  \includegraphics[width=.99\linewidth]{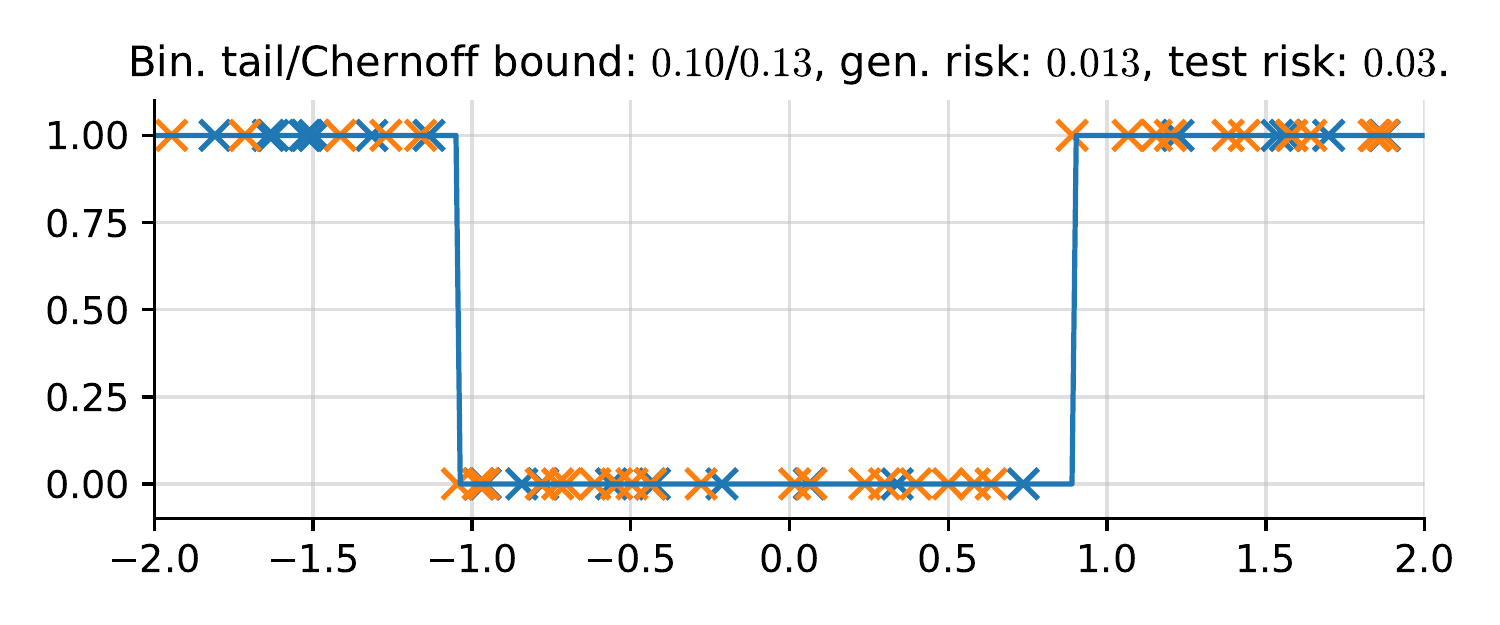}
  \caption{\textbf{Test set bounds}, showing the learned hypothesis, (\textcolor{plotblue}{---}), the train set (\textcolor{plotblue}{{\tiny \XSolid}}) of size 24 and the test set (\textcolor{plotorange}{{\tiny \XSolid}}) of size 36.}
\end{subfigure}\hfill%
\begin{subfigure}{.49\textwidth}
  \centering
  \includegraphics[width=.99\linewidth]{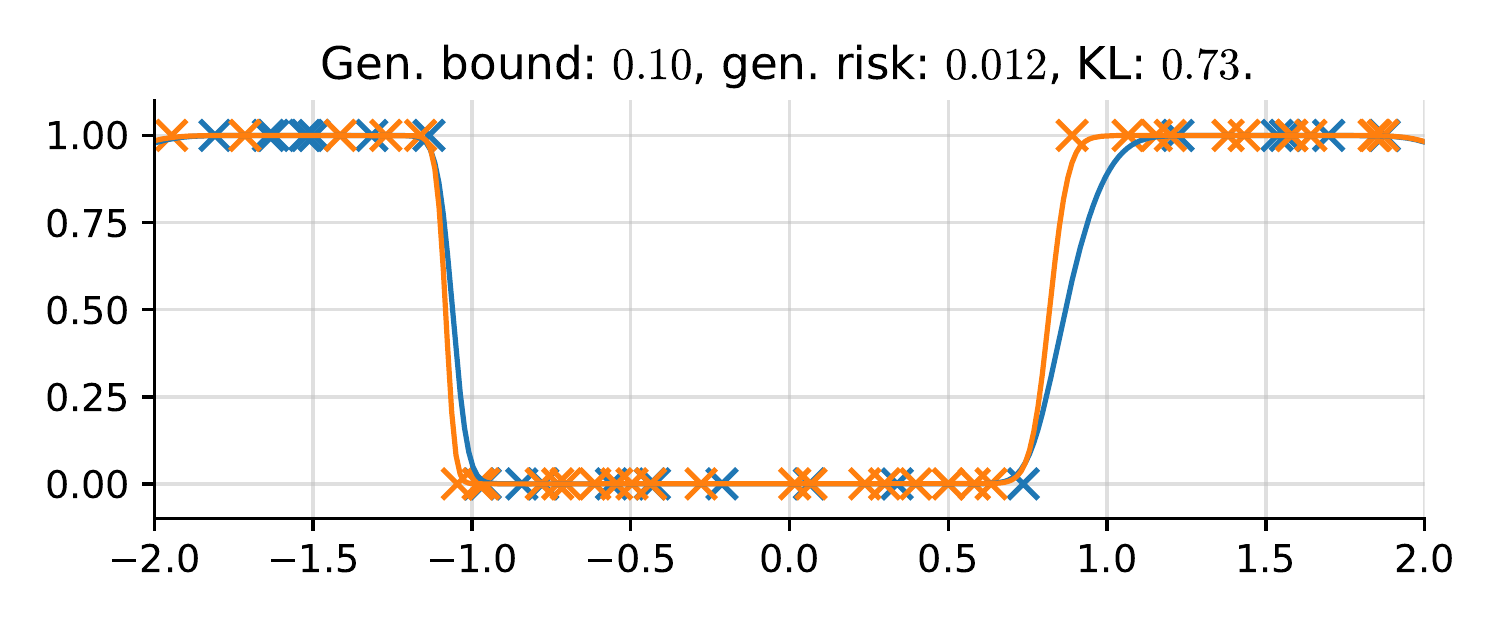}
  \caption{\textbf{Catoni} bound with data-dependent prior, showing the prior (\textcolor{plotorange}{---}) and posterior (\textcolor{plotblue}{---}) predictive, prior set (\textcolor{plotblue}{{\tiny \XSolid}}) of size 24 and risk set (\textcolor{plotorange}{{\tiny \XSolid}}) of size 36.}
\end{subfigure}
\caption{Predictions and bounds on one of the 1D datasets with \textbf{60 datapoints}. The bounds hold with failure probability $\delta = 0.1$.}
\end{figure}

\begin{figure}[h]
\centering
\begin{subfigure}{.49\textwidth}
  \centering
  \includegraphics[width=.99\linewidth]{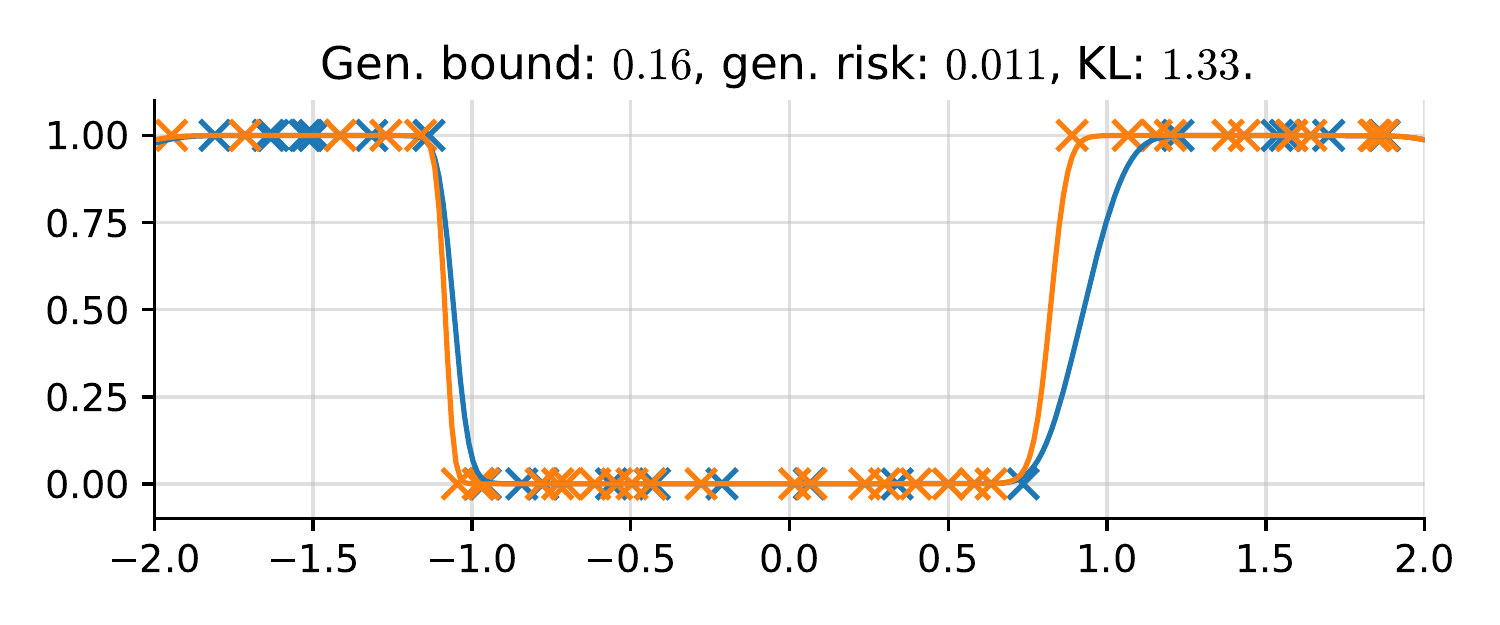}
  \caption{\textbf{PAC-Bayes-kl} bound with data-dependent prior, showing the prior (\textcolor{plotblue}{---}) and posterior (\textcolor{plotorange}{---}) predictive, prior set (\textcolor{plotblue}{{\tiny \XSolid}}) of size 24 and risk set (\textcolor{plotorange}{{\tiny \XSolid}}) of size 36.}
\end{subfigure}\hfill%
\begin{subfigure}{.49\textwidth}
  \centering
  \includegraphics[width=.99\linewidth]{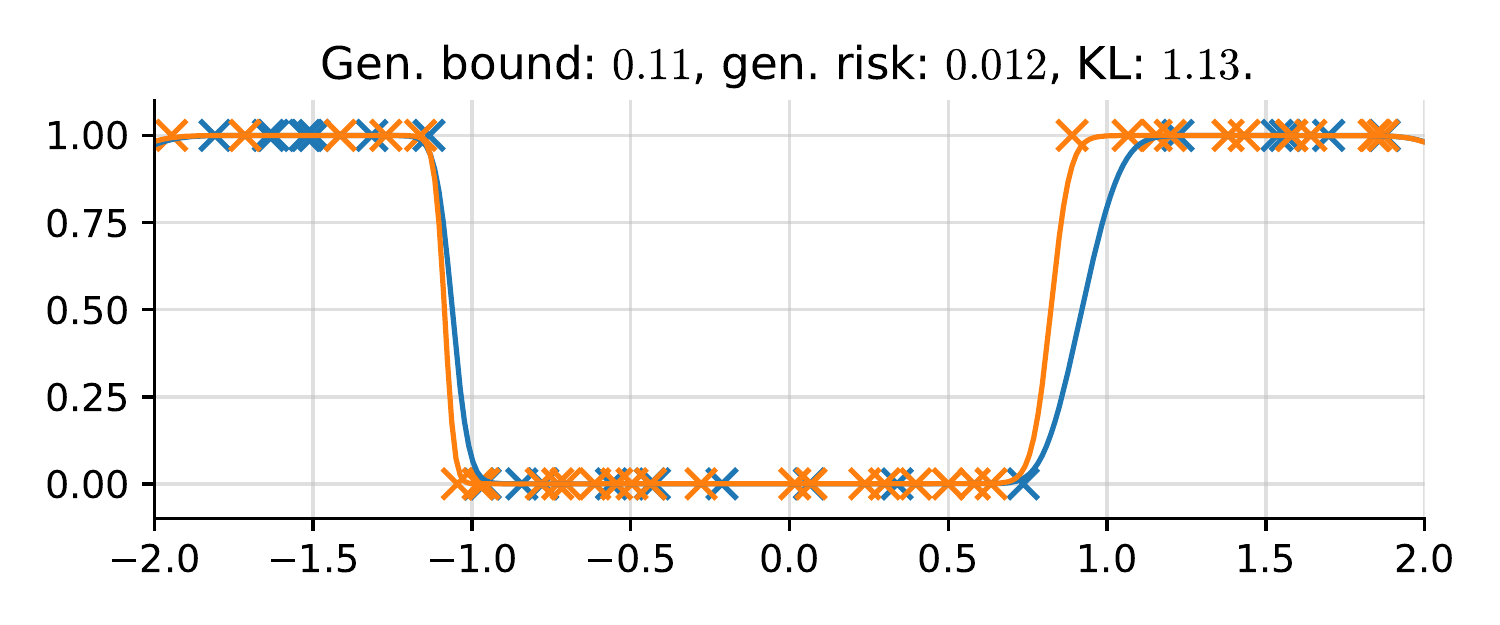}
  \caption{\textbf{Learned convex} bound with data-dependent prior, showing the prior (\textcolor{plotblue}{---}) and posterior (\textcolor{plotorange}{---}) predictive, prior set (\textcolor{plotblue}{{\tiny \XSolid}}) of size 24 and risk set (\textcolor{plotorange}{{\tiny \XSolid}}) of size 36.}
\end{subfigure}
\caption{Predictions and bounds on one of the 1D datasets with \textbf{60 datapoints}. The bounds hold with failure probability $\delta = 0.1$.}
\end{figure}

\subsection{Performance of MLP-NP} \label{app:additional_plots_mlp}

In the main body, we considered the CNN-NP model, since it performed better while training much faster and requiring fewer parameters then the MLP-NP. In \cref{fig:gen_bounds_1d_MLP_post_opt,fig:gen_bounds_MLP_1d_no_post_opt} we also show the performance of the MLP-NP for the test set meta-learners and also the Catoni bound meta-learner, both with and without post-hoc optimisation (see \cref{app:post_optimisation}). We see that the MLP-NP test set meta-learner performs very similarly to the CNN-NP one when $\setsize = 30$, but performs slightly worse when $\setsize = 60$. The MLP-NP Catoni meta-learner is either as tight as the CNN-NP Catoni meta-learner, or slightly looser, except when $\setsize=30$ and the prior proportion is $0$ or $0.8$, in which case the MLP-NP seems to have encountered learning difficulties. Also note that generalisation risk is generally higher for the MLP-NP than the CNN-NP. 

\begin{figure}[h]
\centering
\begin{subfigure}{.49\textwidth}
  \centering
  \includegraphics[width=.49\linewidth]{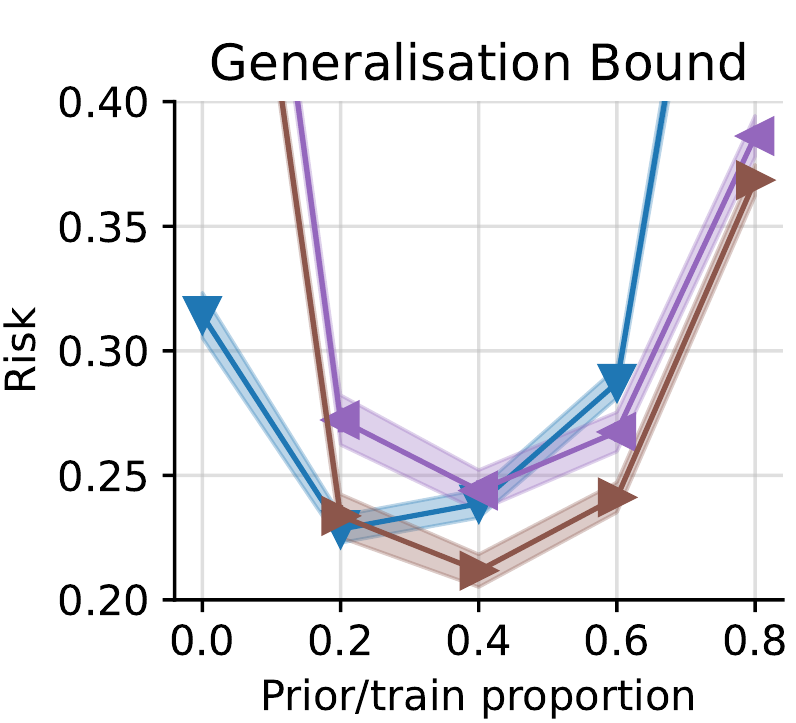}
  \includegraphics[width=.49\linewidth]{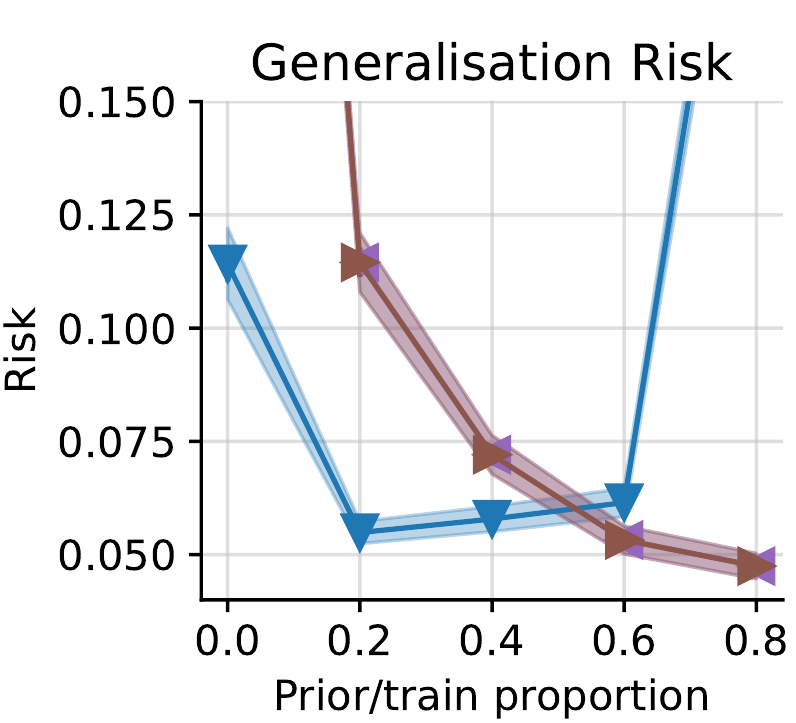}
  \caption{$N=30$ datapoints.}
\end{subfigure}%
\begin{subfigure}{.49\textwidth}
  \centering
  \includegraphics[width=.49\linewidth]{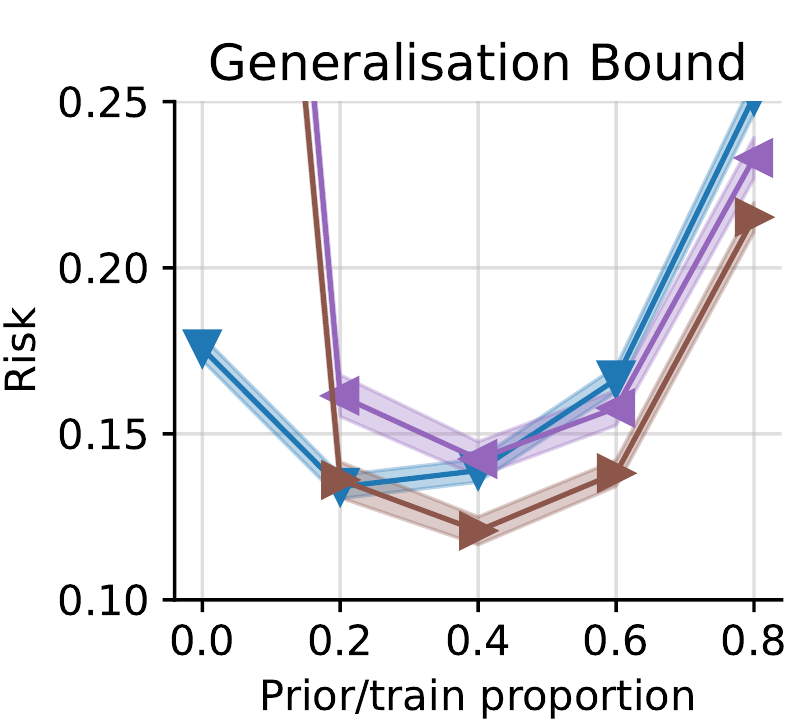}
  \includegraphics[width=.49\linewidth]{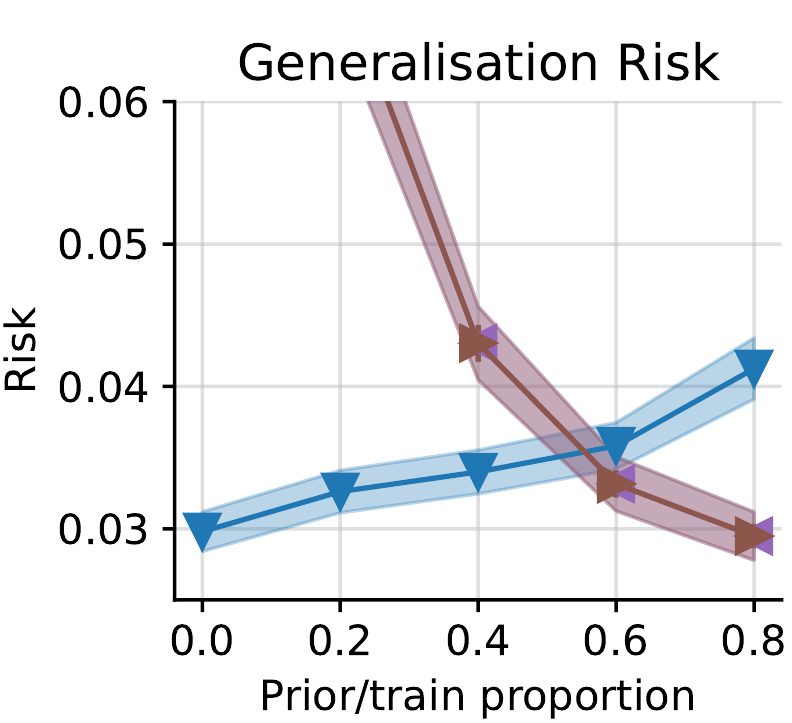}
  \caption{$N=60$ datapoints.}
\end{subfigure} 
\caption{Average generalisation bound and actual generalisation risk \textbf{for MLP-NP with post-hoc optimisation} ($\pm$ two standard errors) for Catoni ($\textcolor{plotblue}{\blacktriangledown}$),
Chernoff test set bound
($\textcolor{plotpurple}{\blacktriangleleft}$), 
and binomial tail test set bound 
($\textcolor{plotbrown}{\blacktriangleright}$). All bounds hold with failure probability $\delta = 0.1$.
} \label{fig:gen_bounds_1d_MLP_post_opt}
\end{figure}

\begin{figure}[h]
\centering
\begin{subfigure}{.49\textwidth}
  \centering
  \includegraphics[width=.49\linewidth]{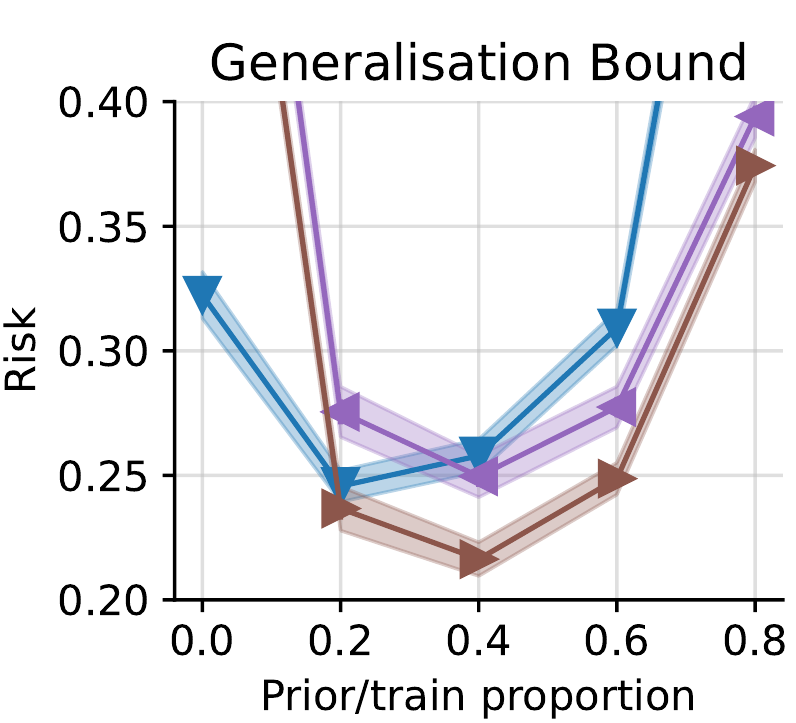}
  \includegraphics[width=.49\linewidth]{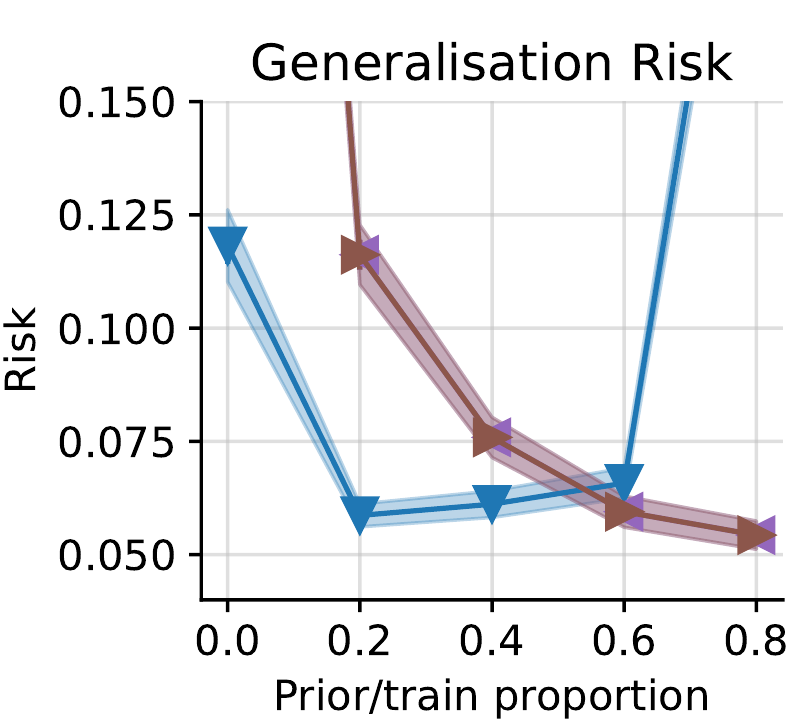}
  \caption{$N=30$ datapoints.}
\end{subfigure}%
\begin{subfigure}{.49\textwidth}
  \centering
  \includegraphics[width=.49\linewidth]{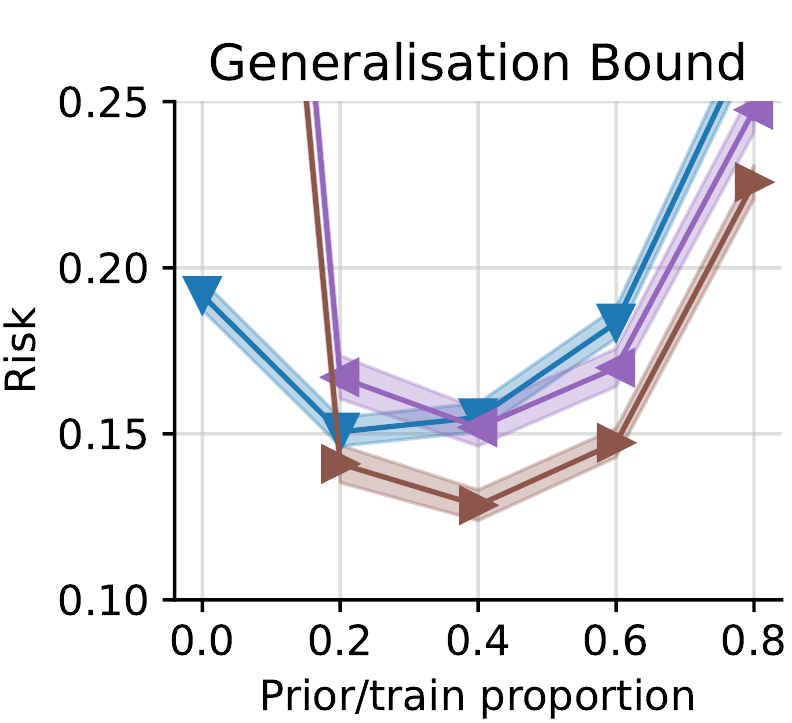}
  \includegraphics[width=.49\linewidth]{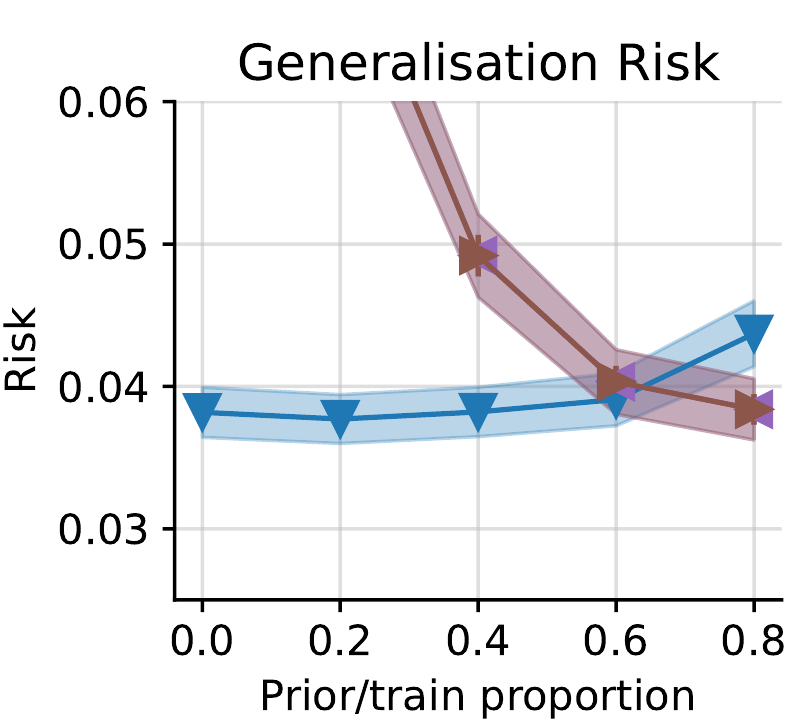}
  \caption{$N=60$ datapoints.}
\end{subfigure} 
\caption{Average generalisation bound and actual generalisation risk \textbf{for MLP-NP without post-hoc optimisation} ($\pm$ two standard errors) for Catoni ($\textcolor{plotblue}{\blacktriangledown}$), 
Chernoff test set bound
($\textcolor{plotpurple}{\blacktriangleleft}$), 
and binomial tail test set bound 
($\textcolor{plotbrown}{\blacktriangleright}$). All bounds hold with failure probability $\delta = 0.1$.} \label{fig:gen_bounds_MLP_1d_no_post_opt}
\end{figure}

\clearpage
\printbibliography



\end{document}